\documentclass[nohyperref]{article}

\usepackage{microtype}
\usepackage{graphicx}
\usepackage{booktabs} 

\usepackage{hyperref}

\usepackage[accepted]{icml2023}

\usepackage{amsmath}
\usepackage{amssymb}
\usepackage{mathtools}
\usepackage{amsthm}

\usepackage{amsmath,amsfonts,bm}

\def\eqref#1{equation~\ref{#1}}

\def\1{\bm{1}}

\DeclareMathAlphabet{\mathsfit}{\encodingdefault}{\sfdefault}{m}{sl}
\SetMathAlphabet{\mathsfit}{bold}{\encodingdefault}{\sfdefault}{bx}{n}

\def\sR{{\mathbb{R}}}

\usepackage{amsfonts}       
\usepackage{nicefrac}       
\usepackage{microtype}      
\usepackage{xcolor}         
\usepackage{amsmath, amsthm, amssymb}
\usepackage{subcaption}
\usepackage{enumitem}
\usepackage{wrapfig}
\usepackage{multirow}

\providecommand{\zvar}{\mathbf{z}}
\providecommand{\wvar}{\mathbf{w}}

\providecommand{\yvar}{\mathbf{y}}
\providecommand{\uvar}{\mathbf{u}}
\providecommand{\vvar}{\mathbf{v}}
\providecommand{\zset}{\mathcal{Z}}
\providecommand{\wset}{\mathcal{W}}
\providecommand{\mset}{\mathcal{M}}
\providecommand{\xset}{\mathcal{X}}
\providecommand{\dset}{\mathcal{D}}

\usepackage{soul}    

\usepackage[capitalize,noabbrev]{cleveref}

\theoremstyle{plain}
\newtheorem{theorem}{Theorem}[section]

\newtheorem{lemma}[theorem]{Lemma}

\theoremstyle{definition}

\theoremstyle{remark}

\usepackage[textsize=tiny]{todonotes}

\icmltitlerunning{Analyzing the Latent Space of GAN through Local Dimension Estimation}

\begin{document}

\twocolumn[
\icmltitle{Analyzing the Latent Space of GAN through Local Dimension Estimation}



\icmlsetsymbol{equal}{*}

\begin{icmlauthorlist}
\icmlauthor{Jaewoong Choi}{kias}
\icmlauthor{Geonho Hwang}{kias}
\icmlauthor{Hyunsoo Cho}{snu}
\icmlauthor{Myungjoo Kang}{snu}
\end{icmlauthorlist}

\icmlaffiliation{kias}{Korea Institute for Advanced Study}
\icmlaffiliation{snu}{Seoul National University}

\icmlcorrespondingauthor{Jaewoong Choi}{chjw1475@kias.re.kr}
\icmlcorrespondingauthor{Myungjoo Kang}{mkang@snu.ac.kr}


\vskip 0.3in
]



\printAffiliationsAndNotice{}  

\begin{abstract}
The impressive success of style-based GANs (StyleGANs) in high-fidelity image synthesis has motivated research to understand the semantic properties of their latent spaces. In this paper, we approach this problem through a geometric analysis of latent spaces as a manifold. In particular, we propose a local dimension estimation algorithm for arbitrary intermediate layers in a pre-trained GAN model. The estimated local dimension is interpreted as the number of possible semantic variations from this latent variable. Moreover, this intrinsic dimension estimation enables unsupervised evaluation of disentanglement for a latent space. Our proposed metric, called \textit{Distortion}, measures an inconsistency of intrinsic tangent space on the learned latent space. Distortion is purely geometric and does not require any additional attribute information. Nevertheless, Distortion shows a high correlation with the global-basis-compatibility and supervised disentanglement score. Our work is the first step towards selecting the most disentangled latent space among various latent spaces in a GAN without attribute labels.
\end{abstract}

\begin{figure}[h]
  \begin{center}
    \includegraphics[width=\columnwidth]{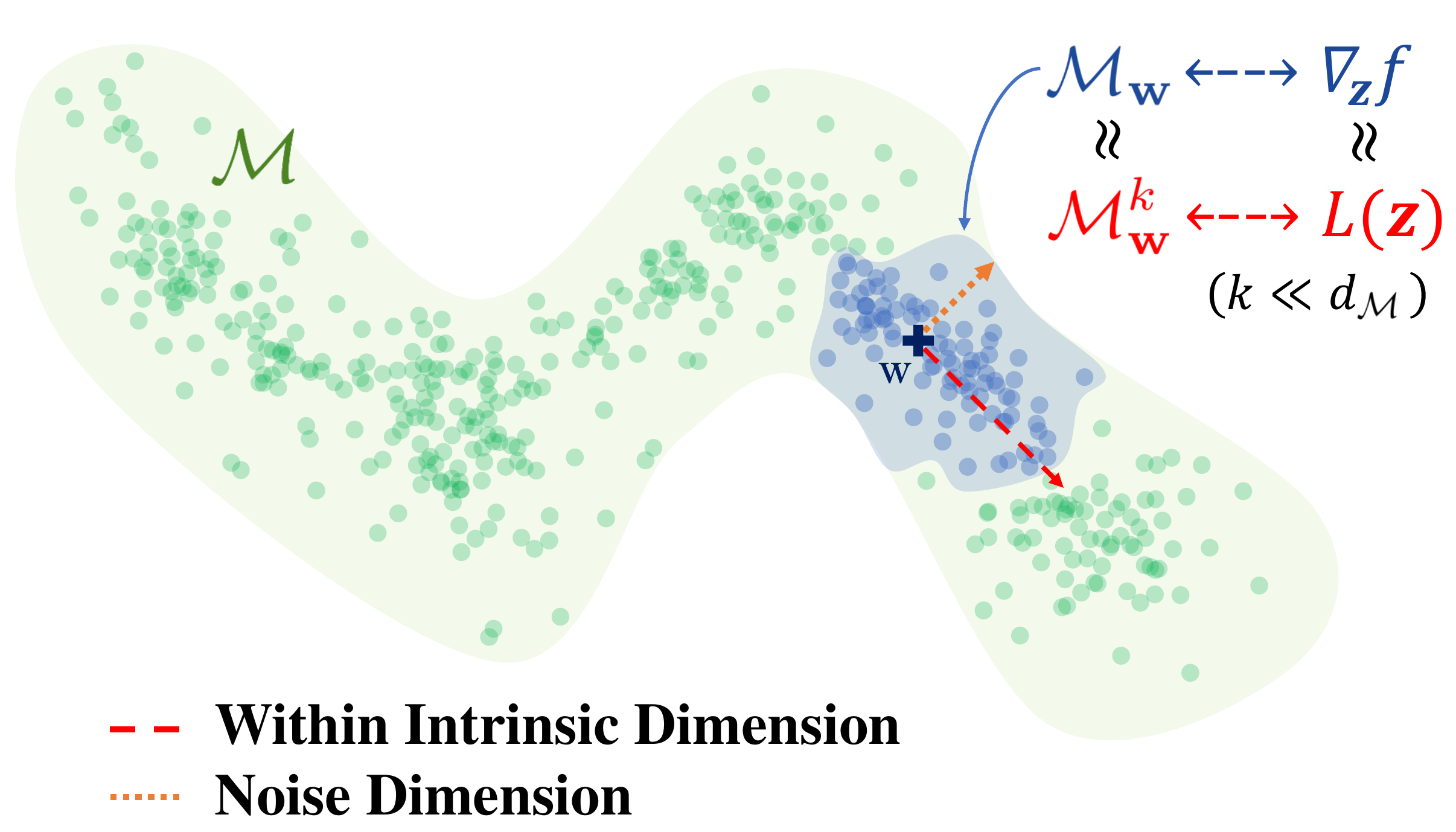}
  \end{center}
  \caption{
   \textbf{Local Dimension Estimation on Latent Manifold $\mset$.} Our goal is to find the dimension $k$ such that the $k$-dimensional submanifold $\mset_{\wvar}^{k}$ represents the denoised local latent manifold $\mset_{\wvar} \subseteq \sR^{d_{\mset}}$.   
   }
  \label{fig:concept}
  \vspace{-10pt}  
\end{figure}

\section{Introduction}
Generative Adversarial Networks (GANs) \citep{goodfellow2014generative} have achieved remarkable success in generating realistic high-resolution images \citep{karras2018progressive, karras2019style, karras2020analyzing, AliasFreeGAN, karras2020training, brock2018large}. Nevertheless, understanding how GAN models represent the semantics of images in their latent spaces is still a challenging problem. To this end, several recent works investigated the disentanglement \citep{bengio2013representation} properties of the latent space in GAN \citep{goetschalckx2019ganalyze, jahanian2019steerability, plumerault2020controlling, shen2020interpreting}. 
A latent space is called \textit{disentangled} if there is a bijective correspondence between each semantic attribute and each axis of latent space when represented with the optimal basis. (See the appendix for detail.)
Due to this bijective correspondence, the disentangled latent space is a concise and interpretable representation of data \cite{bengio2013representation}.

The style-based GAN models \citep{karras2019style, karras2020analyzing} have been popular in previous studies for identifying a disentangled latent space in a pre-trained model. First, the space of style vector, called $\wset$-space, was shown to provide a better disentanglement property compared to the latent noise space $\zset$ \citep{karras2019style}. After that, several attempts have been made to discover other disentangled latent spaces, such as $\wset^{+}$-space \citep{abdal2019image2stylegan} and $\mathcal{S}$-space \citep{wu2020stylespace}. 
However, the degree of disentanglement was assessed by the manual inspection \citep{karras2019style,abdal2019image2stylegan,wu2020stylespace} or by the quantitative scores employing a pre-trained feature extractor \citep{karras2019style} or an attribute annotator \citep{karras2019style,wu2020stylespace}. 
The manual inspection is vulnerable to sample dependency, and the quantitative scores depend on the pre-trained models and the set of selected target attributes. Therefore, we need an \textit{unsupervised} quantitative evaluation scheme for the disentanglement of latent space that \textit{does not rely on pre-trained models}.

In this paper, we investigate the semantic property of a latent space in pre-trained GAN models by analyzing its geometrical property. Specifically, we propose a method for estimating the local (manifold) dimension of the learned latent space. 
Recently, \citet{LocalBasis} demonstrated that, in the learned latent space, the basis of tangent space describes the disentangled semantic variations.
In this regard, the local intrinsic dimension represents the number of disentangled semantic variations from the given latent variable.
Furthermore, this intrinsic dimension estimate leads to an unsupervised quantitative score for the disentanglement property of the latent space, called \textit{Distortion}. Our experiments demonstrate that our proposed metric shows a high correlation with the global-basis-compatibility and supervised disentanglement score.
(The global-basis-compatibility will be defined in Sec \ref{sec:distortion}.)
Our contributions can be summarized as follows:
\begin{enumerate}[topsep=-1pt, itemsep=-1pt]
    \item We propose a local intrinsic dimension estimation scheme for an intermediate latent space in pre-trained GAN models. The estimated dimension and the corresponding refined latent manifold describe the principal variations of generated images.   
    \item We propose a layer-wise disentanglement score, called \textit{Distortion}, that measures the inconsistency of intrinsic tangent space. The proposed metric shows a high correlation with the global-basis-compatibility and supervised disentanglement score.
    \item We analyze the intermediate layers of the mapping network through Distortion metric. Our analysis elucidates the superior disentanglement of $\wset$-space compared to the other intermediate layers and suggests a criterion for finding a similar-or-better alternative.
\end{enumerate}

\section{Related Works}
\paragraph{Style-based Generator}
Recently, GANs with style-based generator architecture \citep{karras2019style, karras2020analyzing, AliasFreeGAN, StyleGANXL} have achieved state-of-the-art performance in realistic image generation. In conventional GAN architecture, such as DCGAN \citep{radford2015unsupervised} and ProGAN \citep{karras2018progressive}, the generator synthesizes an image by transforming a latent noise with a sequence of convolutional layers. 
On the other hand, the style-based generator consists of two subnetworks: \textit{mapping network} $f : \zset \rightarrow \wset$ and \textit{synthesis network} $g :  \sR^{n_{0}} \times \wset^{L}  \rightarrow \xset$. The synthesis network is similar to conventional generators in that it is composed of a series of convolutional layers $\{g_{i}\}_{i=1, \cdots, L}$. The key difference is that the synthesis network takes the learned constant feature $\yvar_{0} \in \sR^{n_0}$ at the first layer $g_{0}$, and then adjusts the output image by injecting the layer-wise styles $\wvar$ and noise (Layer-wise noise is omitted for brevity.):
\begin{equation} \label{eq:stylegan_syn}
    \yvar_{i} = g_{i} (\yvar_{i-1}, \wvar) \,\, \mathrm{with} \,\, \wvar = f(\zvar) 
    \,\, \mathrm{for} \,\, i = 1, \cdots, L,
\end{equation}
where the style vector $\wvar$ is attained by transforming a latent noise $\zvar$ via the mapping network $f$. 

\paragraph{Understanding Latent Semantics}
The previous attempts to understand the semantic property of latent spaces in StyleGANs are categorized into two topics: (i) finding more disentangled latent space in a model; (ii) discovering meaningful perturbation directions in a latent space corresponding to disentangled semantics. More precisely, the semantic perturbation in (ii) is defined as a latent perturbation that changes a generated image in only one semantics.
Several studies on (i) suggested various disentangled latent spaces in StyleGAN models, for example, $\wset$ \citep{karras2019style}, $\wset^{+}$ \citep{abdal2019image2stylegan}, $P_{N}$ \citep{Pnspace}, and $\mathcal{S}$-space \citep{wu2020stylespace}. 
However, the disentanglement of each latent space was compared by observing samples manually \citep{karras2019style,abdal2019image2stylegan,wu2020stylespace} or by quantitative metrics relying on pre-trained models \citep{karras2019style,wu2020stylespace}. Our Distortion metric achieves a high correlation with the supervised metric, e.g., DCI metric \cite{DCIMetric}, while not requiring a pre-trained model.

Also, the previous works on (ii) are classified into local and global methods. The local methods find sample-wise perturbation directions  \citep{ramesh2018spectral, patashnik2021styleclip, abdal2021styleflow, LowRankSubspace, LocalBasis}. On the other hand, the global methods search sample-independent perturbation directions that perform the same semantic manipulation on the entire latent space \citep{harkonen2020ganspace, shen2020closed, voynov2020unsupervised}. \textbf{Throughout this paper, we refer to these local methods as \textit{local basis} and these global methods as \textit{global basis}.}
GANSpace \citep{harkonen2020ganspace} showed that the principal components obtained by PCA can serve as the global basis. SeFa \citep{shen2020closed} suggested the singular vectors of the first weight parameter applied to latent noise as the global basis. 
These global basis showed promising results, but they were successful in a limited area. Depending on the sampled latent variables, these methods exhibited limited semantic factorization and sharp degradation of image fidelity \cite{LocalBasis, frechetbasis}. In this regard, \citet{LocalBasis} suggested the need for diagnosing a global-basis-compatibility of latent space. 
If there is no ideal global basis for semantic attributes in the target latent space in the first place, all global basis can only attain limited success on it. Our Distortion metric can serve as a meaningful score for testing the global-basis-compatibility of latent spaces (Sec \ref{sec:distortion}).

\vspace{-5pt}
\paragraph{Semantic Analysis via Latent Manifold} \label{sec:background_local_basis}
\citet{LocalBasis} proposed an unsupervised method for finding local semantic perturbations based on the local geometry, called \textit{Local Basis (LB)}. Throughout this paper, we denote Local Basis as LB to avoid confusion with the general term "local basis" in the previous paragraph.
Assume the support $\zset$ of input prior distribution $p(\zvar)$ is the entire Euclidean space, i.e., $\zset = \sR^{d_{\zset}}$, for example, Gaussian prior $p(\zvar) = \mathcal{N}(0, I)$. 
We denote the target latent space by $\mset = f(\zset)\subseteq \sR^{d_{\mset}}$ and refer to the subnetwork between them by $f$.
Note that $\mset$ is defined as an image of the trained subnetwork $f$. Thus, we call $\mset$ the \textit{learned latent space} or \textit{latent manifold} following the manifold interpretation of \citet{LocalBasis}.

LB is defined as the ordered basis of tangent space $T_{\wvar} \mset_{\wvar}^{k}$ at $\wvar = f(\zvar) \in \mset$ for the $k$-dimensional local approximating manifold $\mset_{\wvar}^{k}$. Here, $\mset_{\wvar}^{k} \subseteq \mset$ indicates a $k$-dimensional submanifold of $\mset$ that approximates $\mset$ around $\wvar$ (Fig \ref{fig:concept}) with $k \leq d_{\mset}$:
\begin{multline}
    \mset^{k}_{\wvar} \approx \mset_{\wvar} \\
    \textrm{where} \quad
    \mset_{\wvar} = \left\{ f\left(\zvar_{\epsilon} \right) \mid \| \zvar_{\epsilon} - \zvar \| < \epsilon \right\}
    \subseteq \mset. 
\end{multline}
Using the fact that $\mset=f(\zset)$, the local approximating manifold $\mset_{\wvar}^{k}$ can be discovered by solving the low-rank approximation problem of $df_{\zvar}$, i.e., the Jacobian matrix $\nabla_{\zvar} f$ of $f$ because of the assumption $T_{\zvar}\zset = \sR^{d_{\zset}}$:
\begin{equation}
    \textrm{minimize}_{L} \,\, \| df_{\zvar} - L \|_{2}, \quad \textrm{where} \,\, \operatorname{rank}(L) \leq k.
\end{equation}
The analytic solution of this low-rank approximation problem is obtained in terms of Singular Value Decomposition (SVD) by Eckart–Young–Mirsky Theorem \citep{SVD_lowrank}. From that, $\mset_{\wvar}^{k}$ and the corresponding LB are given as follows: For the $i$-th singular vector $\uvar^{\zvar}_{i}\in \sR^{d_{\zset}}$, $\vvar^{\wvar}_{i}\in \sR^{d_{\mset}}$, and $i$-th singular value $\sigma^{\zvar}_{i}\in \sR$ of $df_{\zvar}$ with $\sigma^{\zvar}_1\geq \cdots\geq \sigma^{\zvar}_{n}$,
\begin{align}
    &\text{(SVD) } \,\, df_{\zvar}(\uvar^{\zvar}_{i}) = \sigma^{\zvar}_{i} \cdot \vvar^{\wvar}_{i} \quad \text{for}\,\, \forall i, \label{eq:jacobian_svd}\\
    &\text{LB}(\wvar) = \{ \vvar^{\wvar}_{i} \}_{1\leq i\leq n} \quad \text{where} \,\, \wvar = f(\zvar), \\
    \mset^{k}_{\wvar} &= \left\{ f\left(\zvar + \sum_{1 \leq i \leq k} t_{i} \cdot \uvar^{\zvar}_{i} \right) \mid t_{i} \in (-\epsilon_{i} , \epsilon_{i} )\right\}. \label{eq:approx_mfd}
\end{align}
Note that \textbf{the tangent space of $\mset_{\wvar}^{k}$ is spanned by the top-$\boldsymbol{k}$ LB}, i.e. $T_{\wvar} \mset^{k}_{\wvar} = \operatorname{span}\{ \vvar^{\wvar}_{i} : 1 \leq i \leq k \} $. 
\citet{LocalBasis} demonstrated that LB serves as the local semantic perturbations from the base latent variable $\wvar$. However, estimating the number of meaningful perturbations remains elusive.
Since LB is defined as singular vectors, LB presents the candidates as much as $\min(d_{\zset}, d_{\mset})$, e.g., 512 for $\wset$-space in StyleGANs. To address this problem, we propose the local dimension estimation that can refine these candidates up to 90\%. Moreover, this local dimension estimation leads to an unsupervised disentanglement metric (Sec \ref{sec:distortion}).

\begin{figure*}[t]
    \centering
    \begin{subfigure}[b]{0.3\textwidth}
    \centering
    \includegraphics[width=\columnwidth]{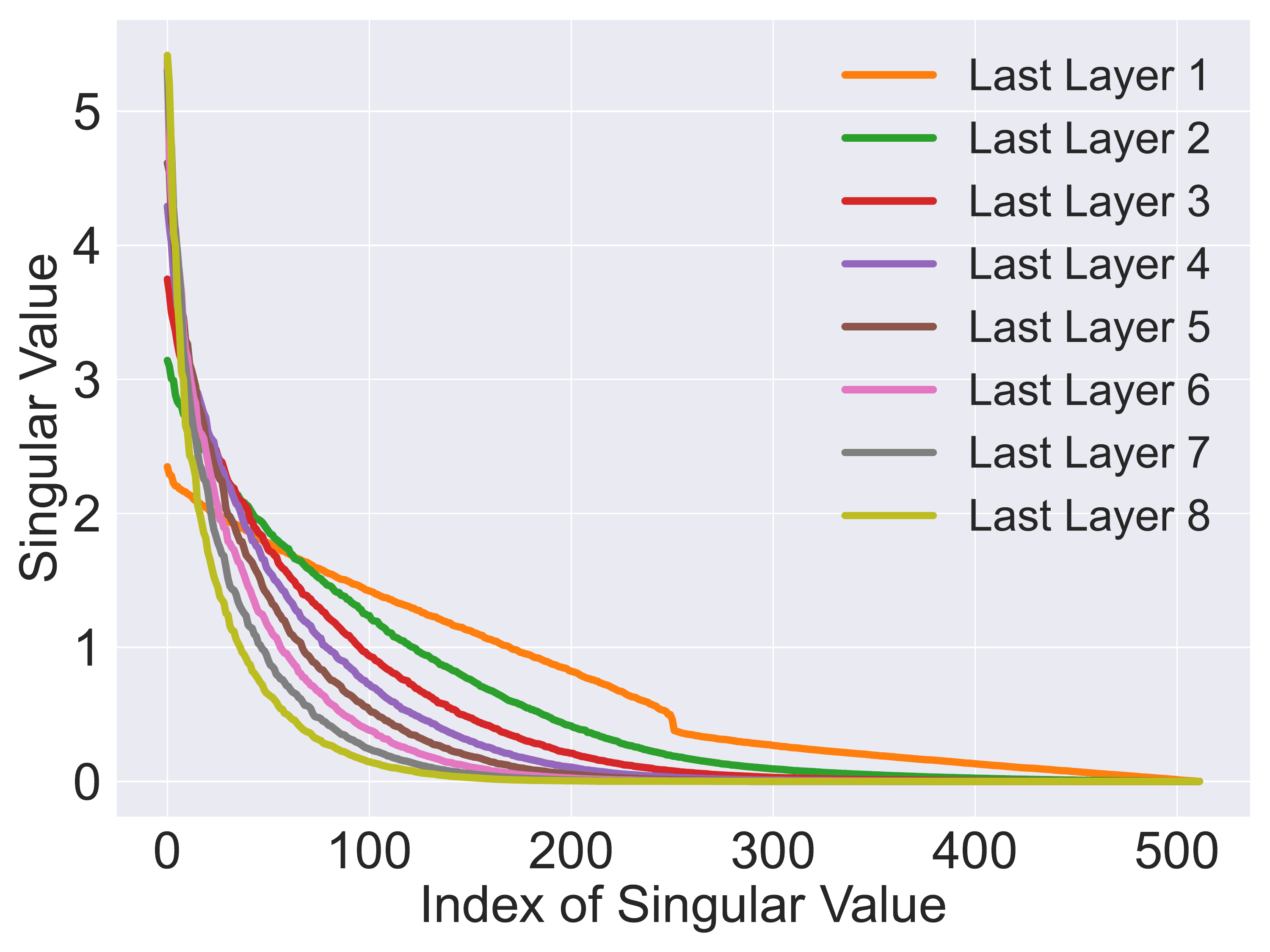}
    \caption{Singular Value Distribution}
    \label{fig:sv_dist}
    \end{subfigure}
    \quad
    \begin{subfigure}[b]{0.3\textwidth}
    \centering
    \includegraphics[width=\columnwidth]{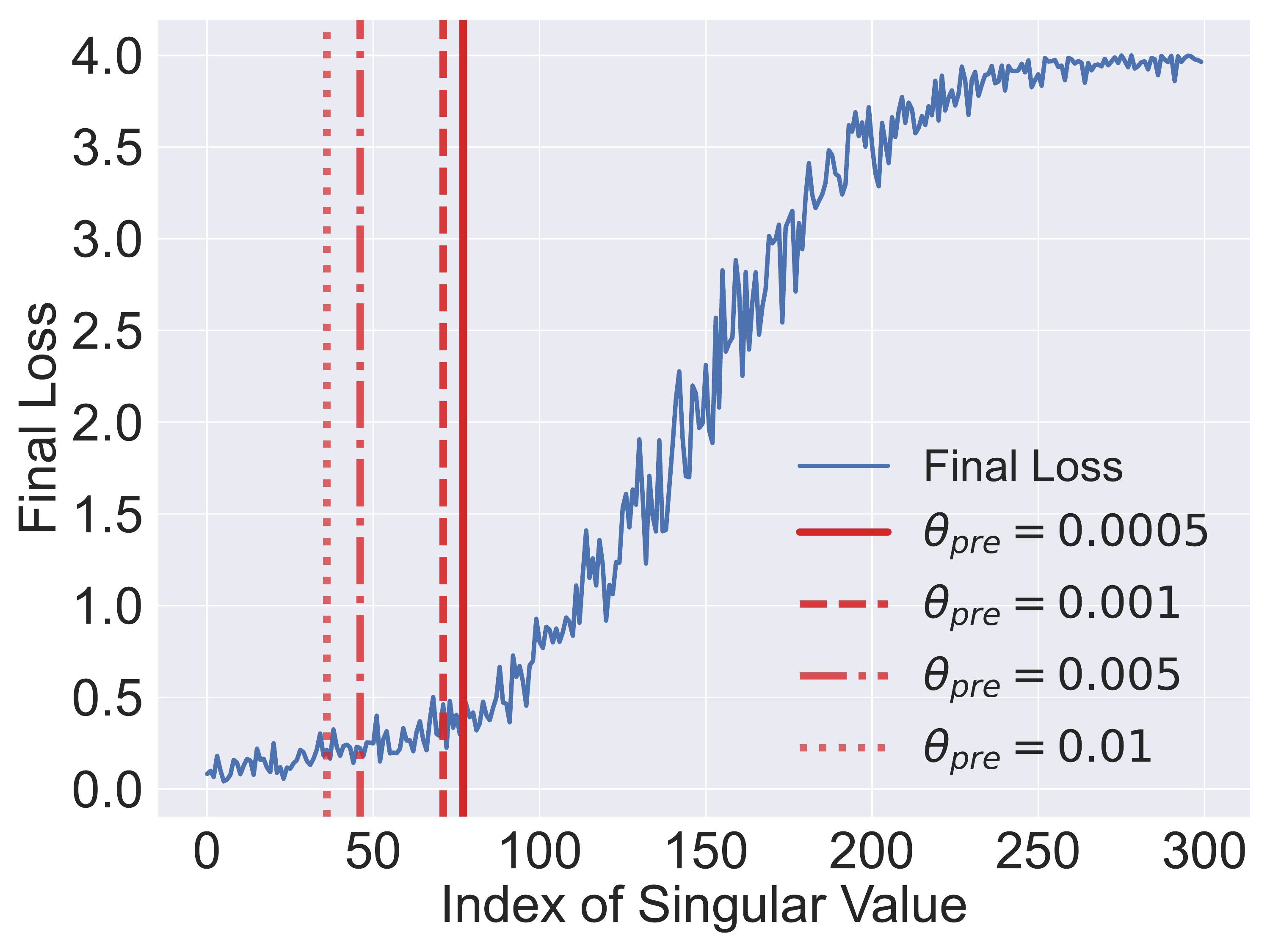}
    \caption{Off-manifold Results}
    \label{fig:off_mfd}
    \end{subfigure}
    \quad
    \begin{subfigure}[b]{0.3\textwidth}
    \centering
    \includegraphics[width=\columnwidth]{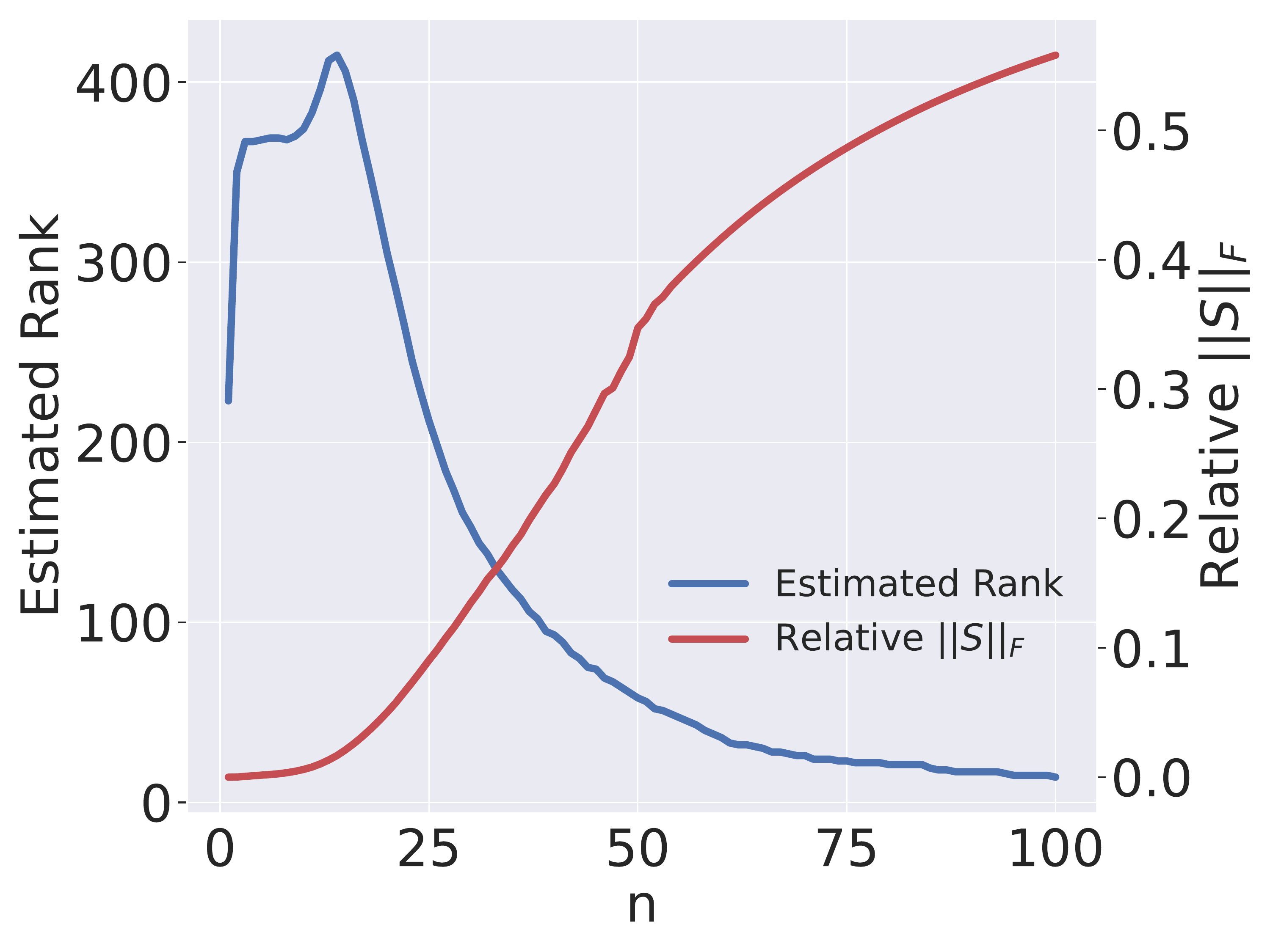}
    \caption{Rank Estimation with Sparsity}
    \label{fig:sparse_obj}
    \end{subfigure}
    \caption{
    \textbf{(a) Singular Value Distribution} of Jacobian matrix for each subnetwork of the mapping network in StyleGAN2. 
    \textbf{(b) Off-manifold Results} in $\wset$-space of StyleGAN2.
    \textbf{(c) Rank Estimation with Sparsity} under various $n = 1/\gamma$.
    }
    \vspace{-10pt}
\end{figure*}

\section{Latent Dimension Estimation}
In this section, we propose our local dimension estimation scheme for a \textit{learned} latent manifold in a pre-trained GAN model. We interpret the latent manifold $\mset=f(\zset)$ as a noisy manifold, and hence the subnetwork differential $df_{\zvar}$ as a noisy linear map. The local dimension at $\wvar \in \mset$ is estimated as the intrinsic rank of $df_{\zvar}$. Then, we evaluate the validity of this estimation scheme. In this section, our analysis of learned latent manifold is focused on the intermediate layers in the mapping network of StyleGAN2 trained on FFHQ (See Fig \ref{fig:architecture_stylegan} for architecture). However, our scheme can be applied to any $\zvar$-differentiable intermediate layers for an input latent noise $\zvar$.


\subsection{Method}
Throughout this work, we follow the notation presented in Sec \ref{sec:background_local_basis}. Consider a target latent manifold $\mset$ given by a subnetwork $f$, i.e., $\mset=f(\zset)$. Our goal is to estimate the intrinsic local dimension of $\mset$ around $\wvar = f(\zvar)$. 
Formally, the local dimension at $\wvar \in \mset$ is the dimension of Euclidean space $\sR^{k}$ whose open set is homeomorphic to the $\epsilon$-neighborhood of $\wvar$. However, the latent manifold exhibits some \textit{noise} in practice. 
The bijective condition of homeomorphism is too strict under the presence of noise.
Therefore, we define the local intrinsic dimension of learned latent space as the local dimension of \textit{denoised} latent space (Fig \ref{fig:concept}). In particular, we discover this intrinsic dimension by estimating the intrinsic rank of noisy subnetwork differential $df_{\zvar}$, i.e., the Jacobian $(\nabla_{\zvar}f)(\zvar)$.
The correspondence between the local dimension and rank of  $df_{\zvar}$ is described in Eq \ref{eq:approx_mfd} because the rank of a linear map is the same as the number of singular vectors with non-zero singular values. 

\vspace{-3pt}
\paragraph{Motivation}
Before presenting our dimension estimation method, we provide motivation for introducing the lower-dimensional approximation to $\mset$. Figure \ref{fig:sv_dist} shows the singular value distribution of Jacobian matrices evaluated for the subnetworks of the mapping network in StyleGAN2. \textit{Last layer $i$} denotes the subnetwork from the input noise space $\zset$ to the $i$-th fully connected layer. The distribution of singular values $\{\sigma_{i}^{\zvar}\}_{i}$ gets monotonically sparser as the subnetwork gets deeper. In particular, $\wset$-space, i.e., \textit{Last layer 8}, is extremely sparse as much as $\sigma_{150}^{\zvar} / \sigma_{1}^{\zvar} \approx 0.005$. Therefore, it is reasonable to prune the singular values with negligible magnitude and consider the lower-dimensional approximation of the learned latent manifold. 

\vspace{-5pt}
\paragraph{Local Dimension Estimation}
Our intrinsic rank estimation algorithm distinguishes the large meaningful components and the small noise-like components given the singular values $\{\sigma_{i}^{\zvar}\}_{i}$ of $(\nabla_{\zvar}f)(\zvar)$. 
The Pseudorank algorithm \citep{RankEstimate} estimates the number of meaningful components in a linear mixture model from noisy samples. We reinterpret this algorithm to find the intrinsic rank of noisy $(\nabla_{\zvar}f)(\zvar)$.
Assume the \textit{isotropic Gaussian noise} on the Jacobian $(\nabla_{\zvar}f)(\zvar) \in \sR^{d_{\mset} \times d_{\zset}}$:
\begin{equation} \label{eq:noisy_jacobian}
    (\nabla_{\zvar}f)(\zvar) = L(\zvar) + \sigma \cdot \left(\epsilon_{1}, \cdots, \epsilon_{d_{\mset}}\right)^{\intercal},
\end{equation}
with $\epsilon_{i} \sim \mathcal{N}(0, I_{d_{\zset}})$ 
where $L(\zvar)$ denotes the denoised low-rank representation of $(\nabla_{\zvar}f  )(\zvar)$. Then, taking the expectation over the noise distribution gives:
\begin{equation}
    \mathbb{E}_{\epsilon}\left[(\nabla_{\zvar}f)^{\intercal}(\zvar) \cdot (\nabla_{\zvar}f)(\zvar) \right] 
    = L^{\intercal}(\zvar) \cdot L(\zvar) + \sigma^{2} \cdot I_{d_{\zset}}.
\end{equation}
The eigenvalues of $\left[(\nabla_{\zvar}f)^{\intercal}(\zvar) \cdot (\nabla_{\zvar}f)(\zvar) \right]$ are the squares of sigular values $\{(\sigma_{i}^{\zvar})^{2}\}_{i}$, and the noise covariance term $\sigma^{2} \cdot I_{d_{\zset}}$ increases all eigenvalues by $\sigma^{2}$. This observation explains our intuition that large singular values correspond to signals and small ones correspond to noise. 
Therefore, determining the intrinsic rank of $(\nabla_{\zvar})f(\zvar)$ is closely related to the largest eigenvalue $\lambda_{1}$ of the empirical covariance matrix $S = \frac{1}{d_{\zset}} \sum_{i} \epsilon_{i} \cdot \epsilon_{i}^{\intercal}$, which is the threshold for distinguishing between signal and noise. 
In the random matrix theory, it is known that the distribution of this largest eigenvalue $\lambda_{1}$ for $n$-samples of $\mathcal{N}(0, I_{p})$ converges to a Tracy-Widom distribution $F_{\beta}$ of order $\beta=1$ for real-valued observations \citep{TracyWidom} (See the appendix for detail.):
\begin{multline}
    P\left(\lambda_{1} < \sigma^{2} \left( \mu_{n, p} + s\cdot \sigma_{n, p} \right) \right) \rightarrow F_{\beta}(s) \\
    \textrm{as $n, p \rightarrow \infty$ with $c=p/n$ fixed.}
\end{multline}
Using the above theoretical results, our algorithm applies a sequence of nested hypothesis tests. Given the Jacobian $(\nabla_{\zvar}f)(\zvar) \in \sR^{d_{\mset} \times d_{\zset}}$, let $p = \min(d_{\mset}, d_{\zset})$. Then, for $k=1, 2, \cdots, p-1$,
\begin{equation}
    \mathcal{H}_{0} : \textrm{rank at least } k  \,\,\, \text{vs.} \,\,\,
    \mathcal{H}_{1} : \textrm{rank at most } (k-1)  
\end{equation}
For each $k$, the hypothesis test consists of two parts. First, the noise level $\sigma$ of $(\nabla_{\zvar}f)(\zvar)$ in Eq \ref{eq:noisy_jacobian} should be estimated to perform a hypothesis test. 
We adopted the consistent noise estimate $\sigma_{est}(k)$ from \citet{RankEstimate}, which is derived by assuming that $\lambda_{k+1}, \lambda_{k+2} \cdots, \lambda_{p}$ are the noise components for $\lambda_{i} = (\sigma_{i}^{\zvar})^{2}$.
Second, we test whether $\lambda_{k}$ belongs to the noise components based on the corresponding Tracy-Widom distribution as follows:
\begin{equation} \label{eq:rank_test}
    \lambda_{k} \leq \sigma_{est}^{2}(k) \left( \mu_{n, p-k} + s(\alpha)\cdot \sigma_{n, p-k} \right),
\end{equation}
where $\alpha$ denotes a confidence level. We chose $\alpha=0.1$ in our experiments. The above test is repeated until Eq \ref{eq:rank_test} is satisfied.  Then, the estimated rank $K$ becomes $K=k-1$ because Eq \ref{eq:rank_test} means that $\lambda_{k}$ is not large enough to be judged as a signal component.

\begin{figure*}[t]
    \centering
    \begin{subfigure}[b]{0.3\textwidth}
    \centering
    \includegraphics[width=\columnwidth]{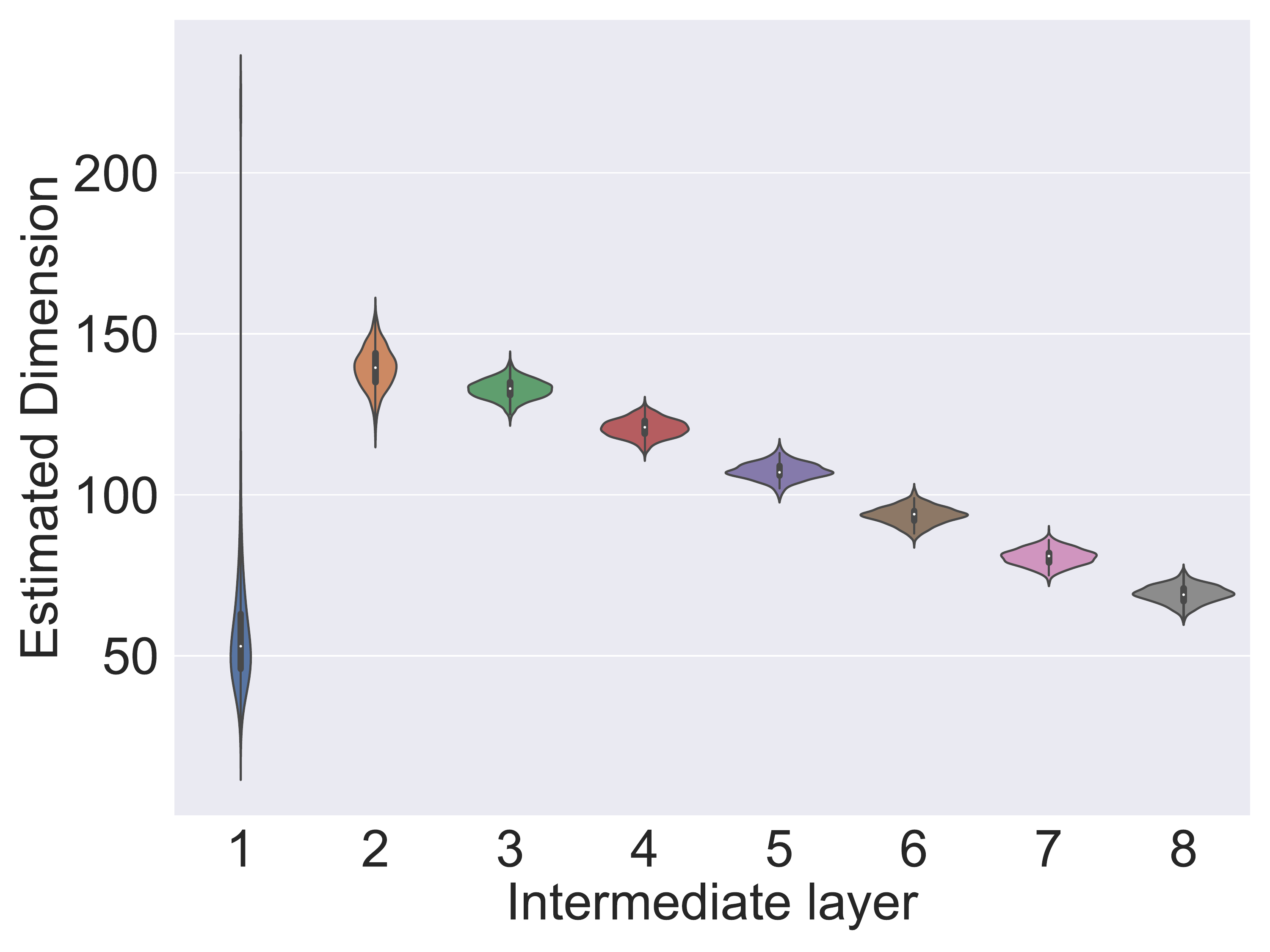}
    \caption{$\theta_{pre}=0.001$}
    \end{subfigure}
    \quad
    \begin{subfigure}[b]{0.3\textwidth}
    \centering
    \includegraphics[width=\columnwidth]{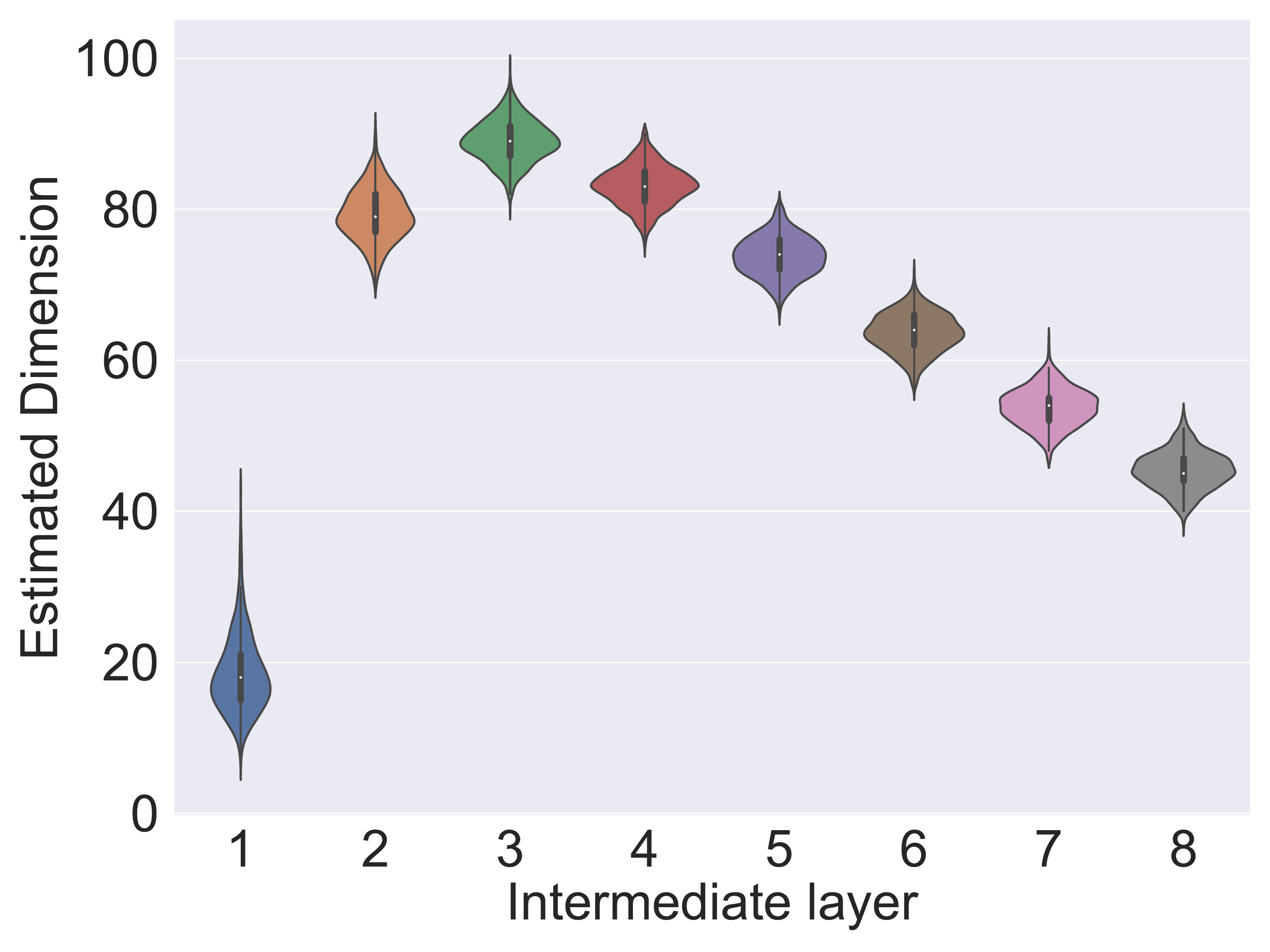}
    \caption{$\theta_{pre}=0.005$}
    \end{subfigure}
    \quad
    \begin{subfigure}[b]{0.3\textwidth}
    \centering
    \includegraphics[width=\columnwidth]{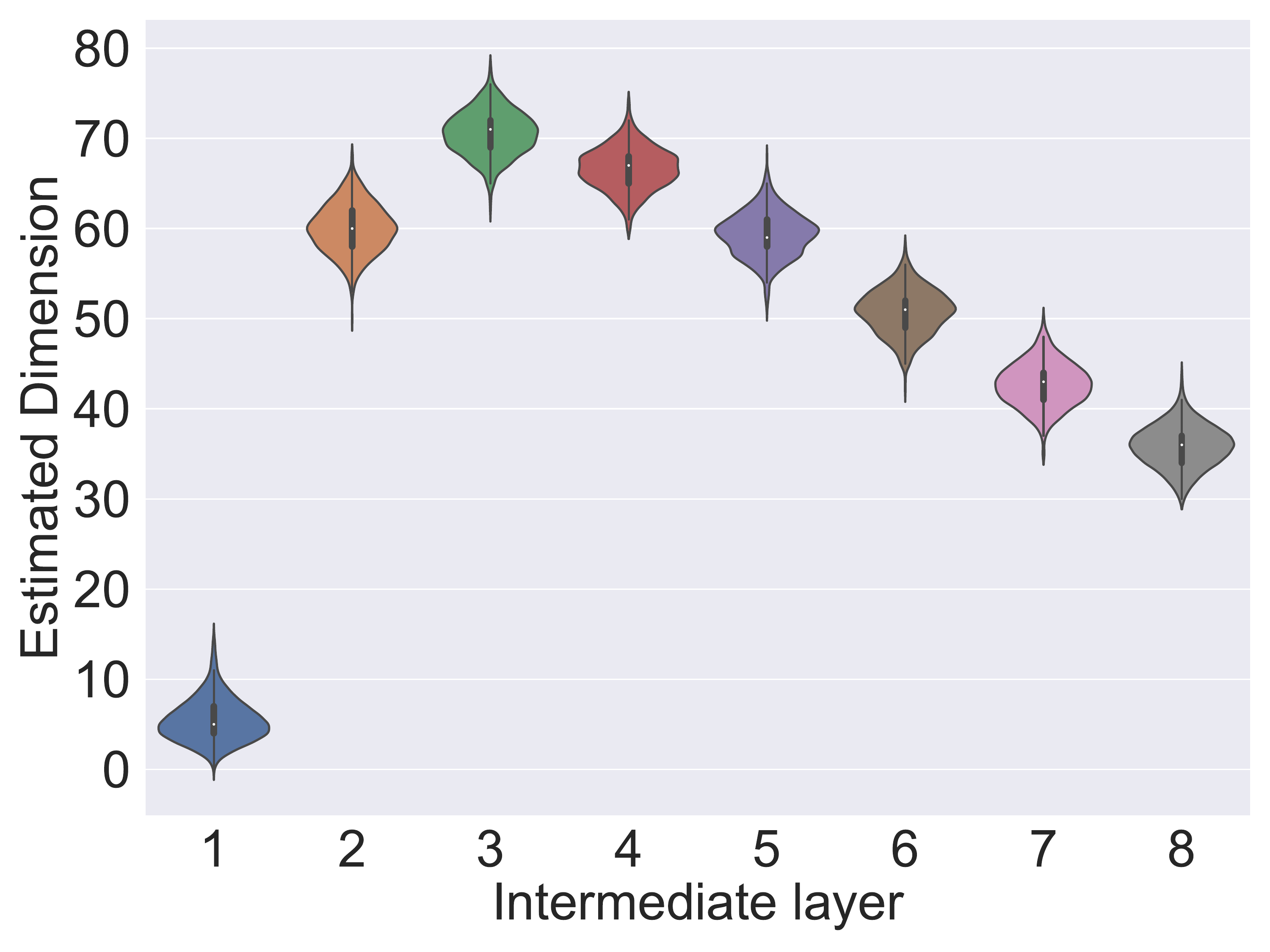}
    \caption{$\theta_{pre}=0.01$}
    \end{subfigure} 
    \caption{
    \textbf{Local Dimension Distribution} of the intermediate layers in the mapping network.
    }
    \label{fig:rank_stats}
    \vspace{-10pt}
\end{figure*}

\paragraph{Preprocessing}
This algorithm supposes the isotropic Gaussian noise on the Jacobian matrix. However, even considering the randomness of empirical covariance, the observed singular values of Jacobian matrix are \textit{too sparse} ($\sigma_{min} / \sigma_{max} \approx 10^{-9}$). Hence, the isotropic Gaussian assumption leads to the underestimation of the noise level, which causes the overestimation of intrinsic rank (In our experiments, estimated rank $> 200$ and $\sigma_{rank} \approx 0.003$). To address this problem, we introduce a simple preprocessing on the singular values of the Jacobian. Before applying the rank estimation algorithm, we filter out the singular values $\{\sigma_{i}^{\zvar}\}_{i}$ with $\{(\sigma_{i}^{\zvar})^{2} \leq \theta_{pre} \cdot \max_{i} (\sigma_{i}^{\zvar})^{2} \}$. We set $\theta_{pre} \in \{0.0005, 0.001, 0.005, 0.01\}$.

\paragraph{Validity of the estimated local dimension}
Our local dimension estimate is based on the Jacobian $\nabla_{\zvar}f$, where $f$ defines the latent manifold $\mset=f(\zset)$. Because $\nabla_{\zvar}f$ is the first order approximation of $f$, the estimated dimension might be improper for $\mset$.
Therefore, we designed and conducted the \textit{Off-manifold} experiment to assess the validity of the estimated local dimension. Intuitively, the intrinsic local dimension at $\wvar \in \mset$ is discovered by refining minor variations of the noisy latent manifold. \textit{The goal of Off-manifold experiment is to test whether the latent perturbation along the $k$-th coordinate axis stays inside $\mset$.} If the margin of $\mset$ at $\wvar$ in the $k$-th axis is large enough, then the $k$-th axis is needed to locally approximate $\mset$.
To be more specific, we solve the following optimization problem by Adam optimizer \citep{Adam} on MSE loss with a learning rate 0.005 for 1000 iterations for each $k$ with $\zvar_{0} = \zvar_{init}$:
\vspace{-3pt}
\begin{gather}
    \wvar_{init} = f(\zvar_{init}), \quad \wvar_{ptb} = \wvar_{init} + c \cdot \vvar^{\wvar_{init}}_{k}, \\
    \zvar_{opt} = \arg \min_{\zvar} \| \wvar_{ptb} - f(\zvar) \|^2.
\end{gather}
Here, $\vvar^{\wvar_{init}}_{k}$ denotes the $k$-th LB at $\wvar_{init}$. $\wvar_{init}$ is perturbed along $\vvar^{\wvar_{init}}_{k}$ because it is the tangent vector at $\wvar_{init}$ along $k$-th axis (Eq \ref{eq:jacobian_svd}).

We ran the Off-manifold experiments on $\wset$-space of StyleGAN2. Figure \ref{fig:off_mfd} shows the final objective $\| \wvar_{ptb} - f(\zvar_{opt}) \|^2$ after the optimization for each LB $\vvar^{\wvar_{init}}_{k}$ with $c=2$. The red vertical lines denote the estimated local dimension for each $\theta_{pre}$. (See the appendix for the Off-manifold results with various $c = \|\wvar_{ptb}-\wvar_{init}\|$.) 
The monotonous increase in the final loss shows that $f(\zvar_{opt})$ cannot approach close to $\wvar_{ptb}$. 
In other words, the margin in the $k$-th LB direction decreases as the index $k$ increases.
Although there is a dependency on the preprocessing threshold, the rank estimation algorithm chooses the principal part of local manifold around $\wvar_{init}$ without overestimates as desired. Particularly, the estimated rank with $\theta_{pre}=0.005$ appears to find a transition point of the final loss.

\begin{figure*}[t]
    \centering
    \begin{tabular}{cc}
        \begin{subfigure}{0.38\textwidth}
            \begin{center}
                \includegraphics[width=\columnwidth]{figure/traversal_2strip/366745668_Local_Basis_2dim_x0_y1_ptb5.0_sublayer8.pdf}
                \caption{Image Traversal along Axis 0 and 1 }
                \label{fig:rank_traveral_1}
            \end{center}
        \end{subfigure} & \\
            \begin{subfigure}{0.38\textwidth}
                \begin{center}
                    \includegraphics[width=\columnwidth]{figure/traversal_2strip/366745668_Local_Basis_svth0.0005_2dim_x78_y79_ptb5.0_sublayer8.pdf}
                    \caption{Image Traversal along Axis $(d-1)$ and $d$}
                    \label{fig:rank_traveral_2}
                \end{center}
            \end{subfigure}
            &  \quad
                \multirow{2}[2]{.35\textwidth}{ 
                \begin{subfigure}{.35\textwidth}
                    \vspace{-165pt}
                    \begin{center}
                      \includegraphics[width=\columnwidth]{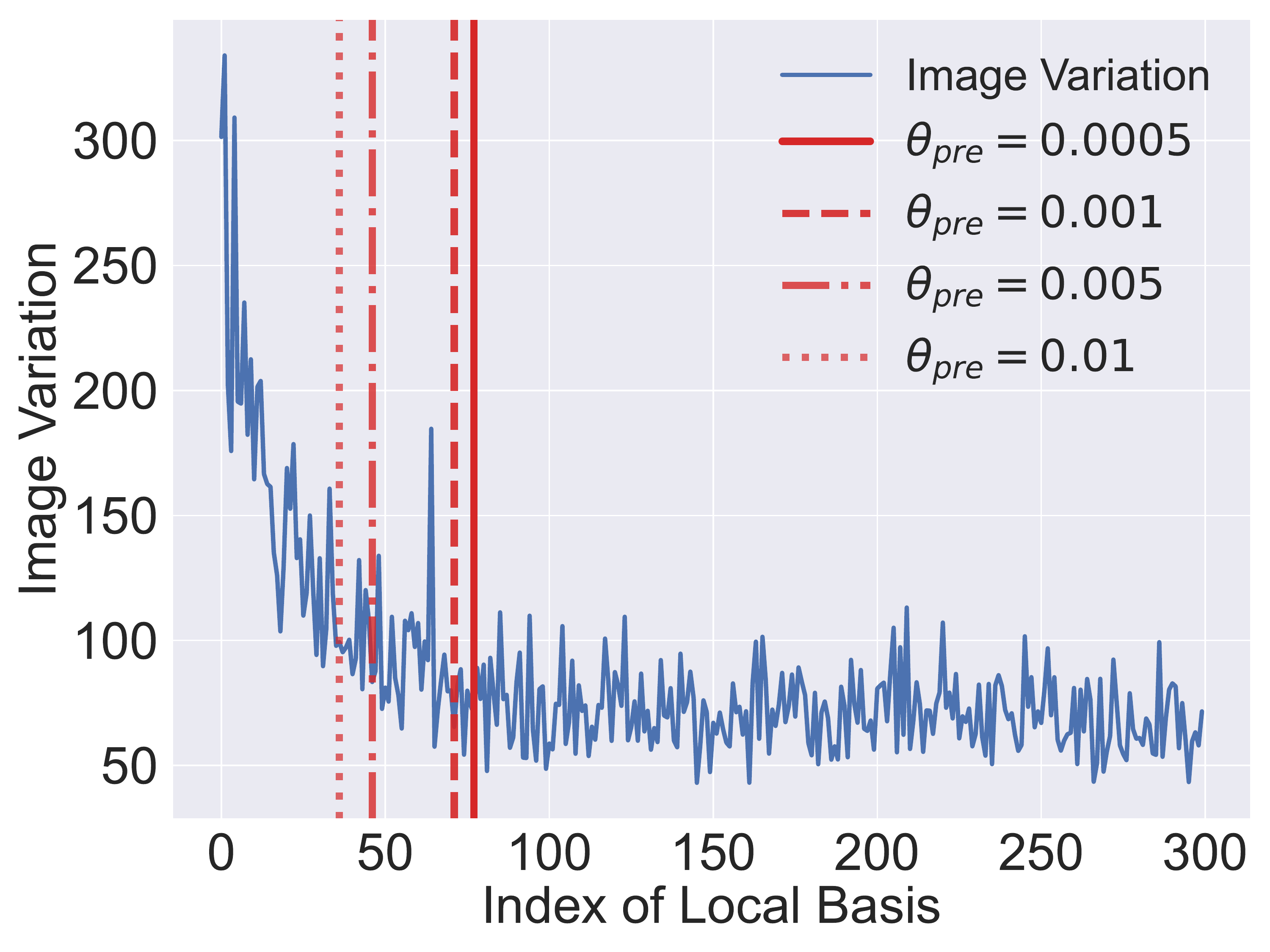}
                        \vspace{-12pt}
                        \caption{Image Variation along Singular Vectors}
                        \vspace{16pt}
                        \label{fig:rank_image_var}
                    \end{center}
                \end{subfigure}} \\
        \end{tabular}
    \vspace{-6pt}
    \caption{
    \textbf{Local Dimension Evaluation in Image Space} where $d$ denotes the estimated local dimension with $\theta_{pre}=0.01$. Fig \ref{fig:rank_image_var} shows the image variation intensity $\| \nabla_{\vvar_{i}^{\wvar}}g(\wvar) \|_{F}$ along each LB $\vvar_{i}$.
    }
    \label{fig:rank_image}
    \vspace{-10pt}
\end{figure*}

\subsection{Comparison to Previous Rank Estimation}
\paragraph{Sparsity Constraint}
LowRankGAN \citep{LowRankSubspace} introduced a convex optimization problem called Principal Component Pursuit (PCP) \citep{PCP} to find a low-rank factorization of Jacobian $(\nabla_{\zvar}f)(\zvar)$ (Eq \ref{eq:pcp}):
\begin{multline} \label{eq:pcp}
    \textrm{min}_{L, S} \quad \| L \|_{*} + \gamma \cdot \|S\|_{1} \\
    \textrm{s.t.} \,\, L+S = (\nabla_{\zvar}f)^{\intercal}(\zvar) \cdot (\nabla_{\zvar}f)(\zvar). 
\end{multline} 
where $\|L\|_{*} = \sum_{i} \sigma_{i}(L)$ is the nuclear norm, i.e. the sum of all singular values, $\| S \|_{1} = \sum_{i,j} |S_{i,j}|$, and $\gamma >0$ is a positive regularization parameter. 
PCP encourages the sparsity on corruption $S$ through $\ell_{1}$ regularizer.

However, we believe that this sparsity assumption is not adequate for finding the intrinsic rank of Jacobian. 
To test the validity of the sparsity assumption, we monitored how the low-rank representation $L$ changes as we vary the regularization parameter $n=1/\gamma$ as in \citep{LowRankSubspace} (Fig \ref{fig:sparse_obj}). 
The estimated rank decreases unceasingly without saturation as we increase $n$, i.e., refine the Jacobian stronger. 
We consider that the rank saturation should occur if this assumption is adequate for finding an \textit{intrinsic} rank because it implies regularization robustness.
But the low-rank factorization through PCP does not show any saturation until the Frobenius norm of corruption $\|S\|_{F}$ reaches over 50$\%$ of the initial matrix $ \| (\nabla_{\zvar}f)^{\intercal}(\zvar) \cdot (\nabla_{\zvar}f)(\zvar)\|_{F}$.

\vspace{-5pt}
\paragraph{Interpretation as Frobenious Norm}
Our algorithm can be interpreted as a Nuclear-Norm Penalization (NNP) problem (Eq \ref{eq:nnp}) for matrix denoising \citep{DenoisingFrob}.
This NNP framework is similar to PCP in LowRankGAN except for the regularization $\| E\|_{F}$. While PCP requires an iterative optimization of Alternating Directions Method of Multipliers (ADMM) \citep{ADMM_1, ADMM_2}, NNP provides an explicit closed-form solution $L^{*}$ through SVD, i.e., $(\nabla_{\zvar}f)(\zvar) = U \Sigma V^{\intercal}$:
\begin{align} \label{eq:nnp}
    &\textrm{min}_{L, S} \,\, \| L \|_{*} + \gamma \|E\|_{F} \quad \textrm{s.t.} \,\, L+E = (\nabla_{\zvar}f)(\zvar). \\
    &\Rightarrow L^{*} = U \left(\Sigma - \frac{1}{2\gamma}\cdot I \right)_{+} V^{\intercal} 
\end{align}
for $(M_{+})_{i,j} = \max\left(M_{i,j}\,, 0\right)$.
Therefore, the intrinsic rank by NNP is determined by a threshold $1/(2\gamma)$ for the singular values $\{\sigma_{i}^{\zvar}\}_{i}$ of Jacobian. Our algorithm can be interpreted as selecting this threshold via a series of hypothesis tests.

\subsection{Latent Space Analysis of StyleGAN}
We analyzed the intermediate layers of the mapping network in StyleGAN2 trained on FFHQ using our local dimension estimation (See Fig \ref{fig:architecture_stylegan} for StyleGAN architectures). First, Figure \ref{fig:rank_stats} shows the distribution of estimated local dimensions for 1k samples of each intermediate layer for each $\theta_{pre}$ (See the appendix for the rank statistics under all $\theta_{pre}$).
Note that the algorithm provides an unstable rank estimate on the most unsparse $1$st layer (Fig \ref{fig:sv_dist}) under the small $\theta_{pre} \in \{0.0005, 0.001 \}$. However, this phenomenon was not observed in the other layers. Hence, we focused our analysis on the layers with reasonable depth, i.e., from $3$ to $8$. Even though changing $\theta_{pre}$ results in an overall shift of the estimation, the trend and relative ordering between layers are the same. In accordance with Fig \ref{fig:sv_dist}, the intrinsic dimension monotonically decrease as the layer goes deeper. 

Second, we evaluated the estimated rank on image space (Fig \ref{fig:rank_image}). Figure \ref{fig:rank_traveral_1} and \ref{fig:rank_traveral_2} show the image traversal along the first two axes and the two axes $(d-1, d)$ around the estimated rank $d$ with $\theta_{pre}=0.01$ 
(See Fig \ref{fig:1dim_comparison_full} for comparison between axis $i < d$ and $i \gg d$.)
Fig \ref{fig:rank_image_var} presents the size of the directional derivative $\| \nabla_{\vvar_{i}^{\wvar}}g(\wvar) \|_{F}$ along the $i$-th LB $\vvar_{i}^{\wvar}$ at $\wvar$, estimated by the finite difference scheme.
The result shows that the estimated rank covers the major variations in the image space.
\citet{LocalBasis} proposed LB as the unsupervised local basis for the latent space in GANs, but did not provide an estimate on the number of meaningful perturbations.
One advantage of the unsupervised method over the supervised one for finding disentangled perturbation is that the discovered semantic is not restricted to pre-defined attributes. However, we cannot know the number of meaningful perturbations without additional inspections. Figure \ref{fig:rank_image} shows that the estimated dimension provides an upper bound on the number of these perturbations.


\begin{figure*}[t]
    \centering
    \begin{subfigure}[b]{0.3\textwidth}
    \centering
    \includegraphics[width=\columnwidth]{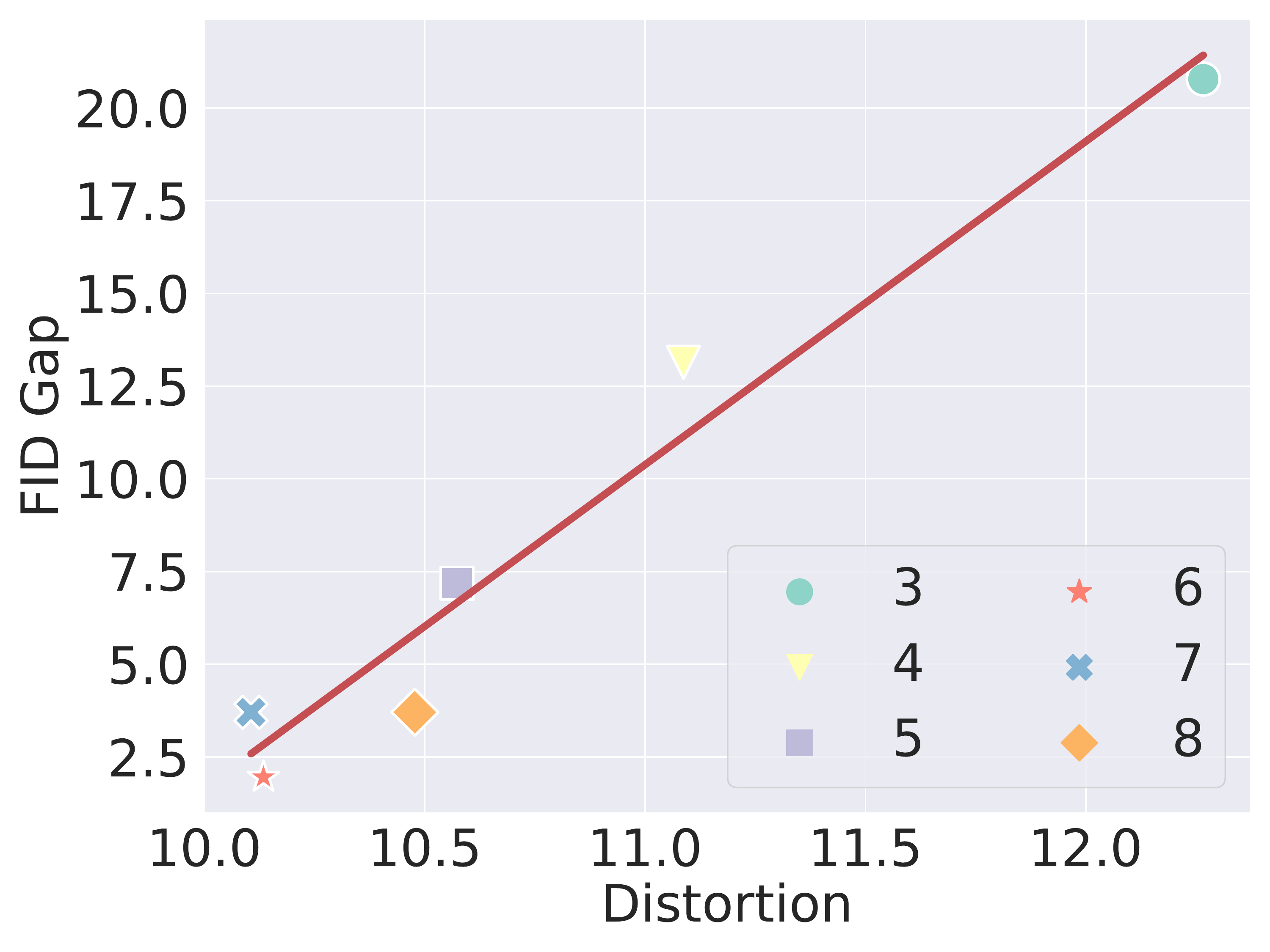}
    \caption{StyleGAN2-Cat - $\rho=0.98$}
    \end{subfigure}
    \quad
    \begin{subfigure}[b]{0.3\textwidth}
    \centering
    \includegraphics[width=\columnwidth]{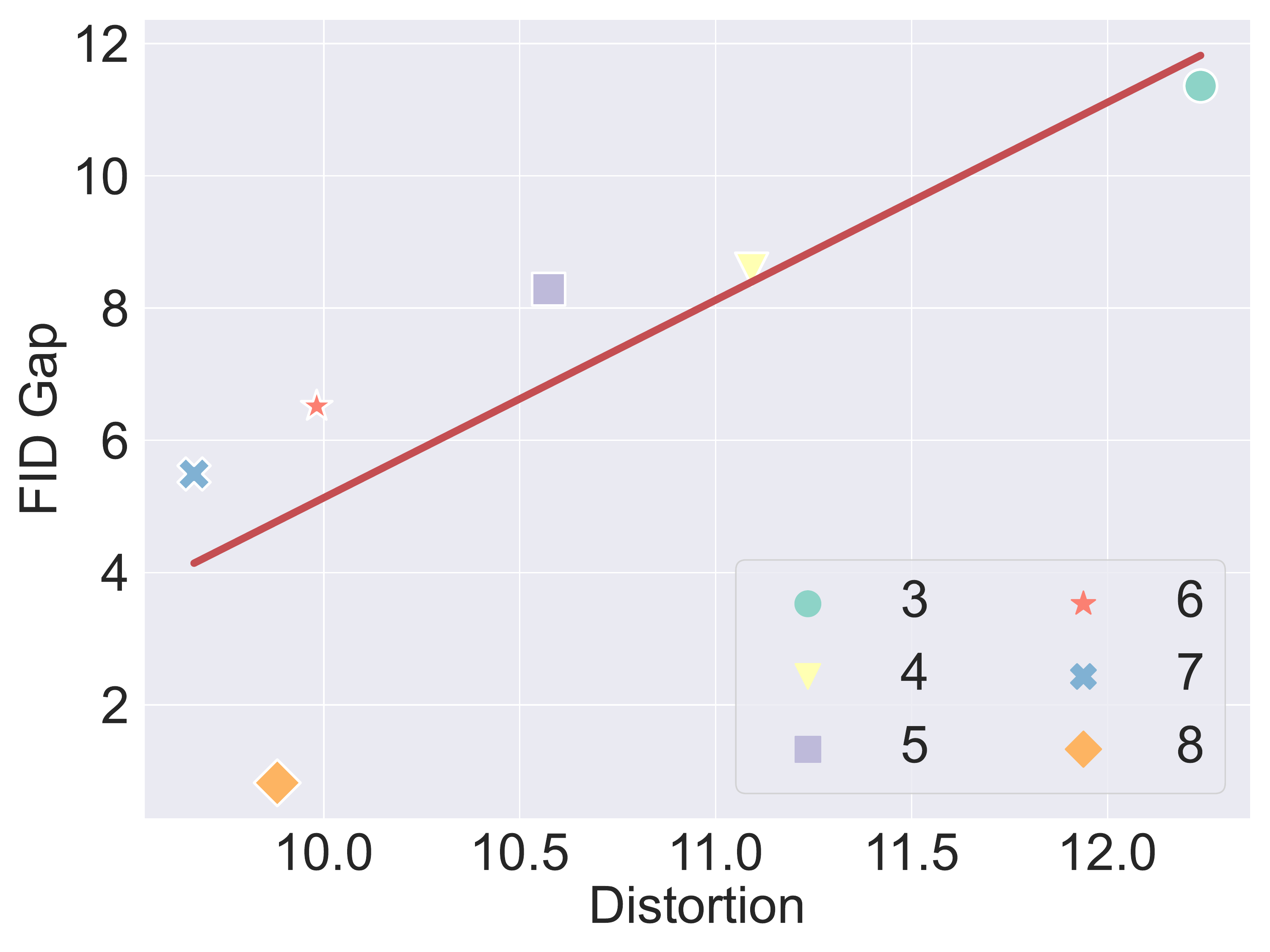}
    \caption{StyleGAN2-e - $\rho=0.81$}
    \end{subfigure}
    \quad
    \begin{subfigure}[b]{0.3\textwidth}
    \centering
    \includegraphics[width=\columnwidth]{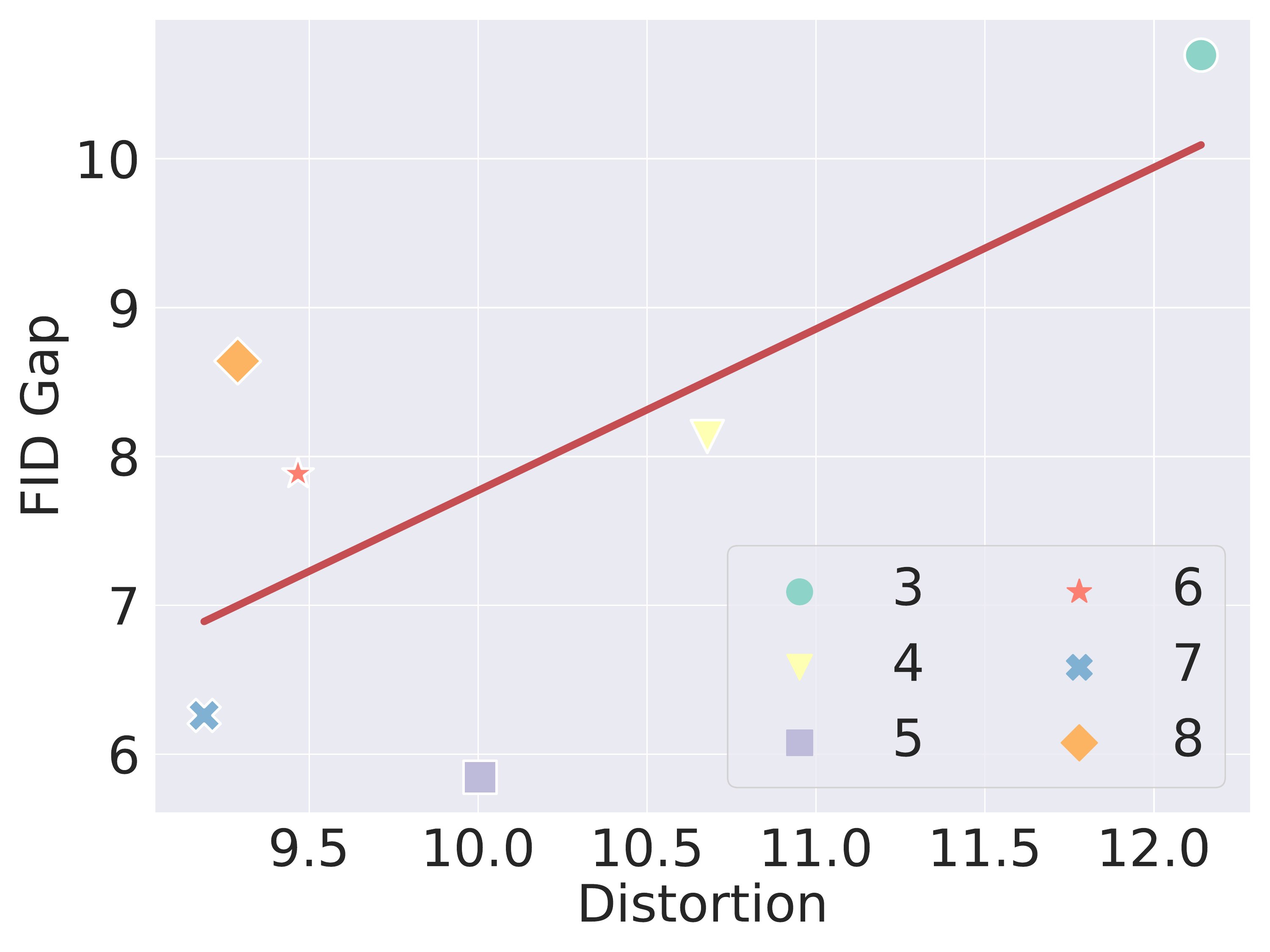}
    \caption{StyleGAN2 - $\rho=0.70$}
    \end{subfigure}
    \caption{
    \textbf{Correlation between Distortion metric ($\downarrow$) and FID gap ($\downarrow$)} when $\theta_{pre}=0.005$. 
    FID gap represents the difference between FID score of LB and GANSpace.
    Each point represents a $i$-th intermediate layer in the mapping network.
    }
    \label{fig:Dist_fid}
\end{figure*}

\begin{figure*}[t]
    \centering
    \begin{subfigure}[b]{0.3\textwidth}
    \centering
    \includegraphics[width=\columnwidth]{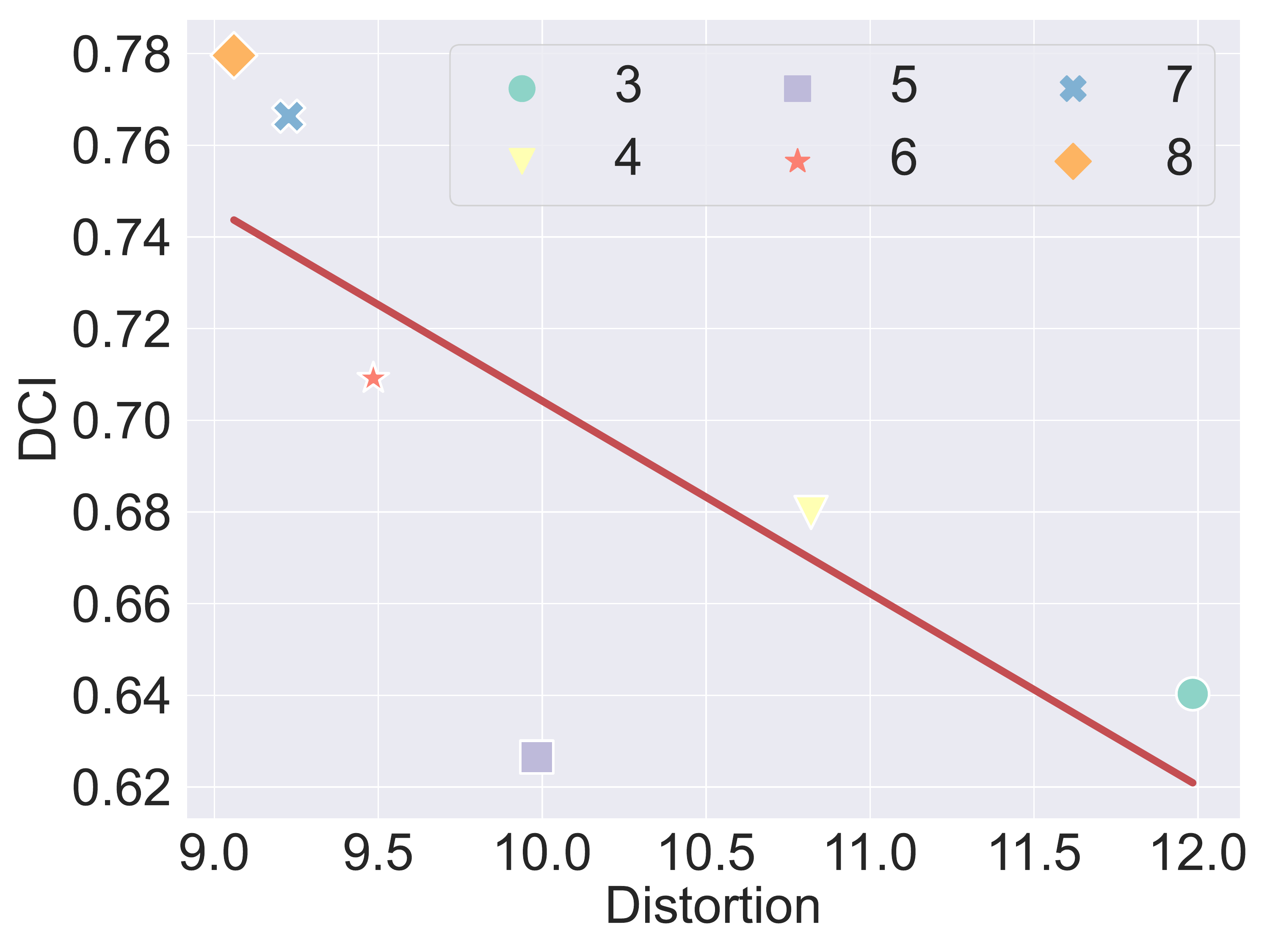}
    \caption{StyleGAN1 - $\rho=-0.74$}
    \end{subfigure}
    \quad
    \begin{subfigure}[b]{0.3\textwidth}
    \centering
    \includegraphics[width=\columnwidth]{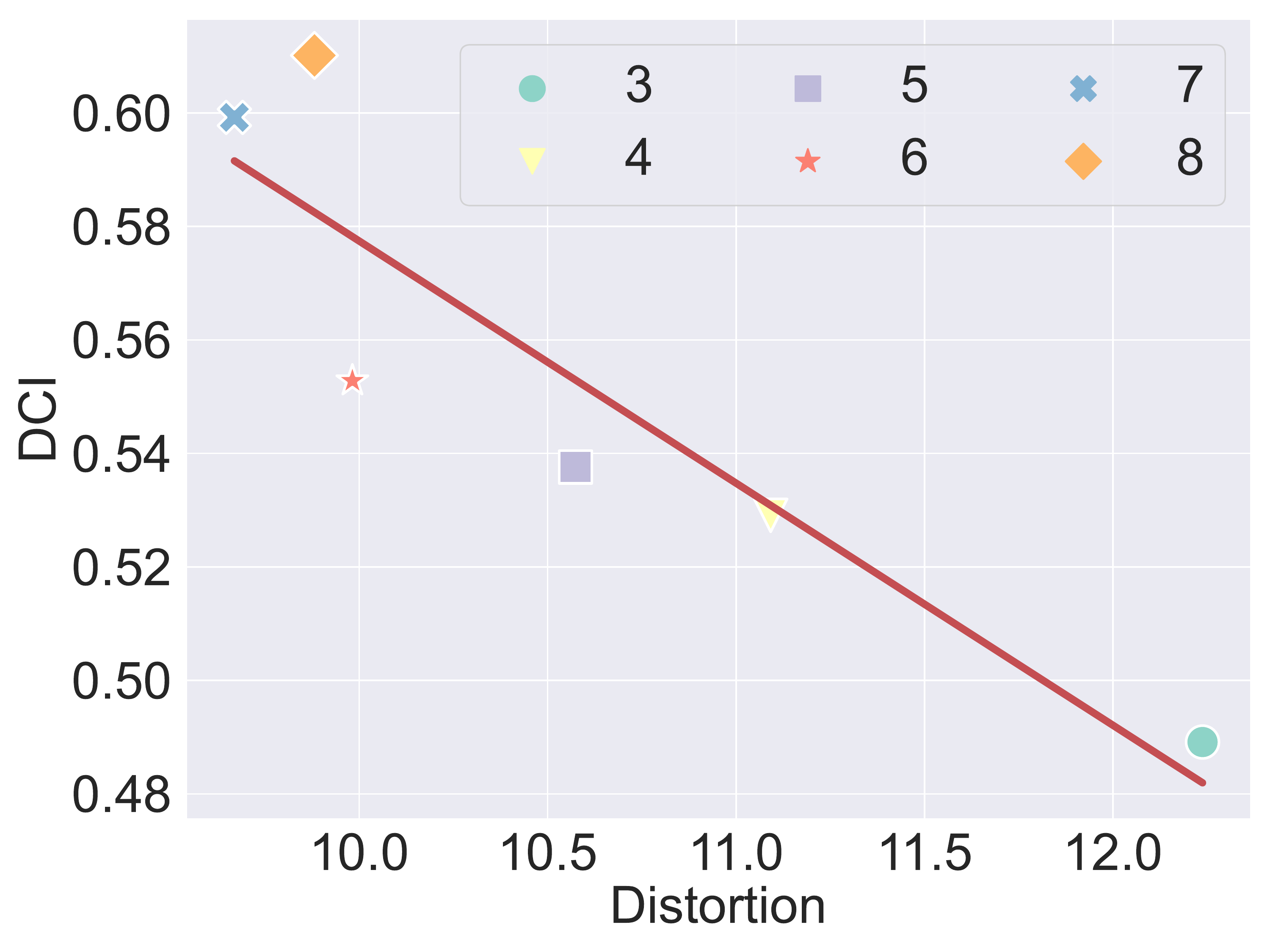}
    \caption{StyleGAN2-e - $\rho=-0.91$}
    \end{subfigure}
    \quad
    \begin{subfigure}[b]{0.3\textwidth}
    \centering
    \includegraphics[width=\columnwidth]{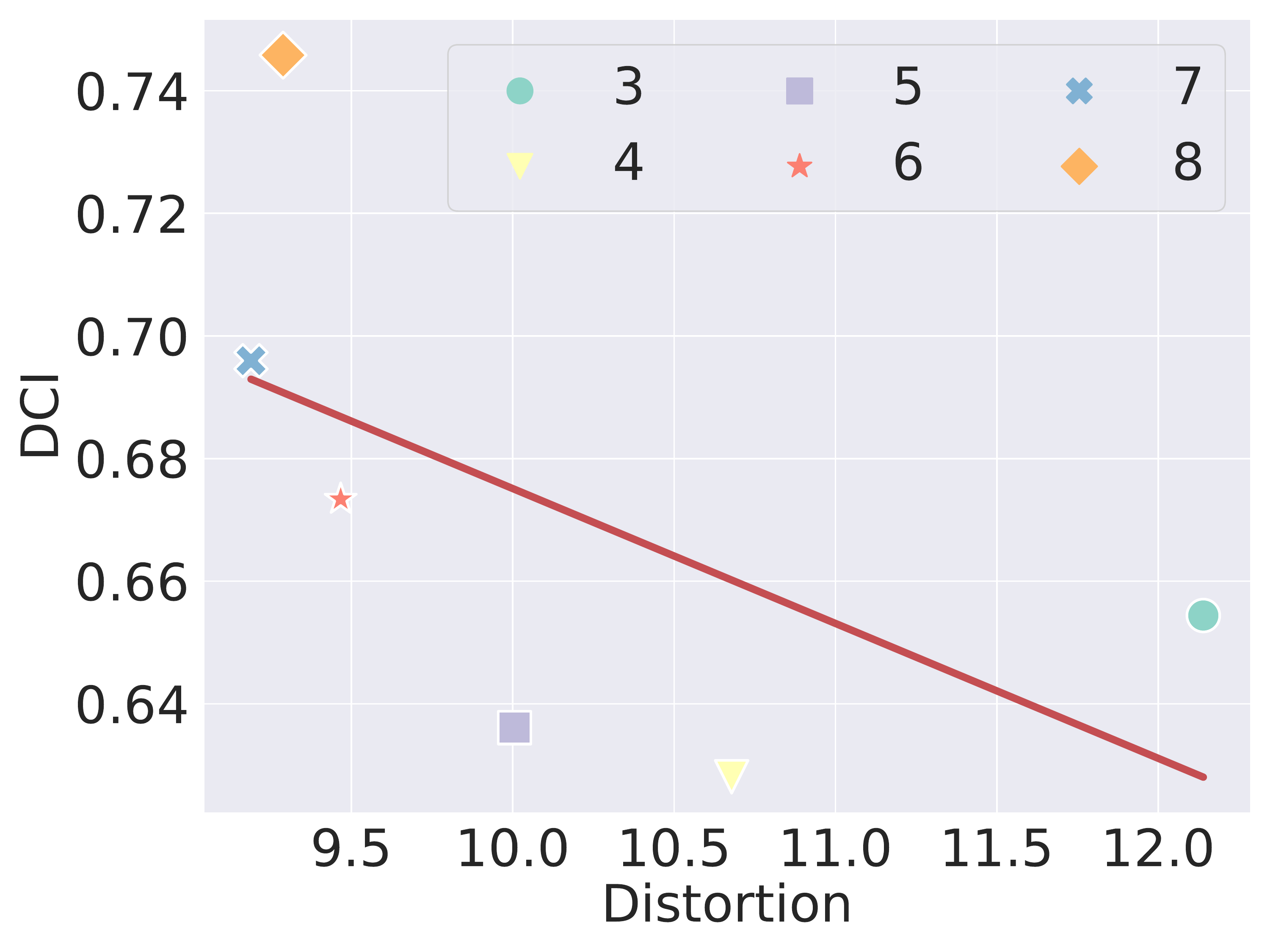}
    \caption{StyleGAN2 - $\rho=-0.57$}
    \end{subfigure} 
    \caption{
    \textbf{Correlation between Distortion metric ($\downarrow$) and DCI ($\uparrow$)} when $\theta_{pre}=0.005$. DCI \citep{DCIMetric} is a supervised disentanglement metric that requires attribute annotations. 
    }
    \label{fig:Dist_dci}
    \vspace{-10pt}
\end{figure*}

\section{Unsupervised Global Disentanglement Evaluation} \label{sec:distortion}
In this section, we investigate two closely related important questions on the disentanglement property of a GAN. 
\begin{enumerate}[topsep=-1pt, itemsep=-1pt]
    \item[Q1.] Can we evaluate the global-basis-compatibility of the latent space without posterior assessment? \citep{LocalBasis} 
    \item[Q2.] Can we evaluate the disentanglement without attribute annotations? \citep{ChallengingDisent}
\end{enumerate}
Here, the global-basis-compatibility represents the representability upper-bound of latent space that the optimal global basis can achieve. If a latent space has a low global-basis-compatibility, all global basis can attain constrained success on it.

These two questions are closely related because the ideal disentanglement includes a global basis representation where each element corresponds to the attribute-coordinate.
\textbf{In this paper, the global disentanglement property of a latent space is defined as this global representability along the attribute-coordinate.} To answer these questions, we propose an unsupervised global disentanglement metric, called \textit{\textbf{Distortion}}. 
We evaluate the global-basis-compatibility by the image fidelity (Q1) and the disentanglement by semantic factorization (Q2). 
Our experimental results show that our proposed metric achieves a high correlation with the global-basis-compatibility (Q1) and the supervised disentanglement score (Q2) on various StyleGANs. (See Fig \ref{fig:Dist_robust} in appendix for robustness of Distortion to $\theta_{pre}$.)

\paragraph{Distortion Score}

Intuitively, our global disentanglement score assesses the inconsistency of \textit{intrinsic} tangent space for each latent manifold. The framework of analyzing the semantic property of a latent space via its tangent space was first introduced in \citet{LocalBasis}. This framework was inspired by the observation that each basis vector (LB) of a tangent space corresponds to a local semantic perturbation. In this work, we develop this idea and propose a \textit{layer-wise score} for global disentanglement property.
As in \citet{LocalBasis}, we employ the Grassmannian \citep{boothby} metric to measure a distance between two tangent spaces. 
We chose the Geodesic Metric \citep{grassgeo} instead of the Projection Metric \citep{grassproj} because of its better discriminability (See Sec \ref{sec:why_geodesic_metric} for detail). Also, we revised the Geodesic Metric to be dimension-normalized because the local dimension changes according to its estimated region.

For two $k$-dimensional subspaces $W, W^{\prime}$ of $\sR^{n}$, let $M_W, M_{W^{\prime}}\in \sR^{n\times k}$ be the column-wise concatenation of orthonormal basis for $W, W^{\prime}$, respectively. Then, the dimension-normalized Geodesic Metric is defined as $d^{k}_{\mathrm{geo}}(W,W^{\prime})=\left(\frac{1}{k}\sum^{k}_{i=1}\theta^2_i\right)^{1/2}$ 
where $\theta_i =\cos^{-1}(\sigma_{i}(M_{W}^{\top} \, M_{W^{\prime}} ))$ denotes the $i$-th principal angle between $W$ and $W^{\prime}$ for $i$-th singular value $\sigma_{i}$. Then, Distortion score $\dset_{\mset}$ for the latent manifold $\mset$ is evaluated as follows:
\begin{enumerate}[topsep=-1pt, itemsep=-4pt]
    \item To assess the overall inconsistency of $\mset$, measure the Grassmannian distance between two intrinsic tangent spaces $T_{\wvar_{i}} \mset^{k_{i}}_{\wvar_{i}}$ (Eq \ref{eq:approx_mfd}) at two \textbf{random} $\wvar \in \mset$. For $k = \min(k_{1}, k_{2})$,
    \vspace{-4pt}
    \begin{equation*}
        I_{rand} = \mathbb{E}_{\substack{\zvar_{i} \sim p(\zvar),\\ \wvar_{i} = f(\zvar_{i})}}
        \left[d^{k}_{\mathrm{geo}}\left( \,T_{\wvar_{1}} \mset^{k}_{\wvar_{1}},T_{\wvar_{2}} \mset^{k}_{\wvar_{2}} \right)\right].
    \end{equation*}
    \item To normalize the overall inconsistency, measure the same Grassmannian distance between two \textbf{close} $\wvar \in \mset$ for $\epsilon = 0.1$
    \vspace{-4pt}
    \begin{equation*}
    I_{local} = \mathbb{E}_{\substack{\zvar_{1} \sim p(\zvar),\\ |\zvar_{2} - \zvar_{1}|=\epsilon}}
    \left[ d^{k}_{\mathrm{geo}}\left( \,T_{\wvar_{1}} \mset^{k}_{\wvar_{1}},T_{\wvar_{2}} \mset^{k}_{\wvar_{2}} \right) \right].
    \end{equation*}
    \item \textit{\textbf{Distortion}} of $\mset$ is defined as the relative inconsistency $\mathcal{D}_{\mset}=I_{rand} / I_{local}$.
\end{enumerate}

\paragraph{Distortion and Global Disentanglement}
In this paragraph, we clarify \textit{why the globally disentangled latent space shows a low Distortion score.}
Assume a latent space $\mset$ is globally disentangled. 
Then, there exists an optimal global basis of $\mset$, where each basis vector corresponds to an image attribute on the entire $\mset$. By definition, this optimal global basis is the local basis at all latent variables. Assuming that LB finds the local basis \cite{LocalBasis}, each global basis vector would correspond to one LB vector at each latent variable.
In this regard, our local dimension estimation finds a principal subset of LB, which includes these corresponding basis vectors. 
In conclusion, if the latent space is globally disentangled, this principal set of LB at each latent variable would contain the common global basis vectors. Hence, the intrinsic tangent spaces would contain the common subspace generated by this common basis, which leads to a small Grassmannian metric between them. Therefore, the global disentanglement of the latent space leads to a low Distortion score. 


\begin{figure*}[t]
    \centering
    \begin{subfigure}[b]{0.27\textwidth}
    \centering
    \includegraphics[width=\columnwidth]{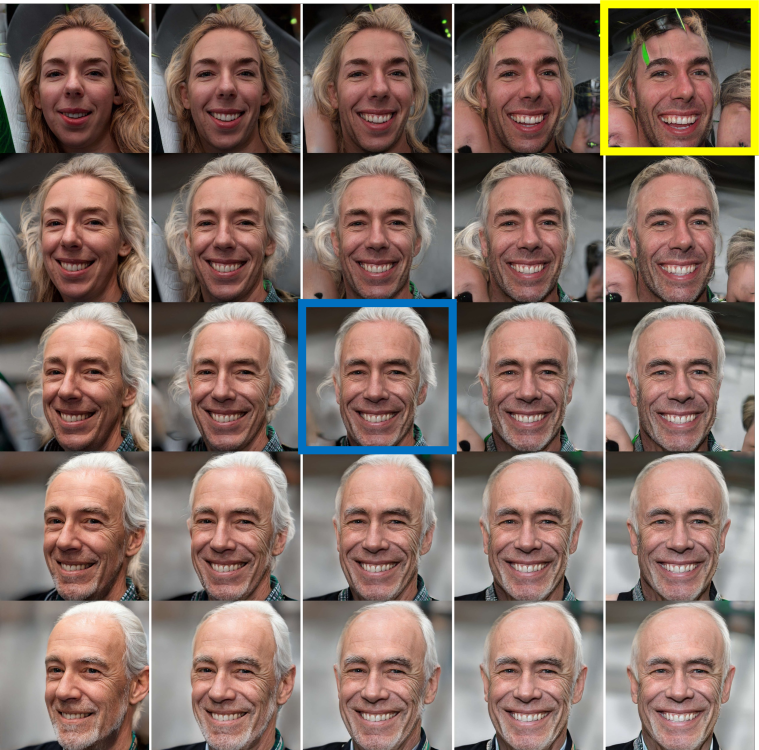}
    \caption{Max-distorted layer 3}
    \end{subfigure}
    \quad
    \begin{subfigure}[b]{0.27\textwidth}
    \centering
    \includegraphics[width=\columnwidth]{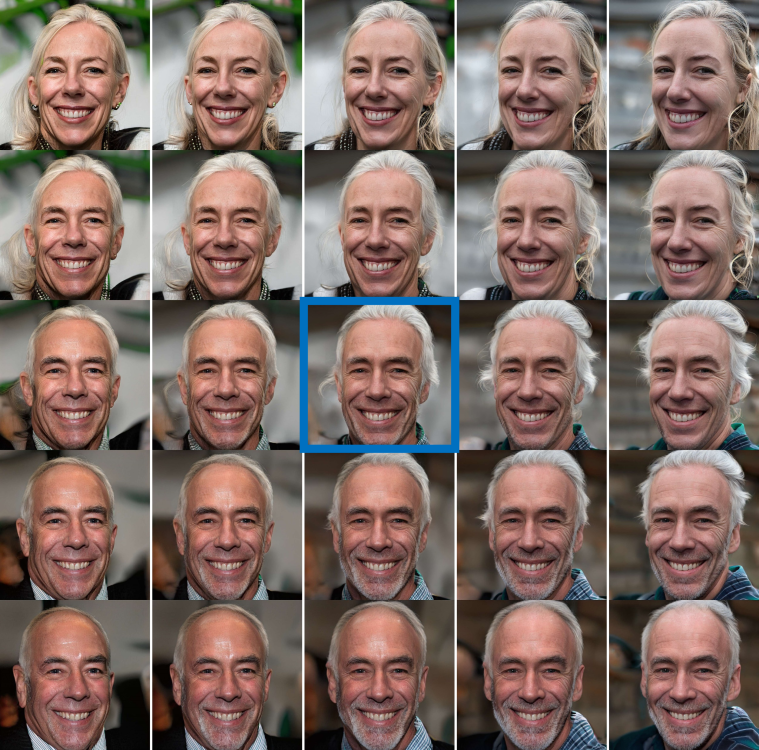}
    \caption{Min-distorted layer 7}
    \end{subfigure}
    \quad
    \begin{subfigure}[b]{0.27\textwidth}
    \centering
    \includegraphics[width=\columnwidth]{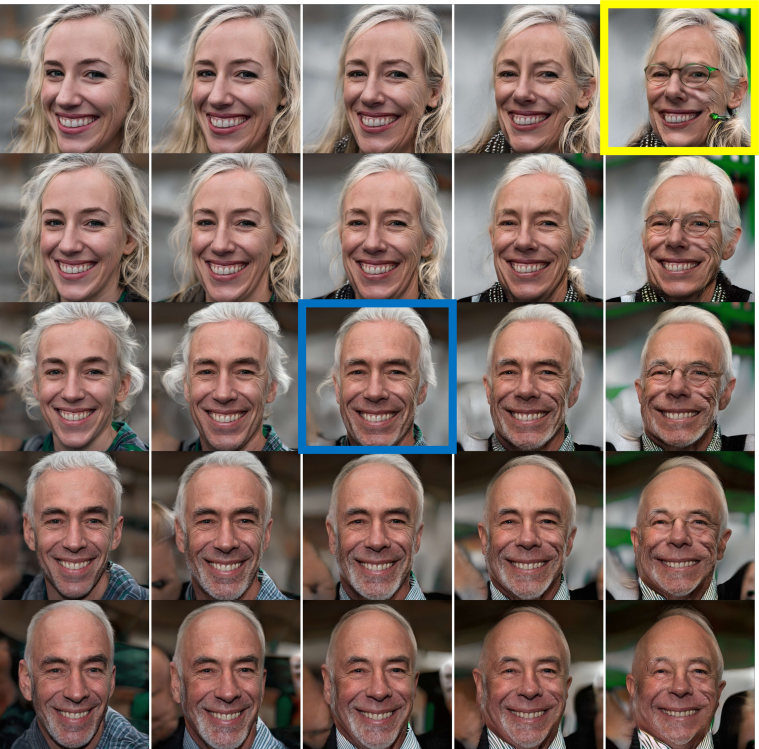}
    \caption{Layer 8 ($\wset$-space)}
    \end{subfigure}
    \caption{
    \textbf{Subspace Traversal on the intermediate layers} along the global basis. The upper-right corner of max-distorted layer 3 and layer 8 show visual artifacts. However, the min-distorted layer 7 does not show such a failure. The initial image (center) is traversed along the $1$st (horizontal) and $2$nd (vertical) components of GANSpace with a perturbation intensity 4.
    }
    \label{fig:subspace_traversal}
    \vspace{-10pt}
\end{figure*}

\vspace{-5pt}
\paragraph{Global-Basis-Compatibility}
We tested whether Distortion $\mathcal{D}_{\mset}$ is meaningful in estimating the global-basis-compatibility. We chose GANSpace \citep{harkonen2020ganspace} as a reference global basis because of its broad applicability. 
\textit{As a measure of global-basis-compatibility, we adopted FID \citep{heusel2017gans} gap between LB and GANSpace.} Here, we interpret FID score under LB traversal as the optimal image fidelity that the latent space can achieve because LB is the basis of tangent space.
Thus, FID gap between LB and GANSpace represents the difference between the optimal image fidelity and the fidelity achieved by a global basis. Therefore, we selected this FID gap for assessing the global-basis-compatibility.
FID is measured for 50k samples of perturbed images along the 1st component of LB and GANSpace, respectively. 
Distortion metric is tested on StyleGAN2 on LSUN Cat, StyleGAN2 with configs E and F \citep{karras2020analyzing} on FFHQ to test the generalizability of correlation to the global-basis-compatibility. StyleGAN2 in Fig \ref{fig:Dist_fid} denotes StyleGAN2 with config F because config F is the usual StyleGAN2 model. The perturbation intensity is set to 5 in LSUN Cat and 3 in FFHQ. Distortion metric shows a strong positive correlation of 0.98, 0.81 and 0.70 to FID gap in Fig \ref{fig:Dist_fid}.
(See Sec \ref{sec:corr_robust} for correlations on other $\theta_{pre}$.)
This result demonstrates that Distortion metric can be an unsupervised criterion for selecting a latent space with high global-basis-compatibility.
This suggests that, before finding a global basis, we can use Distortion metric as a prior investigation for selecting a proper target latent space.

\vspace{-5pt}
\paragraph{Disentanglement Score}
We assessed a correlation between the unsupervised Distortion metric and a supervised disentanglement score. Following \citet{wu2020stylespace}, we adopted DCI score \citep{DCIMetric} as the supervised disentanglement score for evaluation, and employed 40 binary attribute classifiers pre-trained on CelebA \citep{CelebA} to label generated images.
Each DCI score is assessed on 10k samples of latent variables with the corresponding attribute labels. 
In Fig \ref{fig:Dist_dci}, StyleGAN1, StyleGAN2-e, and StyleGAN2 refer to StyleGAN1 and StyleGAN2s with config E and F trained on FFHQ. Note that DCI experiments are all performed on FFHQ because the DCI score requires attribute annotations. DCI and Distortion metrics show a strong negative correlation on StyleGAN1 and StyleGAN2-e. The correlation is relatively moderate on StyleGAN2. 
This moderate correlation is because Distortion metric is based on the Grassmannian metric. The Grassmannian metric measures the distance between tangent spaces, while DCI is based on their specific basis.
Even if the tangent space is identical so that Distortion becomes zero, DCI can have a relatively low value depending on the choice of basis. Hence, in StyleGAN2, the high-distorted layers showed low DCI scores, but the low-distorted layers showed relatively high variance in DCI score. 
Nevertheless, the strong correlation observed in the other two experiments suggests that, in practice, the basis vector corresponding to a specific attribute has a limited variance in a given latent space.
Therefore, Distortion metric can be an unsupervised indicator for the supervised disentanglement score. 

\vspace{-5pt}
\paragraph{Traversal Comparison}
In all StyleGAN models, layer 7 and layer 8 ($\wset$-space) achieved the smallest Distortion metric (Fig \ref{fig:Dist_fid} and \ref{fig:Dist_dci}). These results explain the superior disentanglement of renowned $\wset$-space. Moreover, our Distortion suggests that layer 7 can be an alternative to $\wset$-space.
Our global-basis-compatibility results imply that the min-distorted layer can present better \textbf{image fidelity} under the global basis traversal. For a visual comparison, we observed the subspace traversal \citep{LocalBasis} along the global basis on the max-distorted layer 3, min-distorted layer 7, and $\wset$-space of StyleGAN2 on FFHQ. The subspace traversal visualizes the two-dimensional latent perturbation. 
Hence, the subspace traversal can reveal the deviation from the tangent space more clearly than the linear traversal, because the deviation is assessed in two dimensions. Note that the tangent space is the optimal perturbation in image fidelity.
In Fig \ref{fig:subspace_traversal}, the global basis shows visual artifacts at the corners in the subspace traversal on the max-distorted layer 3 and $\wset$-space. Nevertheless, the min-distorted layer 7 shows the stable traversal without any failure. 
This result proves that comparing Distortion scores can be a criterion for selecting a better latent space with higher global-basis-compatibility. (See Sec \ref{sec:additional_traversal} for additional image fidelity comparisons and Sec \ref{sec:additional_semantic} for \textbf{semantic factorization} comparisons.)


\section{Conclusion}
In this paper, we proposed a local intrinsic dimension estimation algorithm for the intermediate latent space in a pre-trained GAN. Using this algorithm, we analyzed the intermediate layers in the mapping network of StyleGANs on various datasets. Moreover, we suggested an unsupervised global disentanglement metric called Distortion. The analysis of the mapping network demonstrates that Distortion metric shows a high correlation between the global-basis-compatibility and disentanglement score. 
Although finding an optimal preprocessing hyperparameter $\theta_{pre}$ was beyond the scope of this work, the proposed metric showed robustness to the hyperparameter. Moreover, our local dimension estimation has the potential to be applied to feature spaces of various models. This kind of research would be an interesting future work.


\bibliography{references}

\begin{thebibliography}{42}
\providecommand{\natexlab}[1]{#1}
\providecommand{\url}[1]{\texttt{#1}}
\expandafter\ifx\csname urlstyle\endcsname\relax
  \providecommand{\doi}[1]{doi: #1}\else
  \providecommand{\doi}{doi: \begingroup \urlstyle{rm}\Url}\fi

\bibitem[Abdal et~al.(2019)Abdal, Qin, and Wonka]{abdal2019image2stylegan}
Abdal, R., Qin, Y., and Wonka, P.
\newblock Image2stylegan: How to embed images into the stylegan latent space?
\newblock In \emph{Proceedings of the IEEE/CVF International Conference on
  Computer Vision}, pp.\  4432--4441, 2019.

\bibitem[Abdal et~al.(2021)Abdal, Zhu, Mitra, and Wonka]{abdal2021styleflow}
Abdal, R., Zhu, P., Mitra, N.~J., and Wonka, P.
\newblock Styleflow: Attribute-conditioned exploration of stylegan-generated
  images using conditional continuous normalizing flows.
\newblock \emph{ACM Transactions on Graphics (TOG)}, 40\penalty0 (3):\penalty0
  1--21, 2021.

\bibitem[Bengio et~al.(2013)Bengio, Courville, and
  Vincent]{bengio2013representation}
Bengio, Y., Courville, A., and Vincent, P.
\newblock Representation learning: A review and new perspectives.
\newblock \emph{IEEE transactions on pattern analysis and machine
  intelligence}, 35\penalty0 (8):\penalty0 1798--1828, 2013.

\bibitem[Boothby(1986)]{boothby}
Boothby, W.~M.
\newblock \emph{An introduction to differentiable manifolds and Riemannian
  geometry}.
\newblock Academic press, 1986.

\bibitem[Boyd et~al.(2011)Boyd, Parikh, Chu, Peleato, Eckstein, et~al.]{ADMM_1}
Boyd, S., Parikh, N., Chu, E., Peleato, B., Eckstein, J., et~al.
\newblock Distributed optimization and statistical learning via the alternating
  direction method of multipliers.
\newblock \emph{Foundations and Trends{\textregistered} in Machine learning},
  3\penalty0 (1):\penalty0 1--122, 2011.

\bibitem[Brock et~al.(2018)Brock, Donahue, and Simonyan]{brock2018large}
Brock, A., Donahue, J., and Simonyan, K.
\newblock Large scale gan training for high fidelity natural image synthesis.
\newblock In \emph{International Conference on Learning Representations}, 2018.

\bibitem[Cand\`{e}s et~al.(2011)Cand\`{e}s, Li, Ma, and Wright]{PCP}
Cand\`{e}s, E.~J., Li, X., Ma, Y., and Wright, J.
\newblock Robust principal component analysis?
\newblock \emph{Journal of the ACM (JACM)}, 58\penalty0 (3), 2011.
\newblock ISSN 0004-5411.
\newblock \doi{10.1145/1970392.1970395}.

\bibitem[Choi et~al.(2022{\natexlab{a}})Choi, Hwang, Cho, and
  Kang]{frechetbasis}
Choi, J., Hwang, G., Cho, H., and Kang, M.
\newblock Finding the global semantic representation in gan through frechet
  mean.
\newblock \emph{arXiv preprint arXiv:2210.05509}, 2022{\natexlab{a}}.

\bibitem[Choi et~al.(2022{\natexlab{b}})Choi, Lee, Yoon, Park, Hwang, and
  Kang]{LocalBasis}
Choi, J., Lee, J., Yoon, C., Park, J.~H., Hwang, G., and Kang, M.
\newblock Do not escape from the manifold: Discovering the local coordinates on
  the latent space of {GAN}s.
\newblock In \emph{International Conference on Learning Representations},
  2022{\natexlab{b}}.

\bibitem[Donoho \& Gavish(2014)Donoho and Gavish]{DenoisingFrob}
Donoho, D. and Gavish, M.
\newblock Minimax risk of matrix denoising by singular value thresholding.
\newblock \emph{The Annals of Statistics}, 42\penalty0 (6):\penalty0
  2413--2440, 2014.

\bibitem[Eastwood \& Williams(2018)Eastwood and Williams]{DCIMetric}
Eastwood, C. and Williams, C.~K.
\newblock A framework for the quantitative evaluation of disentangled
  representations.
\newblock In \emph{International Conference on Learning Representations}, 2018.

\bibitem[Eckart \& Young(1936)Eckart and Young]{SVD_lowrank}
Eckart, C. and Young, G.
\newblock The approximation of one matrix by another of lower rank.
\newblock \emph{Psychometrika}, 1\penalty0 (3):\penalty0 211--218, 1936.

\bibitem[Goetschalckx et~al.(2019)Goetschalckx, Andonian, Oliva, and
  Isola]{goetschalckx2019ganalyze}
Goetschalckx, L., Andonian, A., Oliva, A., and Isola, P.
\newblock Ganalyze: Toward visual definitions of cognitive image properties.
\newblock In \emph{Proceedings of the IEEE/CVF International Conference on
  Computer Vision}, pp.\  5744--5753, 2019.

\bibitem[Goodfellow et~al.(2014)Goodfellow, Pouget-Abadie, Mirza, Xu,
  Warde-Farley, Ozair, Courville, and Bengio]{goodfellow2014generative}
Goodfellow, I.~J., Pouget-Abadie, J., Mirza, M., Xu, B., Warde-Farley, D.,
  Ozair, S., Courville, A.~C., and Bengio, Y.
\newblock Generative adversarial nets.
\newblock In \emph{NIPS}, 2014.

\bibitem[H{\"a}rk{\"o}nen et~al.(2020)H{\"a}rk{\"o}nen, Hertzmann, Lehtinen,
  and Paris]{harkonen2020ganspace}
H{\"a}rk{\"o}nen, E., Hertzmann, A., Lehtinen, J., and Paris, S.
\newblock Ganspace: Discovering interpretable gan controls.
\newblock \emph{Advances in Neural Information Processing Systems}, 33, 2020.

\bibitem[Heusel et~al.(2017)Heusel, Ramsauer, Unterthiner, Nessler, and
  Hochreiter]{heusel2017gans}
Heusel, M., Ramsauer, H., Unterthiner, T., Nessler, B., and Hochreiter, S.
\newblock Gans trained by a two time-scale update rule converge to a local nash
  equilibrium.
\newblock \emph{Advances in neural information processing systems}, 30, 2017.

\bibitem[Jahanian et~al.(2019)Jahanian, Chai, and
  Isola]{jahanian2019steerability}
Jahanian, A., Chai, L., and Isola, P.
\newblock On the" steerability" of generative adversarial networks.
\newblock In \emph{International Conference on Learning Representations}, 2019.

\bibitem[Johnstone(2001)]{TracyWidom}
Johnstone, I.~M.
\newblock On the distribution of the largest eigenvalue in principal components
  analysis.
\newblock \emph{The Annals of statistics}, 29\penalty0 (2):\penalty0 295--327,
  2001.

\bibitem[Karras et~al.(2018)Karras, Aila, Laine, and
  Lehtinen]{karras2018progressive}
Karras, T., Aila, T., Laine, S., and Lehtinen, J.
\newblock Progressive growing of gans for improved quality, stability, and
  variation.
\newblock In \emph{International Conference on Learning Representations}, 2018.

\bibitem[Karras et~al.(2019)Karras, Laine, and Aila]{karras2019style}
Karras, T., Laine, S., and Aila, T.
\newblock A style-based generator architecture for generative adversarial
  networks.
\newblock In \emph{Proceedings of the IEEE/CVF Conference on Computer Vision
  and Pattern Recognition}, pp.\  4401--4410, 2019.

\bibitem[Karras et~al.(2020{\natexlab{a}})Karras, Aittala, Hellsten, Laine,
  Lehtinen, and Aila]{karras2020training}
Karras, T., Aittala, M., Hellsten, J., Laine, S., Lehtinen, J., and Aila, T.
\newblock Training generative adversarial networks with limited data.
\newblock \emph{arXiv preprint arXiv:2006.06676}, 2020{\natexlab{a}}.

\bibitem[Karras et~al.(2020{\natexlab{b}})Karras, Laine, Aittala, Hellsten,
  Lehtinen, and Aila]{karras2020analyzing}
Karras, T., Laine, S., Aittala, M., Hellsten, J., Lehtinen, J., and Aila, T.
\newblock Analyzing and improving the image quality of stylegan.
\newblock In \emph{Proceedings of the IEEE/CVF Conference on Computer Vision
  and Pattern Recognition}, pp.\  8110--8119, 2020{\natexlab{b}}.

\bibitem[Karras et~al.(2021)Karras, Aittala, Laine, H{\"a}rk{\"o}nen, Hellsten,
  Lehtinen, and Aila]{AliasFreeGAN}
Karras, T., Aittala, M., Laine, S., H{\"a}rk{\"o}nen, E., Hellsten, J.,
  Lehtinen, J., and Aila, T.
\newblock Alias-free generative adversarial networks.
\newblock \emph{Advances in Neural Information Processing Systems}, 34, 2021.

\bibitem[Karrasch(2017)]{grassproj}
Karrasch, D.
\newblock An introduction to grassmann manifolds and their matrix
  representation.
\newblock 2017.

\bibitem[Kingma \& Ba(2015)Kingma and Ba]{Adam}
Kingma, D.~P. and Ba, J.
\newblock Adam: A method for stochastic optimization.
\newblock In \emph{International Conference on Learning Representations
  (ICLR)}, 2015.

\bibitem[Kritchman \& Nadler(2008)Kritchman and Nadler]{RankEstimate}
Kritchman, S. and Nadler, B.
\newblock Determining the number of components in a factor model from limited
  noisy data.
\newblock \emph{Chemometrics and Intelligent Laboratory Systems}, 94\penalty0
  (1):\penalty0 19--32, 2008.

\bibitem[Lin et~al.(2010)Lin, Chen, and Ma]{ADMM_2}
Lin, Z., Chen, M., and Ma, Y.
\newblock The augmented lagrange multiplier method for exact recovery of
  corrupted low-rank matrices.
\newblock \emph{arXiv preprint arXiv:1009.5055}, 2010.

\bibitem[Liu et~al.(2015)Liu, Luo, Wang, and Tang]{CelebA}
Liu, Z., Luo, P., Wang, X., and Tang, X.
\newblock Deep learning face attributes in the wild.
\newblock In \emph{Proceedings of International Conference on Computer Vision
  (ICCV)}, December 2015.

\bibitem[Locatello et~al.(2019)Locatello, Bauer, Lucic, Raetsch, Gelly,
  Sch{\"o}lkopf, and Bachem]{ChallengingDisent}
Locatello, F., Bauer, S., Lucic, M., Raetsch, G., Gelly, S., Sch{\"o}lkopf, B.,
  and Bachem, O.
\newblock Challenging common assumptions in the unsupervised learning of
  disentangled representations.
\newblock In \emph{International Conference on Machine Learning}, pp.\
  4114--4124, 2019.

\bibitem[Patashnik et~al.(2021)Patashnik, Wu, Shechtman, Cohen-Or, and
  Lischinski]{patashnik2021styleclip}
Patashnik, O., Wu, Z., Shechtman, E., Cohen-Or, D., and Lischinski, D.
\newblock Styleclip: Text-driven manipulation of stylegan imagery.
\newblock \emph{arXiv preprint arXiv:2103.17249}, 2021.

\bibitem[Plumerault et~al.(2020)Plumerault, Borgne, and
  Hudelot]{plumerault2020controlling}
Plumerault, A., Borgne, H.~L., and Hudelot, C.
\newblock Controlling generative models with continuous factors of variations.
\newblock In \emph{International Conference on Learning Representations}, 2020.

\bibitem[Radford et~al.(2016)Radford, Metz, and
  Chintala]{radford2015unsupervised}
Radford, A., Metz, L., and Chintala, S.
\newblock Unsupervised representation learning with deep convolutional
  generative adversarial networks.
\newblock In \emph{Proceedings of the International Conference on Learning
  Representations (ICLR)}, 2016.

\bibitem[Ramesh et~al.(2018)Ramesh, Choi, and LeCun]{ramesh2018spectral}
Ramesh, A., Choi, Y., and LeCun, Y.
\newblock A spectral regularizer for unsupervised disentanglement.
\newblock \emph{arXiv preprint arXiv:1812.01161}, 2018.

\bibitem[Sauer et~al.(2022)Sauer, Schwarz, and Geiger]{StyleGANXL}
Sauer, A., Schwarz, K., and Geiger, A.
\newblock Stylegan-xl: Scaling stylegan to large diverse datasets.
\newblock In \emph{Special Interest Group on Computer Graphics and Interactive
  Techniques Conference Proceedings}, pp.\  1--10, 2022.

\bibitem[Shen \& Zhou(2021)Shen and Zhou]{shen2020closed}
Shen, Y. and Zhou, B.
\newblock Closed-form factorization of latent semantics in gans.
\newblock In \emph{CVPR}, 2021.

\bibitem[Shen et~al.(2020)Shen, Gu, Tang, and Zhou]{shen2020interpreting}
Shen, Y., Gu, J., Tang, X., and Zhou, B.
\newblock Interpreting the latent space of gans for semantic face editing.
\newblock In \emph{Proceedings of the IEEE/CVF Conference on Computer Vision
  and Pattern Recognition}, pp.\  9243--9252, 2020.

\bibitem[Voynov \& Babenko(2020)Voynov and Babenko]{voynov2020unsupervised}
Voynov, A. and Babenko, A.
\newblock Unsupervised discovery of interpretable directions in the gan latent
  space.
\newblock In \emph{International Conference on Machine Learning}, pp.\
  9786--9796. PMLR, 2020.

\bibitem[Wu et~al.(2020)Wu, Lischinski, and Shechtman]{wu2020stylespace}
Wu, Z., Lischinski, D., and Shechtman, E.
\newblock Stylespace analysis: Disentangled controls for stylegan image
  generation.
\newblock \emph{arXiv preprint arXiv:2011.12799}, 2020.

\bibitem[Ye \& Lim(2016)Ye and Lim]{grassgeo}
Ye, K. and Lim, L.-H.
\newblock Schubert varieties and distances between subspaces of different
  dimensions.
\newblock \emph{SIAM Journal on Matrix Analysis and Applications}, 37\penalty0
  (3):\penalty0 1176--1197, 2016.

\bibitem[Yu et~al.(2015)Yu, Zhang, Song, Seff, and Xiao]{yu15lsun}
Yu, F., Zhang, Y., Song, S., Seff, A., and Xiao, J.
\newblock Lsun: Construction of a large-scale image dataset using deep learning
  with humans in the loop.
\newblock \emph{arXiv preprint arXiv:1506.03365}, 2015.

\bibitem[Zhu et~al.(2021)Zhu, Feng, Shen, Zhao, Zha, Zhou, and
  Chen]{LowRankSubspace}
Zhu, J., Feng, R., Shen, Y., Zhao, D., Zha, Z., Zhou, J., and Chen, Q.
\newblock Low-rank subspaces in gans.
\newblock \emph{arXiv preprint arXiv:2106.04488}, 2021.

\bibitem[Zhu et~al.(2020)Zhu, Abdal, Qin, Femiani, and Wonka]{Pnspace}
Zhu, P., Abdal, R., Qin, Y., Femiani, J., and Wonka, P.
\newblock Improved stylegan embedding: Where are the good latents?
\newblock \emph{arXiv preprint arXiv:2012.09036}, 2020.

\end{thebibliography}
\bibliographystyle{icml2023}

\newpage
\appendix
\onecolumn
\section{Definition of disentangled latent space}
\paragraph{Disentangled perturbation} 
In the GAN disentanglement literature, several studies investigated the disentanglement property of the latent space by finding disentangled perturbations that make a disentangled transformation of an image in one generative factor, such as GANSpace \citep{harkonen2020ganspace},  SeFa \citep{shen2020closed}, and Local Basis \citep{LocalBasis}. To be more specific,
for a latent variable $z \in \mathcal{Z} \subset \mathbb{R}^{d}$, let $f = (f_{1}, f_{2}, \cdots, f_{d})$ be a generative factor of $G(z)$ where $G$ denotes the generator.
$T_{j}(x)$ denotes a transformation of an image $x$ in the $j$-th generative factor. The disentangled perturbation $v_{j}(z)$ for the base latent variable $z$ on the $j$-th generative factor is defined as follows (The perturbation intensity $\| v_{j}(z) \|$ and the corresponding change in $j$-th generative factor $\bigtriangleup f_{j}$ is omitted for brevity.):
\begin{equation}
    G(z + v_{j}(z)) = T_{j}(G(z)).
\end{equation}
In this paper, the global basis refers to the sample-independent disentangled perturbations on a latent space:
\begin{equation}
    v_{j}(z) = v_{j} \quad \text{ for all } z \in Z.
\end{equation}
For example, consider a pre-trained GAN model that generates face images. Then, the disentangled perturbation in this model is the latent perturbation direction that make the generated face change only in the wrinkles or hair color as presented in \citet{harkonen2020ganspace}. This disentangled perturbation is the global basis if all generated images show the same semantic variation when latent perturbed along it. 

\paragraph{Disentangled space}
The (globally) disentangled latent space is defined in terms of disentangled perturbations. The latent space is globally disentangled if there exists the global basis for the generative factors of data. In other words, for each generative factor $f_j$ for $1 \leq j \leq d$, there exists a corresponding latent perturbation direction $v_{j}$ such that all latent variables show the semantic variation in $f_j$ when perturbed along $v_{j}$. Then, we can interpret the vector component of this global basis $v_{j}$ as having a correspondence with the $j$-th generative factor $f_j$.
\begin{equation}
    f_{j} \longleftrightarrow c_{j}
    \quad\text{ when } z = \sum_{1\leq j \leq d} c_{j} \cdot v_{j} \text{ and } f_{j} \text{ denotes the $j$-th generative factor of $G(z)$.}
\end{equation}
In this paper, we described the above correspondence as the representation of globally disentangled latent space in the attribute-coordinate in Sec \ref{sec:distortion}.
This is consistent with the definition of disentanglement, introduced in \citet{bengio2013representation}.
For example, consider the dSprites dataset. The dSprites is a synthetic dataset consisting of two-dimensional shape images, which is widely used for disentanglement evaluation. The generative factors of dSprites are shape, scale, orientation, position on the x-axis, and position on the y-axis. Then, the (globally) disentangled latent space for the dSprites dataset is a five-dimensional vector space $Z=\mathbb{R}^{5}$ where $c_{1}$ represents the shape, $c_{2}$ represents the scale, and so on.

\section{Noise Estimation of Pseudorank Algorithm}
For completeness, we include the convergence theorem for the largest eigenvalue of the empirical covariance matrix for Gaussian noise in \citet{TracyWidom} and the noise estimation algorithm provided in \citet{RankEstimate}. 

\begin{theorem}[\citep{TracyWidom}]
    The distribution of the largest eigenvalue $\lambda_{1}$ of the empirical covariance matrix for $n$-samples of $\mathcal{N}(0, I_{p})$ converges to a Tracy-Widom distribution:
    \begin{align}
        &P\left(\lambda_{1} < \sigma^{2} \left( \mu_{n, p} + s\cdot \sigma_{n, p} \right) \right) \rightarrow F_{\beta}(s) 
        \quad \textrm{as $n, p \rightarrow \infty$ with $c=p/n$ fixed.}
        \\
        &\textrm{where} \quad\mu_{n, p}=\frac{1}{n}\left(\sqrt{n- \frac{1}{2}}+\sqrt{p-\frac{1}{2}}\right)^{2},  \\
        &\qquad \quad \,\,\,\sigma_{n, p}=\frac{1}{n}\left(\sqrt{n-\frac{1}{2}}+\sqrt{p-\frac{1}{2}}\right)
        \left(\frac{1}{\sqrt{n-1/2}}+\frac{1}{\sqrt{p-1/2})}\right)^{1/3},
    \end{align}
    where $F_{\beta}$ denotes the Tracy-Widom distribution of order $\beta = 1$ for real-valued observations. 
\end{theorem}

\paragraph{Algorithm}
Solve the following non-linear system of $K+1$ equations involving the $K+1$ unknowns $\hat{\rho}^{1}, \cdots, \hat{\rho}^{K}$ and $\sigma^{2}_{est}$:
\begin{align}
    &\sigma_{\mathrm{KN}}^{2}-\frac{1}{p-K}\left[\sum_{j=K+1}^{p} \lambda_{j}+\sum_{j=1}^{K}\left(\lambda_{j}-\hat{\rho}_{j}\right)\right]=0, \\
    &\hat{\rho}_{j}^{2}-\hat{\rho}_{j}\left(\lambda_{j}+\sigma_{est}^{2}-\sigma_{est}^{2} \frac{p-K}{n}\right)+\lambda_{j} \sigma_{est}^{2}=0.
\end{align}
This system of equations can be solved iteratively. Check \citet{RankEstimate} for detail.

\section{Relation Between Rank Estimation Algorithm and Optimization}
\begin{theorem}
    The following optimization problem, called Nuclear-Norm Penalization (NNP), 
\begin{align}
    &\textrm{minimize}_{L, S} \quad \| L \|_{*} + \gamma \cdot \|E\|_{F} \qquad \textrm{s.t.} \,\, L+E = (\nabla_{\zvar}f)(\zvar),
\end{align}
has a solution
\begin{equation}
    L^{*} = U \left(\Sigma - \frac{1}{2\gamma}\cdot I \right)_{+} V^{\intercal}, 
\end{equation}
where $(\nabla_{\zvar}f)(\zvar) = U \Sigma V^{\intercal}\, (\textrm{SVD})$ and  $(M_{+})_{i,j} = \max\left(M_{i,j}\,, 0\right)$.
\end{theorem}
\begin{proof}
Denote $Y:=\nabla_{\zvar}f$ and $h(L):= \|L\|_{*}+\gamma \|Y-L\|_F$. We want to show that $L_{*}$ minimizes $h(L)$. Then, the necessary and sufficient condition for this is:
\begin{align}
   0\in  \partial h (\widehat{L_{*}}) = \{ 2\gamma (L_{*}-Y)+ z: z \in\partial\|\widehat{L_{*}}\|_{*} \} \quad
   \Longleftrightarrow \quad 2\gamma (Y - L_{*})\in \partial \|\widehat{L_{*}}\|_{*}.
\end{align}
Note that $Y = U\Sigma V^{\intercal}$ and $\hat{L} = U(\Sigma -\frac{1}{2\gamma} I)_{+}V^{\intercal}$. We can write 
\begin{equation}
    Y = U_1 \Sigma_1 V_{1}^{\intercal} + U_2 \Sigma_2 V_{2}^{\intercal},
\end{equation}
with $\text{diag}(\Sigma_1)>\frac{1}{2\gamma}$ and $\text{diag}(\Sigma_2)\leq \frac{1}{2\gamma}$. Then, 
\begin{align}
    \widehat{L_{*}} &=U_1\left(\Sigma - \frac{1}{2\gamma}I\right)V_{1}^{\intercal},  \\
    Y - \widehat{L_{*}} &= U_2 \Sigma_2 V_2^{\intercal} + \frac{1}{2\gamma}U_1 V_1^{\intercal} 
    = \frac{1}{2\gamma} \left( U_1 V_1^{\intercal} + 2\gamma U_2 \Sigma_2 V_2^{\intercal}\right).
\end{align}
By doing tedious calculation, we can verify that $U_1 V_1^{\intercal} + 2\gamma U_2 \Sigma_2 V_2^{\intercal}$ meet the condition of Lemma \ref{nuclear_grad}, so that $U_1 V_1^{\intercal} + 2\gamma U_2 \Sigma_2 V_2^{\intercal}\in \partial \|\widehat{L_{*}}\|$. Therefore, $0\in \partial h(\widehat{L_{*}})$ and it completes the proof.
\end{proof}

\begin{lemma}\label{nuclear_grad}
Let $X\in \sR^{m\times n}$ and $f(x)=\|X\|_{*}$. Then, 
\begin{equation}
    \partial f(X)= \partial \|X\|_{*}=\{ Z\in \mathbb{R}^{m\times n}: \|Z\|_2\leq 1 \,\,\textrm{ and }\,\, \langle Z,X \rangle=\|X\|_{*}\}.
\end{equation}
\end{lemma}
\begin{proof}
If $Z\in \partial f(X)$, then 
\begin{align}
    &f(Y) \geq f(X)+\langle Z, Y-X \rangle, \qquad \forall Y\in \mathbb{R}^{m\times n},
    \\ &\Leftrightarrow \,\,  \langle Z, X \rangle - \|X\|_{*} \geq
     \langle Z, Y \rangle - \|Y\|_{*} , \quad \forall Y\in \mathbb{R}^{m\times n},
    \\ &\Leftrightarrow \,\,   \langle Z, X \rangle - \|X\|_{*} \geq \sup_{Y\in \mathbb{R}^{m\times n}}\left(    \langle Z, Y \rangle - \|Y\|_{*}\right)
    = \begin{cases} 0, \quad \text{ if } \|Z\|_2\leq 1, \\ \infty, \,\, \text{ otherwise.}\end{cases}
\end{align}
And $0 \leq  \langle Z, X, \rangle - \|X\|_{*} = \langle Z, X, \rangle - \sup_{\|M\|_{2}\leq 1}\langle M, X, \rangle \leq 0 $, thus $\langle Z, X, \rangle = \|X\|_{*}$. 
\end{proof} 

\newpage
\section{Grassmannian Metric for Distortion - Geodesic vs. Projection} \label{sec:why_geodesic_metric}
Our proposed Distortion metric $\dset = I_{rand}/I_{local}$ is defined as the relative inconsistency of intrinsic tangent spaces on a latent manifold (Sec \ref{sec:distortion}). The inconsistency ($I_{rand}$, $I_{local}$) is measured by the Grassmannain \citep{boothby} distance between tangent spaces, particularly by Geodesic Metric \citep{grassgeo}. In this section, we present why we choose the Geodesic Metric instead of the Projection Metric \citep{grassproj} among Grassmannian distances. Informally, the Geodesic Metric provides a better discriminability compared to the Projection Metric. For completeness, we begin with the definitions of the Grassmannian manifold and two distances defined on it.

\paragraph{Definitions}
Let $V$ be the $n$-dimensional vector space. The \textit{Grassmannian manifold} $\textrm{Gr}(k, V)$ \citep{boothby} is defined as the set of all $k$-dimensional linear subspaces of $V$. Then, for two $k$-dimensional subspaces $W, W^{\prime} \in \textrm{Gr}(k, V)$, two Grassmannian metrics are defined as follows:
\begin{equation}
    d_{\mathrm{proj}}\left(W, W^{\prime}\right)=\left\|P_{W}-P_{W^{\prime}}\right\|, \qquad d_{\mathrm{geo}}(W,W^{\prime})=\left(\sum^{k}_{i=1}\theta^2_i\right)^{1/2}.
\end{equation}
For the \textit{Projection Metric} $d_{\mathrm{proj}}\left(W, W^{\prime}\right)$, $P_{W}$ and $P_{W^{\prime}}$ denote the projection into each subspaces and $\| \cdot \|$ represents the operator norm. For the \textit{Geodesic Metric} $d_{\mathrm{geo}}(W,W^{\prime})$, $\theta_{i}$ denotes the $i$-th principal angle between $W$ and $W^{\prime}$. To be more specific, $\theta_i =\cos^{-1}(\sigma_{i}(M_{W}^{\top} \, M_{W^{\prime}} ))$ where $M_W, M_{W^{\prime}}\in \sR^{n\times k}$ are the column-wise concatenation of orthonormal basis for $W, W^{\prime}$ and $\sigma_{i}$ represents the $i$-th singular value.

\paragraph{Experiments}
To test the discriminability of these two metrics, we designed a simple experiment. Let $W, W^{\prime}$ be the two 50-dimensional subspaces of  $\sR^{512}$ because the dimension of intermediate layers in the mapping network is 512. We measure the Grassmannian distance between two subspaces as we vary $\dim \left(W \cap W^{\prime} \right) = k_{0}$,
\begin{equation}
    W = \langle e_{1}, e_{2}, \cdots, e_{k} \rangle, \qquad
    W = \langle \{e_{1}, e_{2}, \cdots, e_{k_{0}}\} \cup \{ e_{k+1}, \cdots,  e_{2k-k_{0}}\} \rangle
\end{equation}
where $\{e_{i}\}_{1\leq i \leq n}$ denotes the standard basis of $\sR^{n}$. Fig \ref{fig:geo_vs_proj} reports the results. The Geodesic Metric reflects the degree of intersection between two subspaces.  As we increase the dimension of intersection, the Geodesic Metric decreases. However, the Projection Metric cannot discriminate the intersected dimension until it reaches the entire space.

\begin{figure}[h]
    \centering
    \begin{subfigure}[b]{0.4\columnwidth}
    \centering
    \includegraphics[width=\columnwidth]{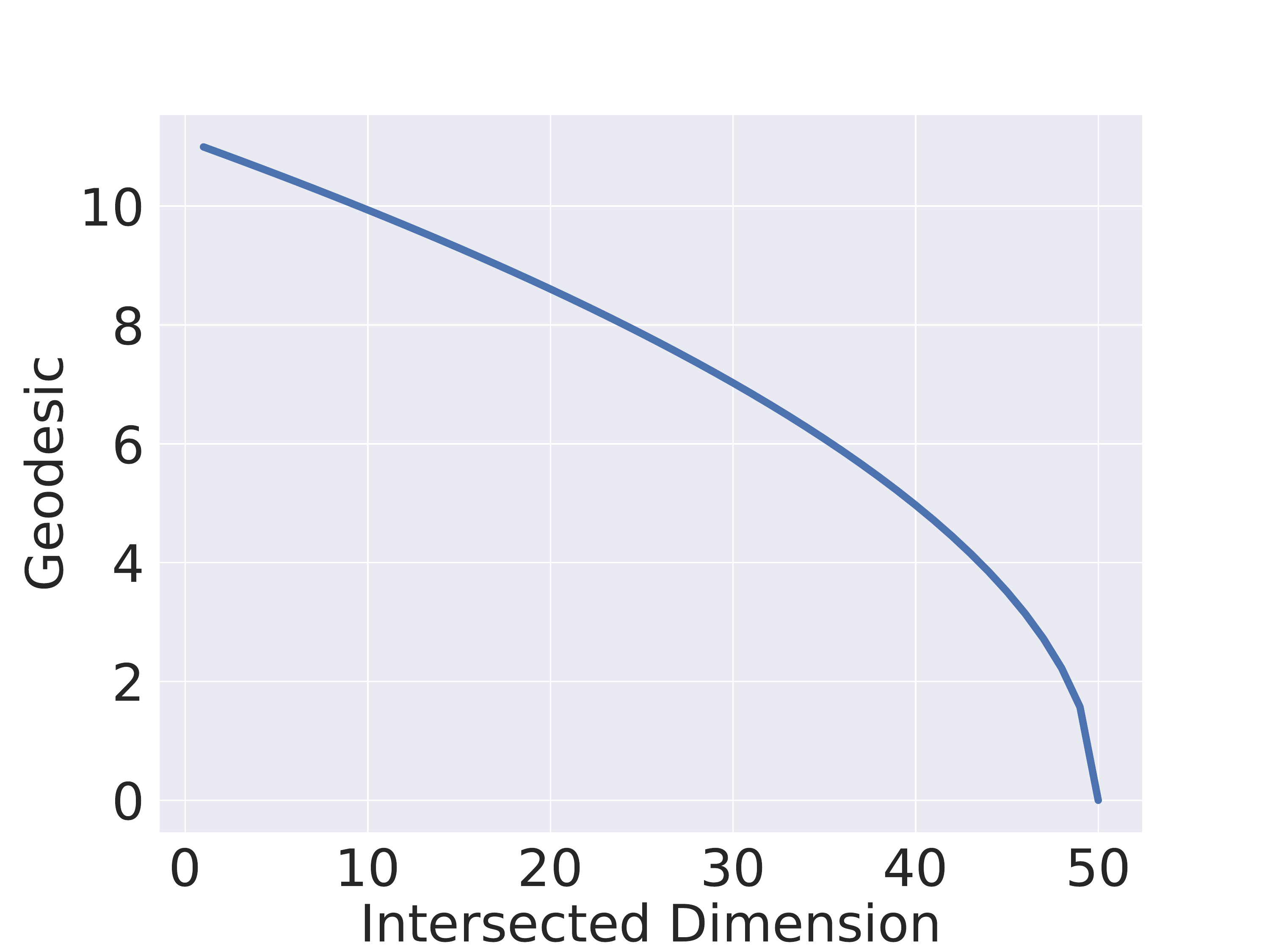}
    \caption{Geodesic Metric}
    \end{subfigure}
    \quad
    \begin{subfigure}[b]{0.4\columnwidth}
    \centering
    \includegraphics[width=\columnwidth]{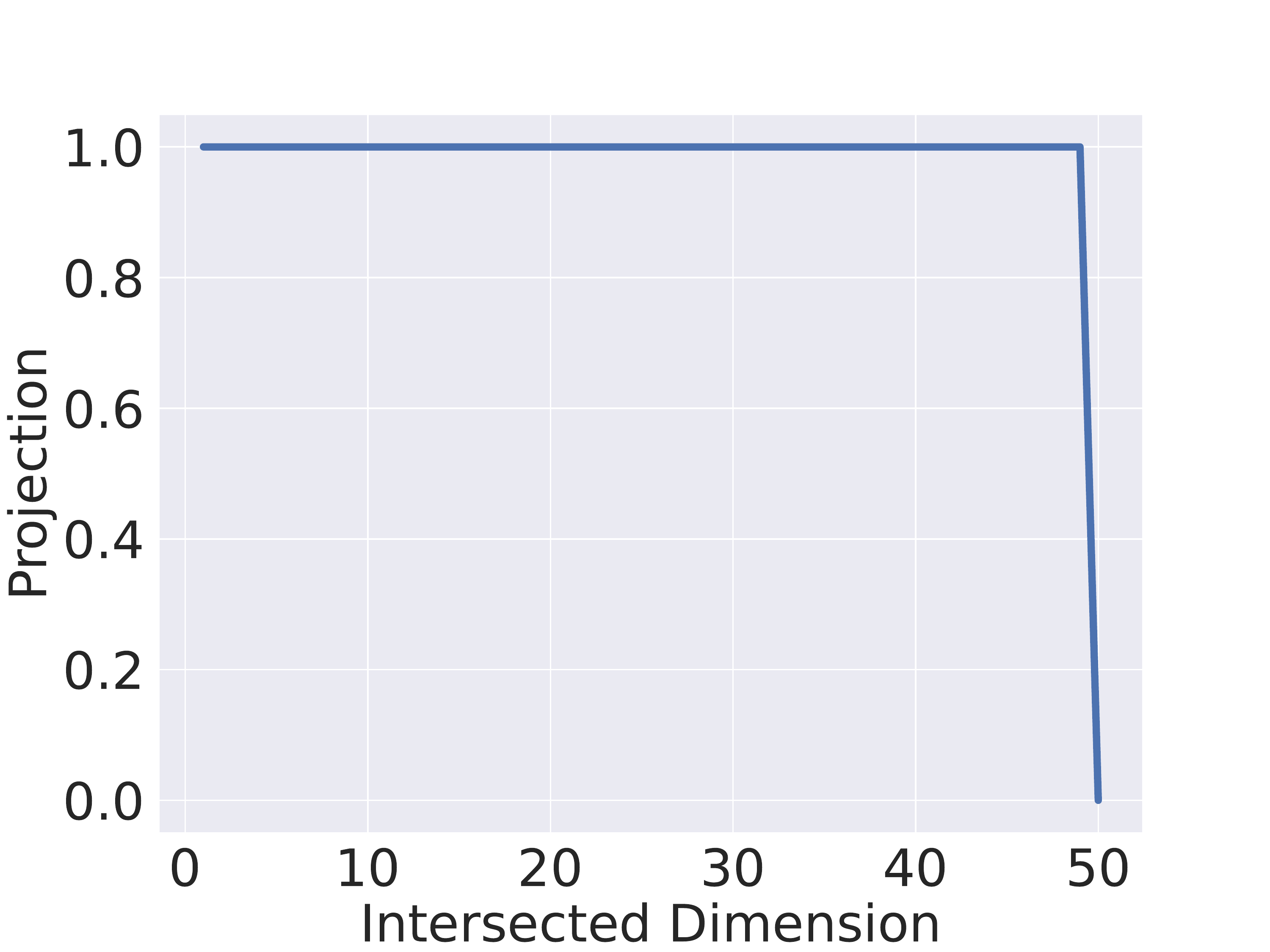}
    \caption{Projection Metric}
    \end{subfigure}
    \caption{
    \textbf{Grassmannian metric} between two 50-dimensional subspaces $W, W^{\prime} \in \textrm{Gr}(50, \sR^{512})$ for each intersected dimension $k_{0} = \dim \left(W \cap W^{\prime} \right)$. While the Geodesic Metric monotonically decreases as more dimensions intersect, the Projection Metric cannot discriminate $0 \leq k_{0}  \leq 49$.
    }
    \label{fig:geo_vs_proj}
\end{figure}

\newpage
\section{Robustness to preprocessing} 
In this section, we assessed the robustness of Distortion $\mathcal{D}_{\mset}$ to preprocessing hyperparameter $\theta_{pre}$. Figure \ref{fig:Dist_robust} presents the distribution of 1k samples of distortion before taking an expectation, i.e., $\left( d^{k}_{\mathrm{geo}}\left( \,T_{\wvar_{1}} \mset^{k}_{\wvar_{1}},T_{\wvar_{2}} \mset^{k}_{\wvar_{2}} \right) / I_{local} \right)$, for each intermediate layer. In Fig \ref{fig:Dist_robust}, increasing $\theta_{pre}$ makes an overall translation of Distortion. However, the relative ordering between the layers remains the same. The low Distortion score of layer 8 provides an explanation for the superior disentanglement of $\wset$-space observed in many literatures \citep{karras2019style, harkonen2020ganspace}. Moreover, the results suggest that the min-distorted layer $7$ can serve as a similar-or-better alternative.


\begin{figure}[h]
    \centering
    \begin{subfigure}[b]{0.35\columnwidth}
    \centering
    \includegraphics[width=\columnwidth]{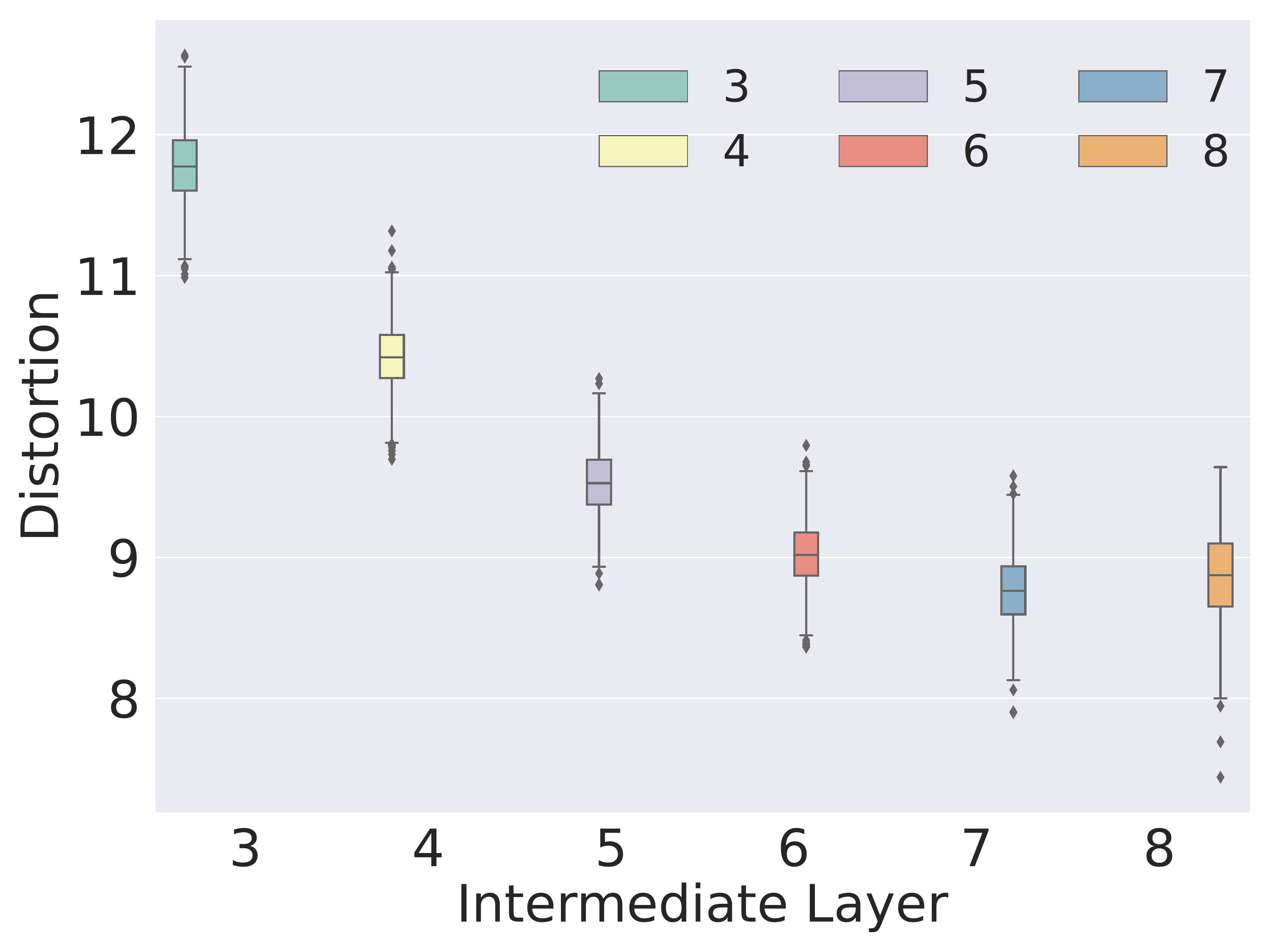}
    \caption{$\theta_{pre}=0.0005$}
    \end{subfigure} 
    \qquad
    \begin{subfigure}[b]{0.35\columnwidth}
    \centering
    \includegraphics[width=\columnwidth]{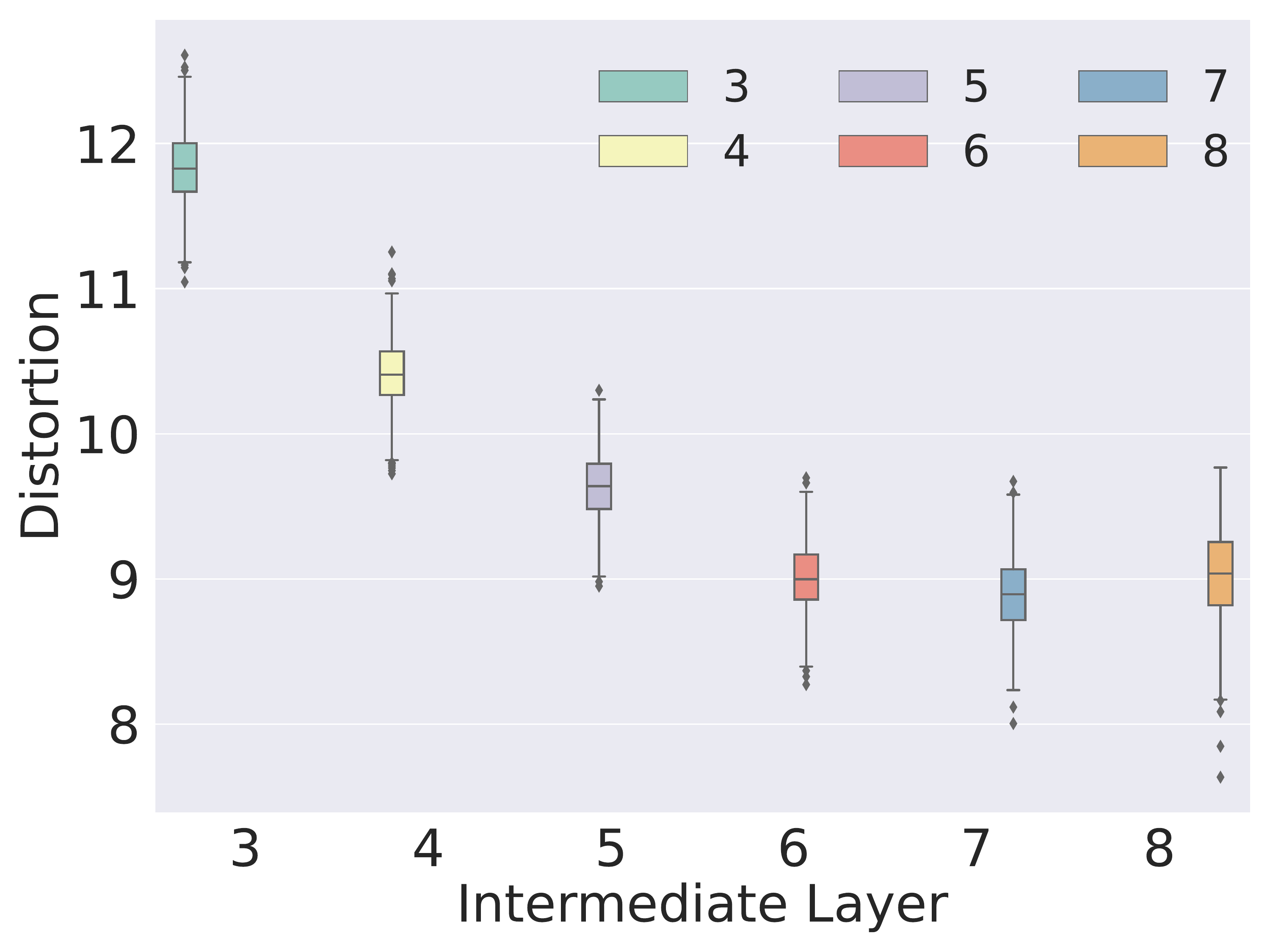}
    \caption{$\theta_{pre}=0.001$}
    \end{subfigure}
    \\
    \begin{subfigure}[b]{0.35\columnwidth}
    \centering
    \includegraphics[width=\columnwidth]{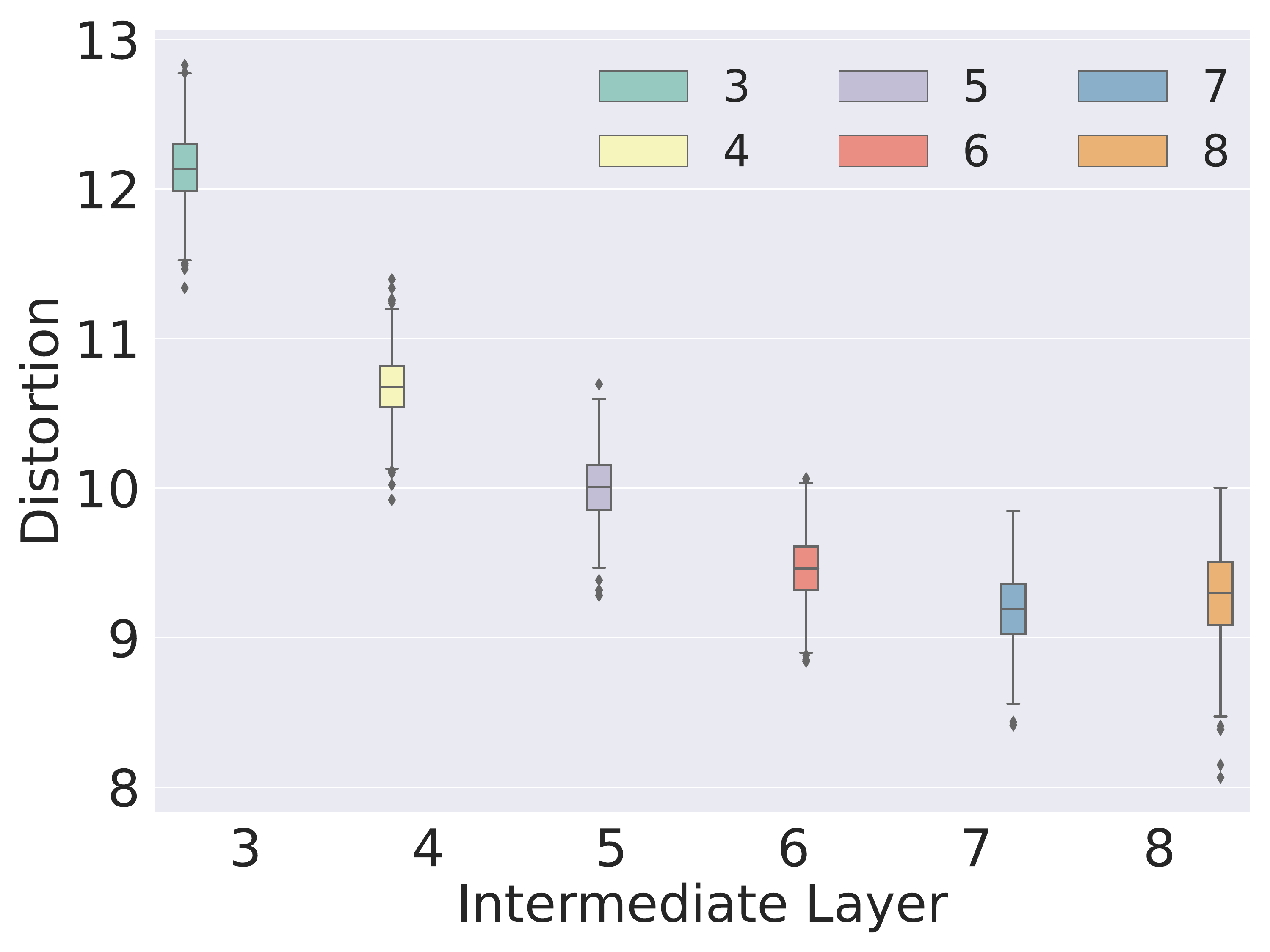}
    \caption{$\theta_{pre}=0.005$}
    \end{subfigure} 
    \qquad
    \begin{subfigure}[b]{0.35\columnwidth}
    \centering
    \includegraphics[width=\columnwidth]{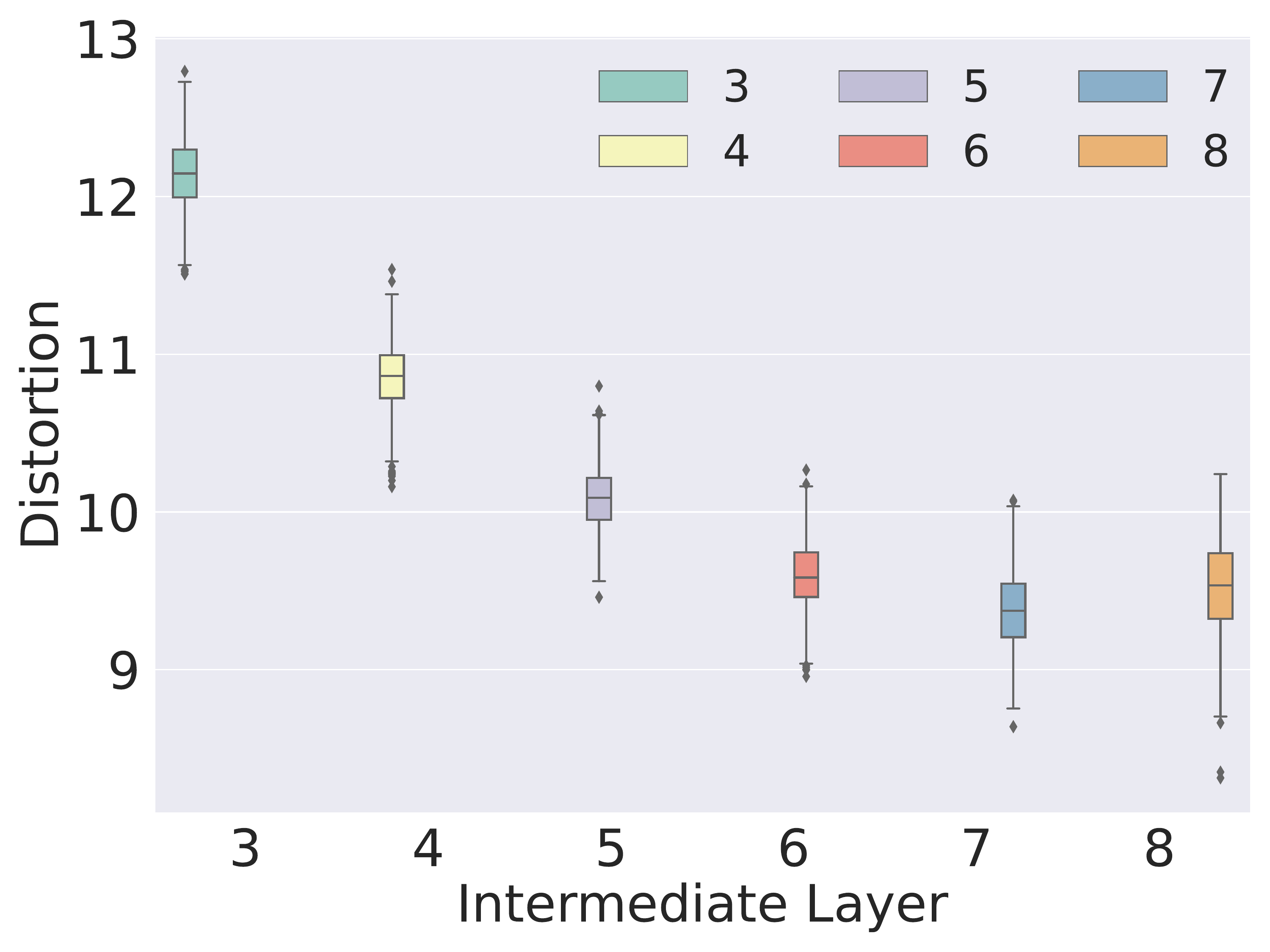}
    \caption{$\theta_{pre}=0.01$}
    \end{subfigure} 
    \caption{
    \textbf{Robustness of Distortion metric $\dset$ to $\theta_{pre}$} of StyleGAN2 on FFHQ. 
    }
    \label{fig:Dist_robust}
\end{figure}

\section{Architecture diagram of StyleGANs}
\begin{figure}[h]
    \centering
    \begin{subfigure}[b]{0.30\columnwidth}
    \centering
    \includegraphics[width=\columnwidth]{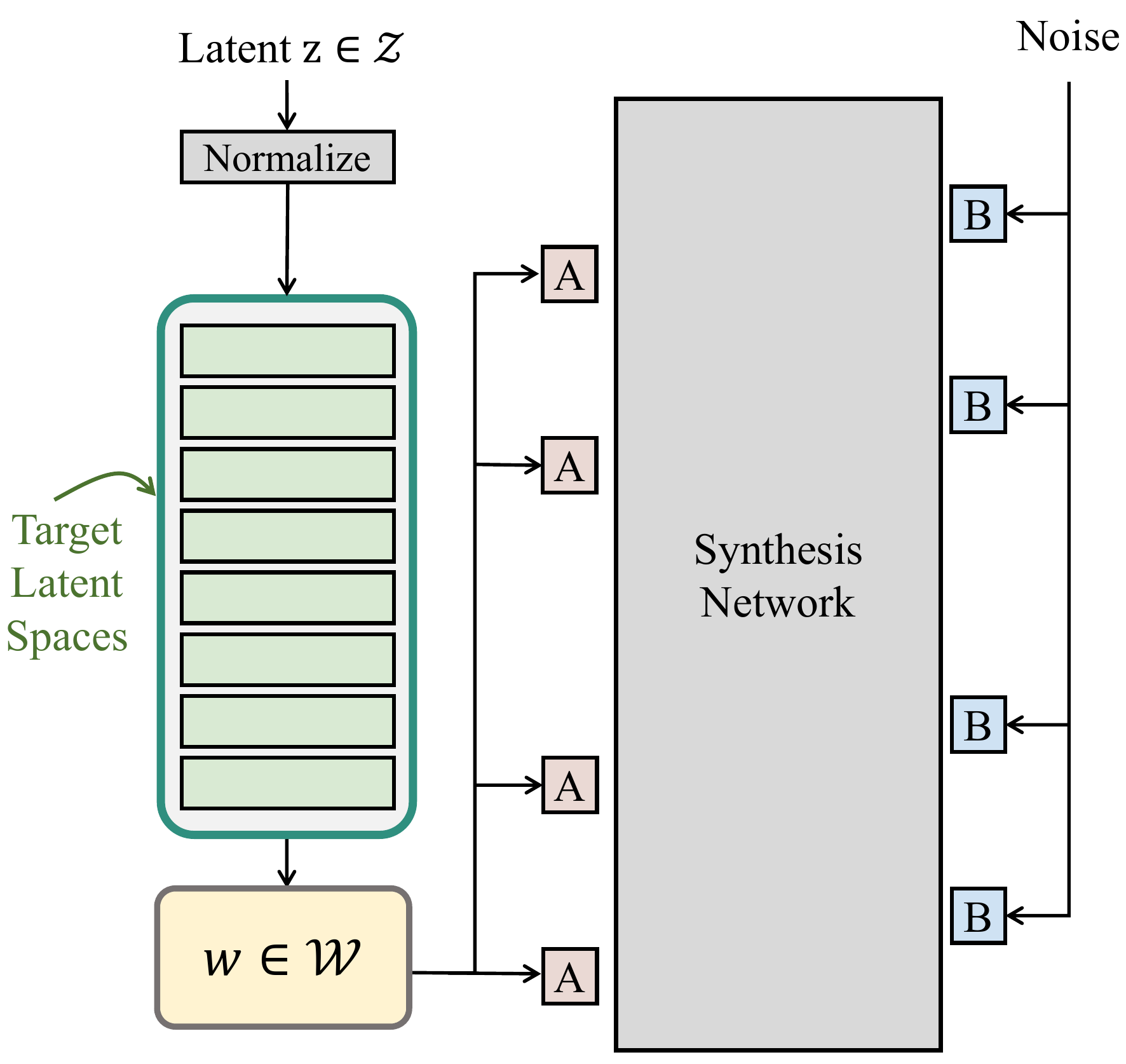}
    \end{subfigure}
    \caption{
    \textbf{Architecture of StyleGANs.} Our analysis in Sec \ref{sec:distortion} is performed in the intermediate layers of the mapping network. 
    }
    \label{fig:architecture_stylegan}
\end{figure}

\newpage
\section{Robustness of Correlations to preprocessing} \label{sec:corr_robust}
\begin{figure}[h]
    \centering
    \begin{subfigure}[b]{0.4\columnwidth}
    \centering
    \includegraphics[width=\columnwidth]{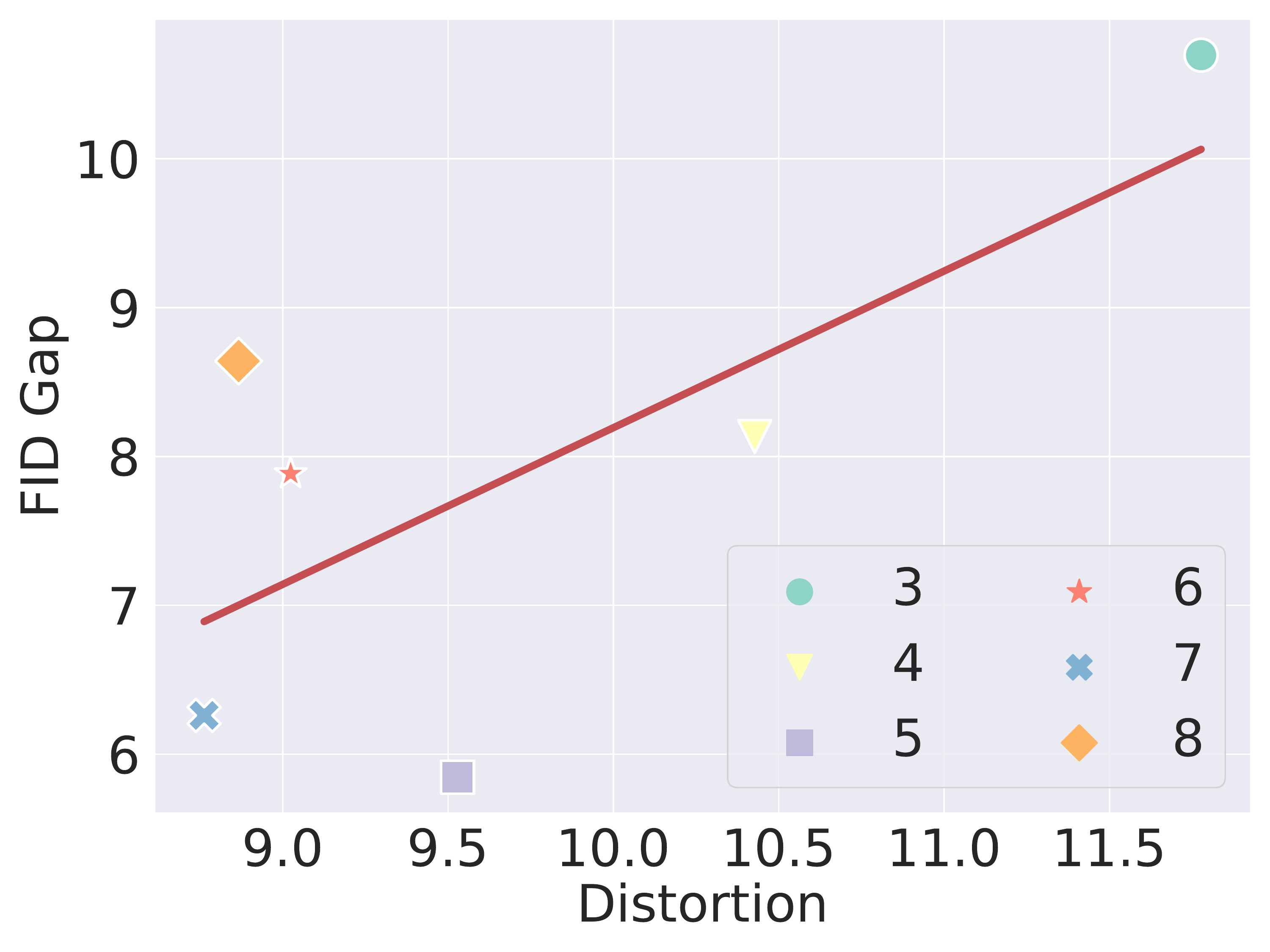}
    \caption{$\theta_{pre}=0.0005$, $\rho=0.706$}
    \end{subfigure}
    \quad
    \begin{subfigure}[b]{0.4\columnwidth}
    \centering
    \includegraphics[width=\columnwidth]{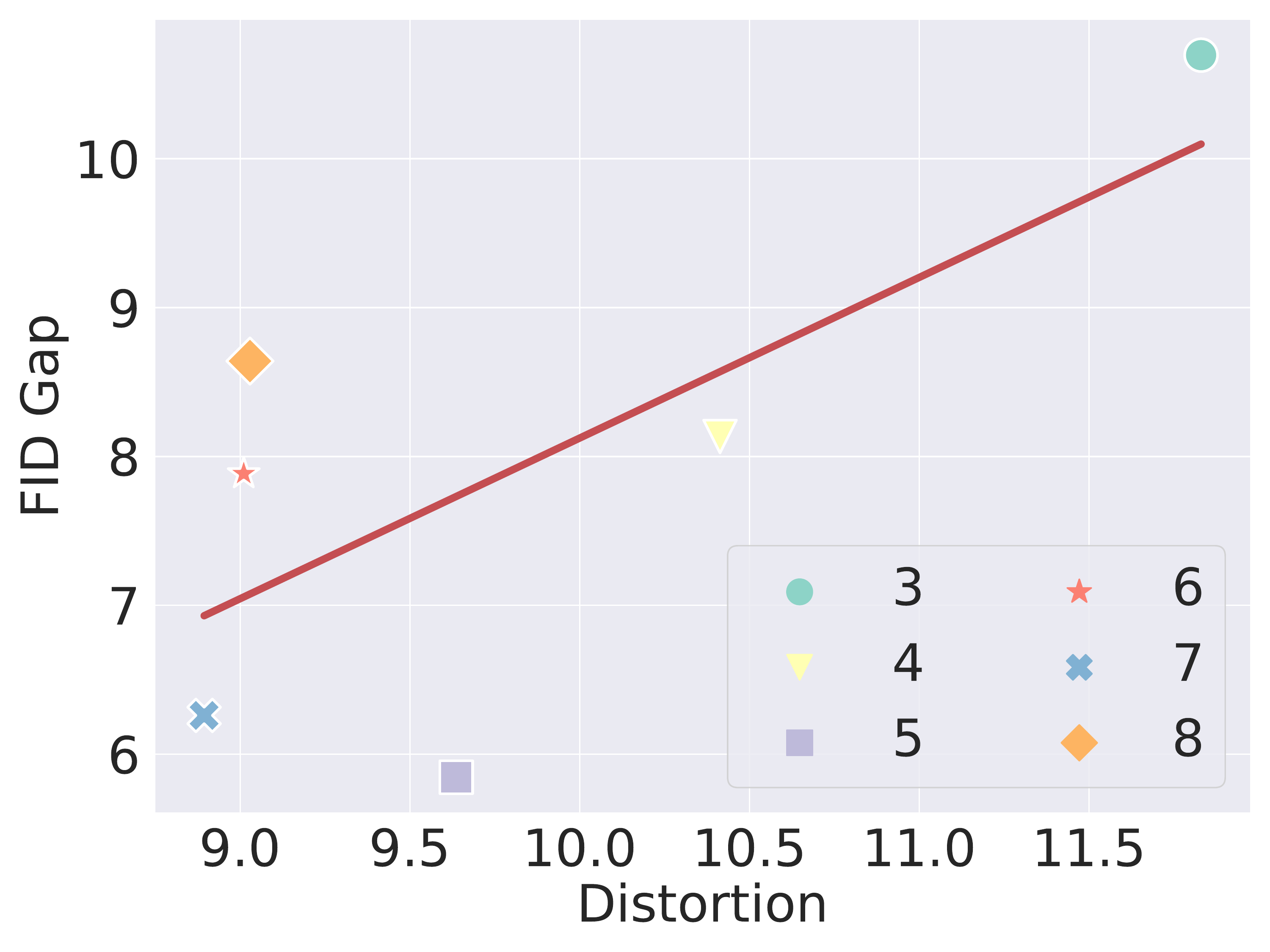}
    \caption{$\theta_{pre}=0.001$, $\rho=0.706$}
    \end{subfigure}
    \\
    \begin{subfigure}[b]{0.4\columnwidth}
    \centering
    \includegraphics[width=\columnwidth]{figure/FID/geodesic_fid_StyleGAN2-f_ptb_3_0.005_corr_0.7005341764257045.pdf}
    \caption{$\theta_{pre}=0.005$, $\rho=0.701$}
    \end{subfigure}
    \quad
    \begin{subfigure}[b]{0.4\columnwidth}
    \centering
    \includegraphics[width=\columnwidth]{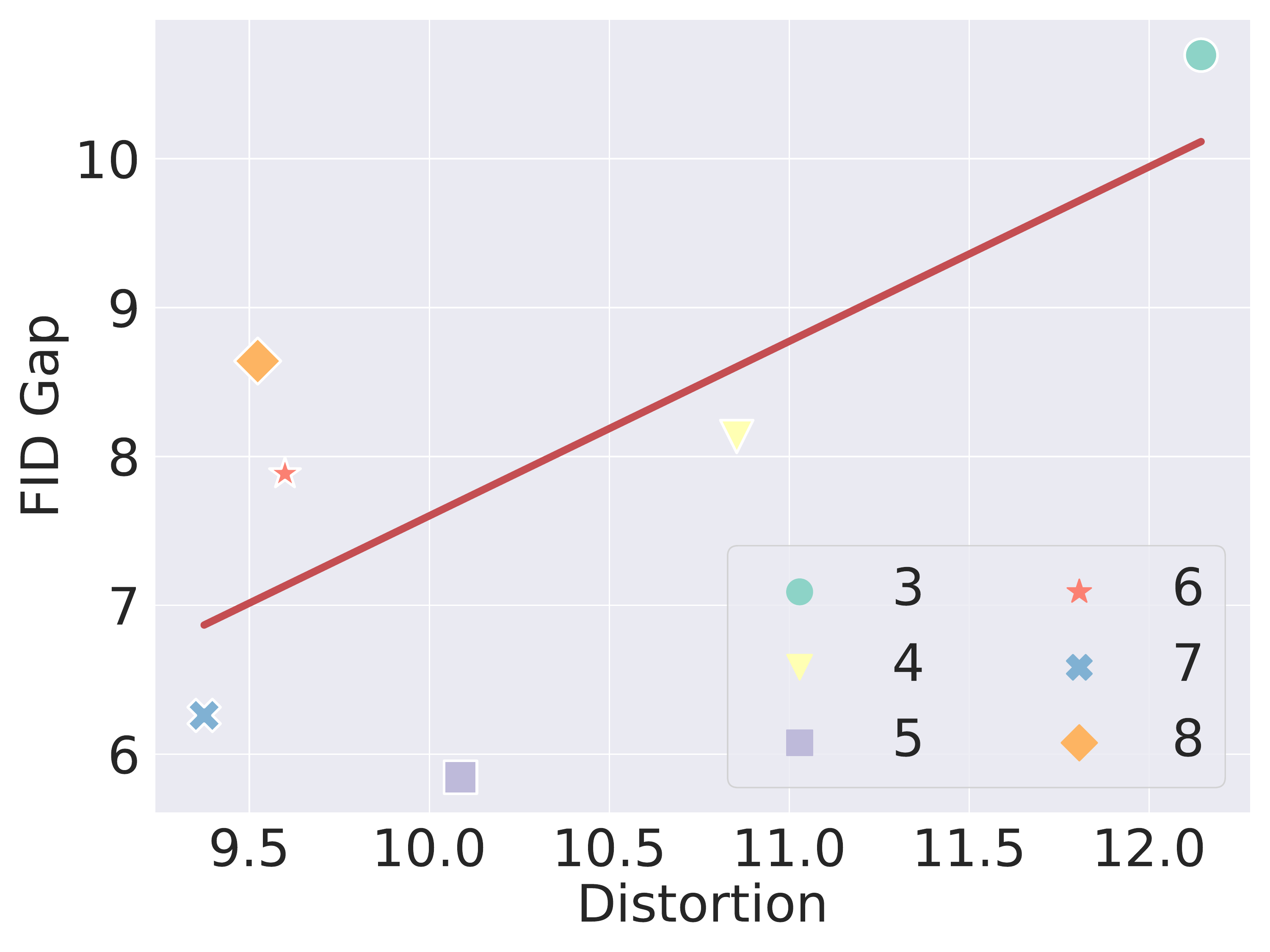}
    \caption{$\theta_{pre}=0.01$, $\rho=0.715$}
    \end{subfigure}
    \caption{ 
    \textbf{Correlation between Distortion metric and FID gap} of StyleGAN2 on FFHQ. FID gap represents the difference between FID \citep{heusel2017gans} score of Local Basis \citep{LocalBasis} and the global basis \citep{harkonen2020ganspace}. Each FID score is measure for 50k samples of latent-perturbed images along the first component of Local Basis and GANSpace. The perturbation intensity is fixed to 3. Each point represents a $i$-th intermediate layer in the mapping network, and the red-line illustrates the linear regression of these points. $\rho$ denotes the Pearson correlation coefficient of Distortion and FID. The positive correlation between Distortion and FID remains robust regardless of $\theta_{pre}$.
    }
    \label{fig:Dist_fid_app}
\end{figure}

\begin{figure}[t]
    \centering
    \begin{subfigure}[b]{0.4\columnwidth}
    \centering
    \includegraphics[width=\columnwidth]{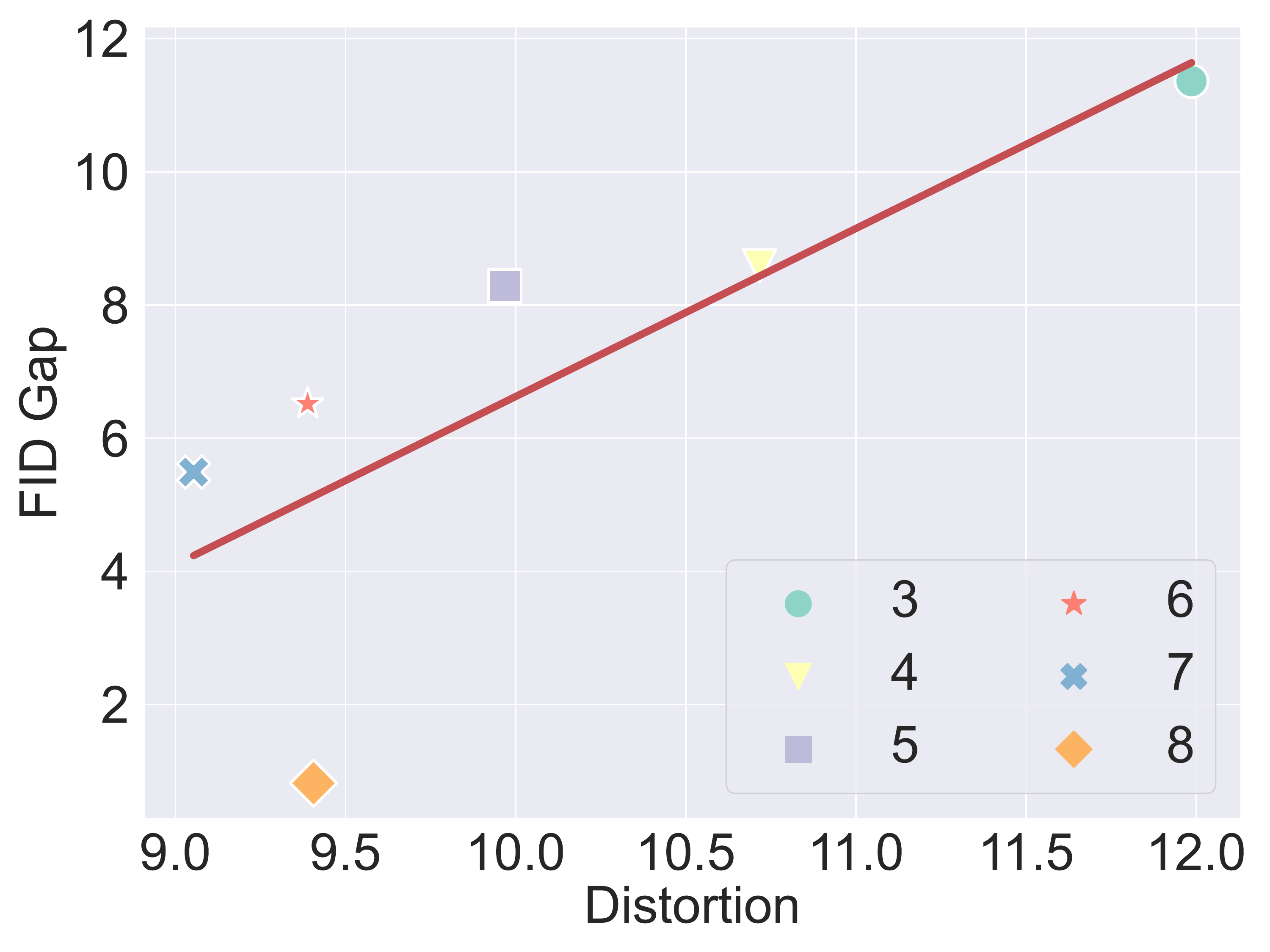}
    \caption{$\theta_{pre}=0.0005$, $\rho=0.777$}
    \end{subfigure}
    \quad
    \begin{subfigure}[b]{0.4\columnwidth}
    \centering
    \includegraphics[width=\columnwidth]{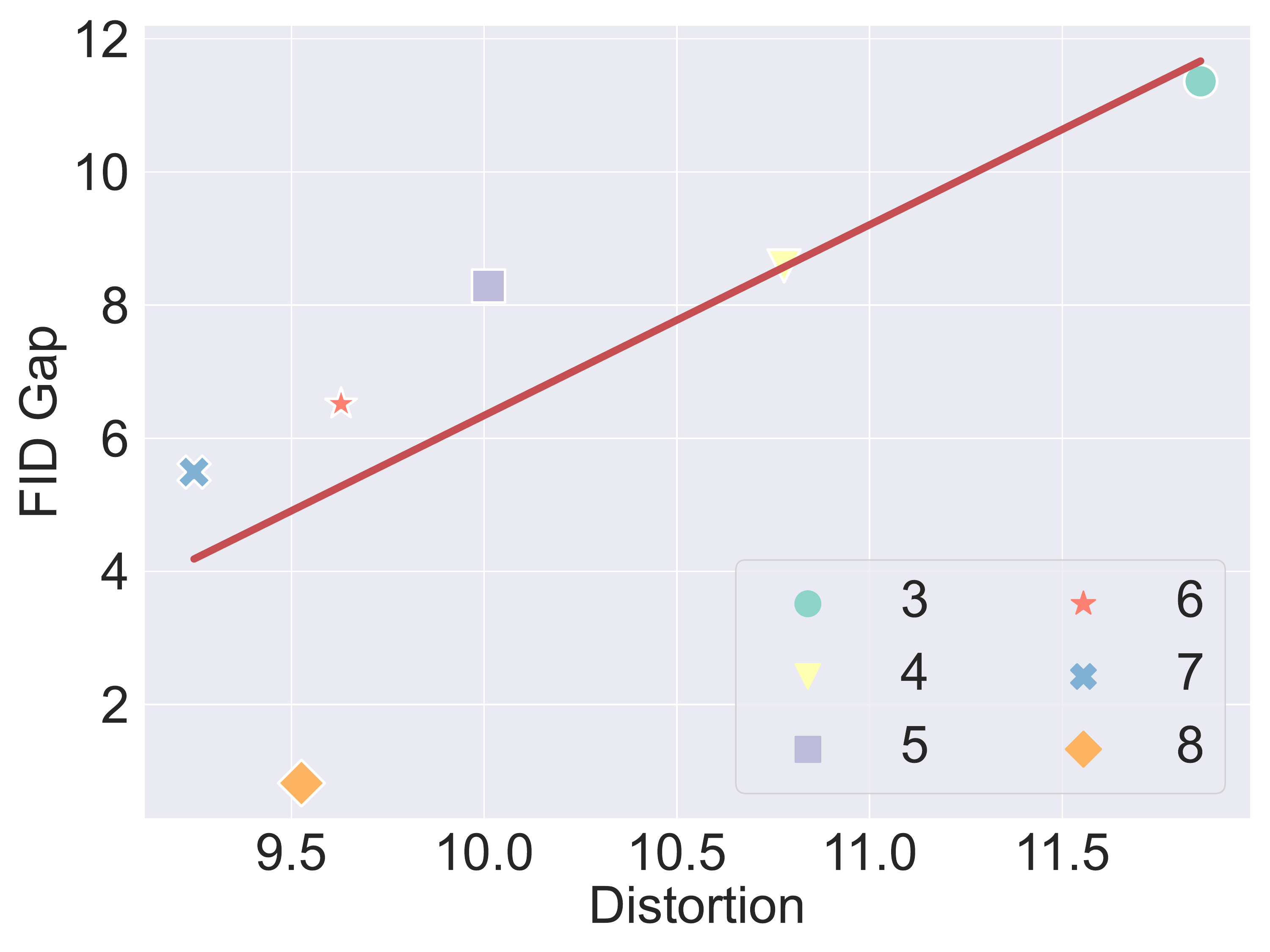}
    \caption{$\theta_{pre}=0.001$, $\rho=0.786$}
    \end{subfigure}
    \\
    \begin{subfigure}[b]{0.4\columnwidth}
    \centering
    \includegraphics[width=\columnwidth]{figure/FID/geodesic_fid_StyleGAN2-e_ptb_3_0.005_corr_0.8105834671788595_fix.pdf}
    \caption{$\theta_{pre}=0.005$, $\rho=0.811$}
    \end{subfigure}
    \quad
    \begin{subfigure}[b]{0.4\columnwidth}
    \centering
    \includegraphics[width=\columnwidth]{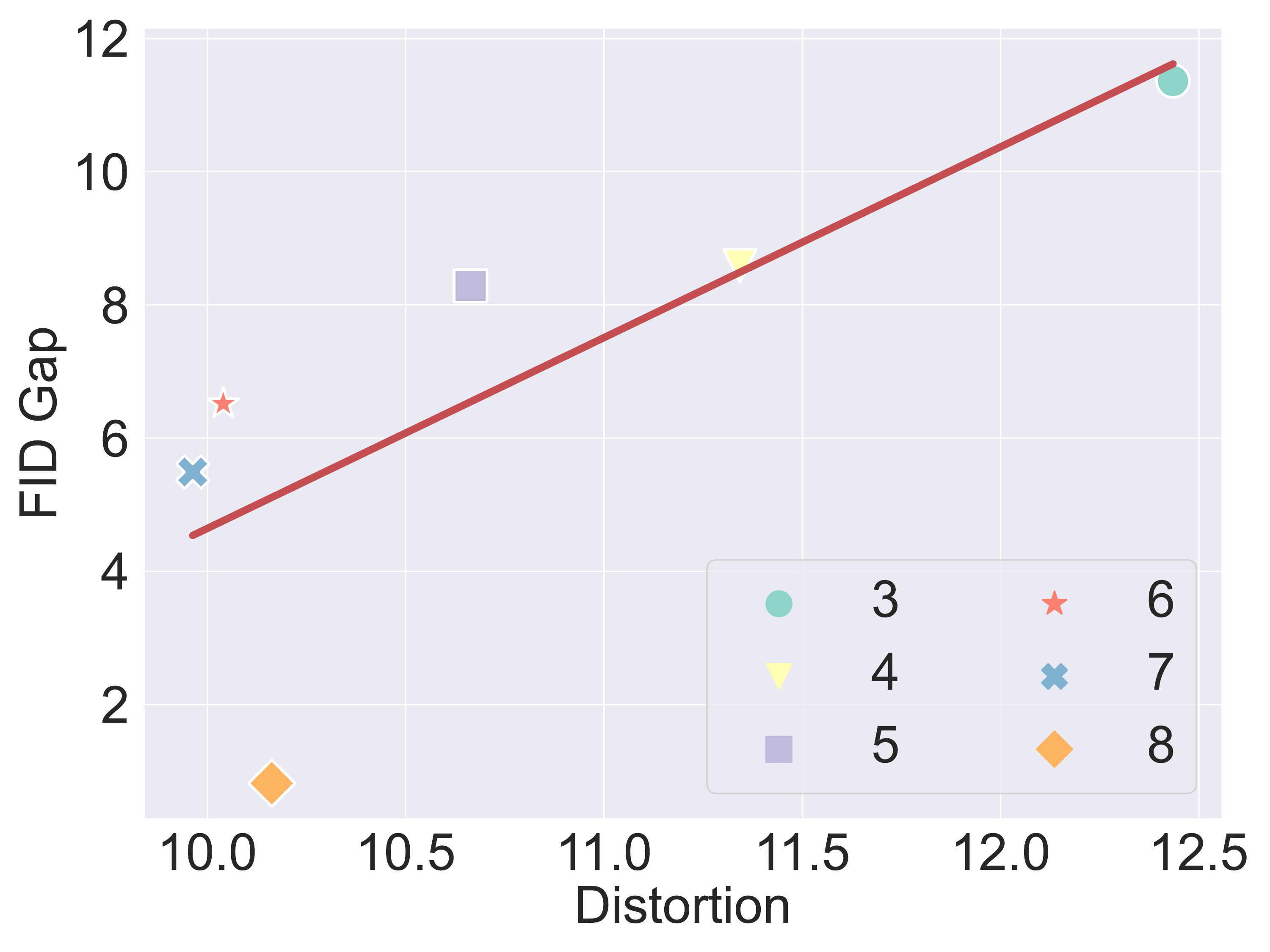}
    \caption{$\theta_{pre}=0.01$, $\rho=0.775$}
    \end{subfigure}
    \caption{ 
    \textbf{Correlation between Distortion metric and FID gap} of StyleGAN2 with config E on FFHQ.}
    \label{fig:Dist_fid_app_style-e}
\end{figure}

\begin{figure}[t]
    \centering
    \begin{subfigure}[b]{0.4\columnwidth}
    \centering
    \includegraphics[width=\columnwidth]{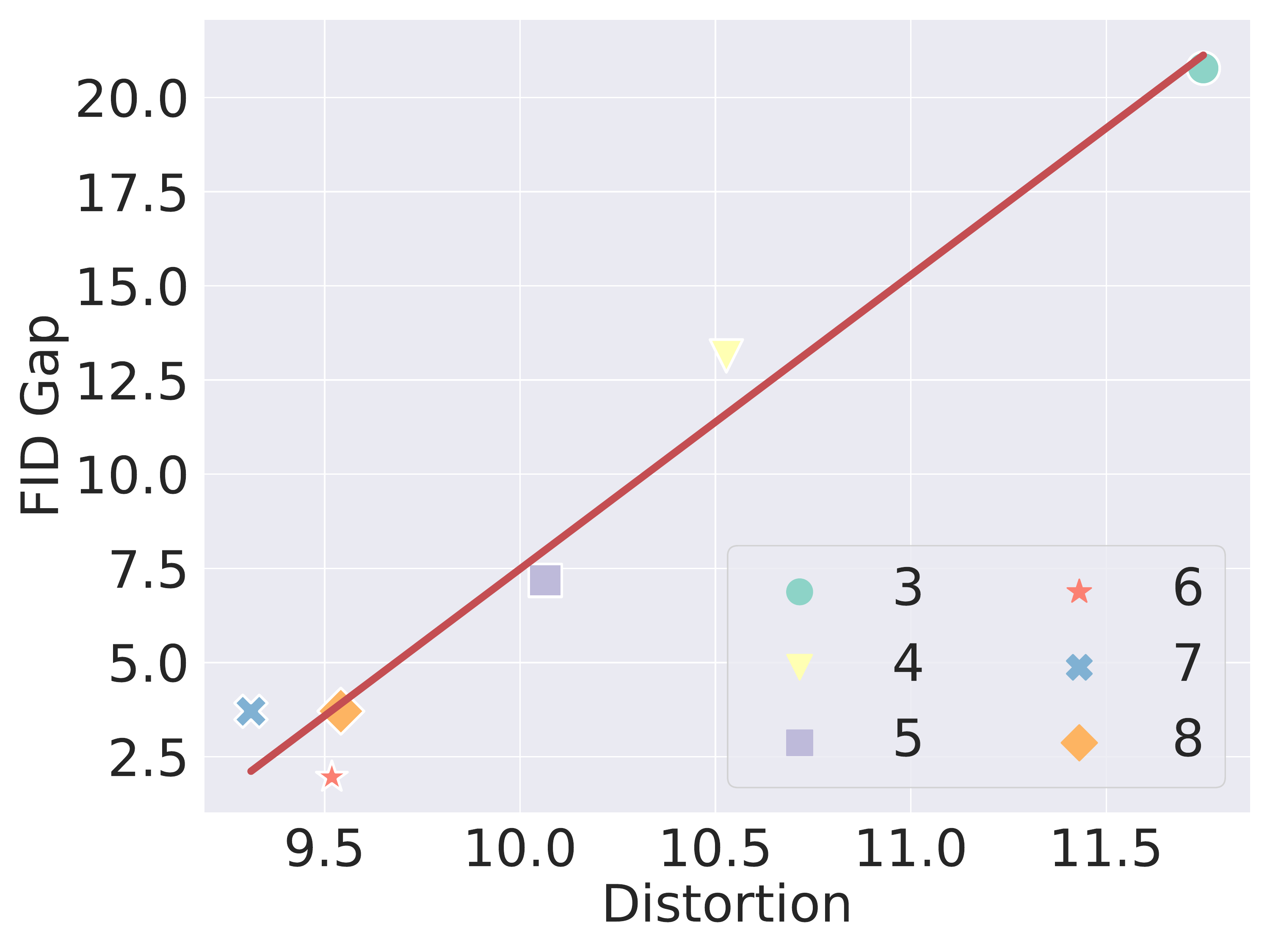}
    \caption{$\theta_{pre}=0.0005$, $\rho=0.98$}
    \end{subfigure}
    \quad
    \begin{subfigure}[b]{0.4\columnwidth}
    \centering
    \includegraphics[width=\columnwidth]{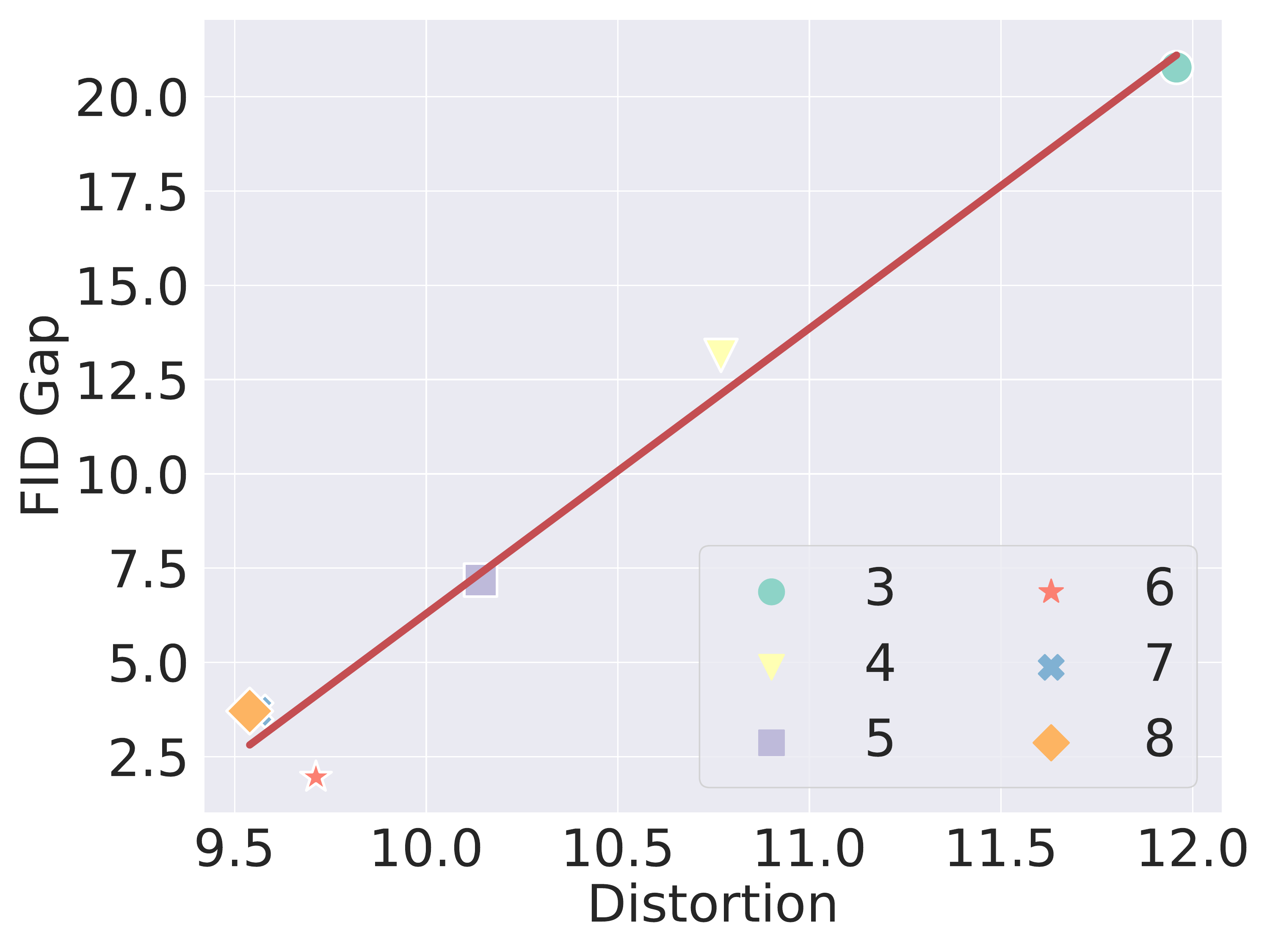}
    \caption{$\theta_{pre}=0.001$, $\rho=0.986$}
    \end{subfigure}
    \\
    \begin{subfigure}[b]{0.4\columnwidth}
    \centering
    \includegraphics[width=\columnwidth]{figure/FID/geodesic_fid_StyleGAN2-cat_ptb_5_0.005_corr_0.9784764908714002.pdf}
    \caption{$\theta_{pre}=0.005$, $\rho=0.978$}
    \end{subfigure}
    \quad
    \begin{subfigure}[b]{0.4\columnwidth}
    \centering
    \includegraphics[width=\columnwidth]{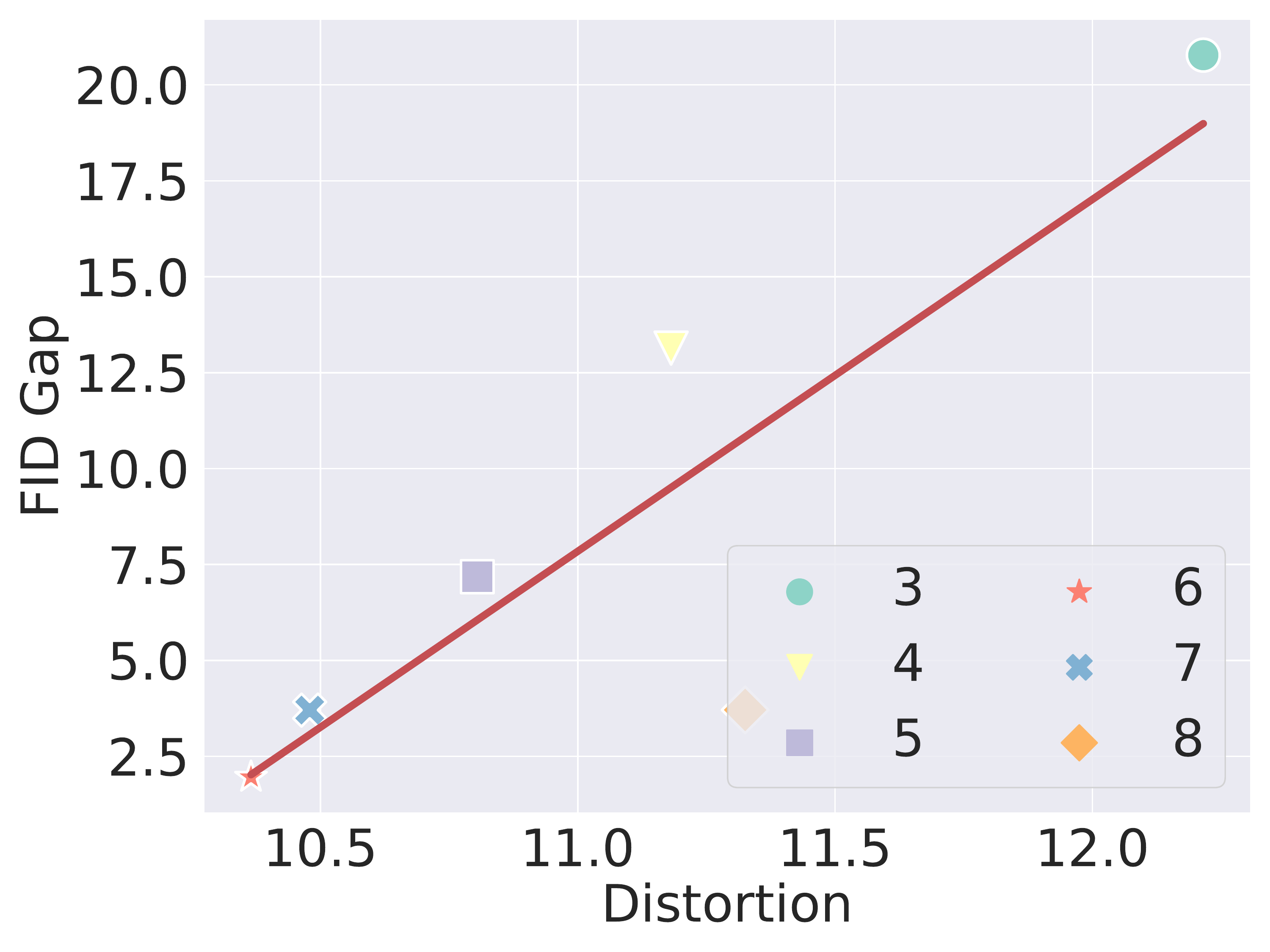}
    \caption{$\theta_{pre}=0.01$, $\rho=0.859$}
    \end{subfigure}
    \caption{ 
    \textbf{Correlation between Distortion metric and FID gap} of StyleGAN2 on LSUN Cat \cite{yu15lsun}. 
    }
    \label{fig:Dist_fid_app_style-cat}
\end{figure}

\begin{figure}[t]
    \centering
    \begin{subfigure}[b]{0.4\columnwidth}
    \centering
    \includegraphics[width=\columnwidth]{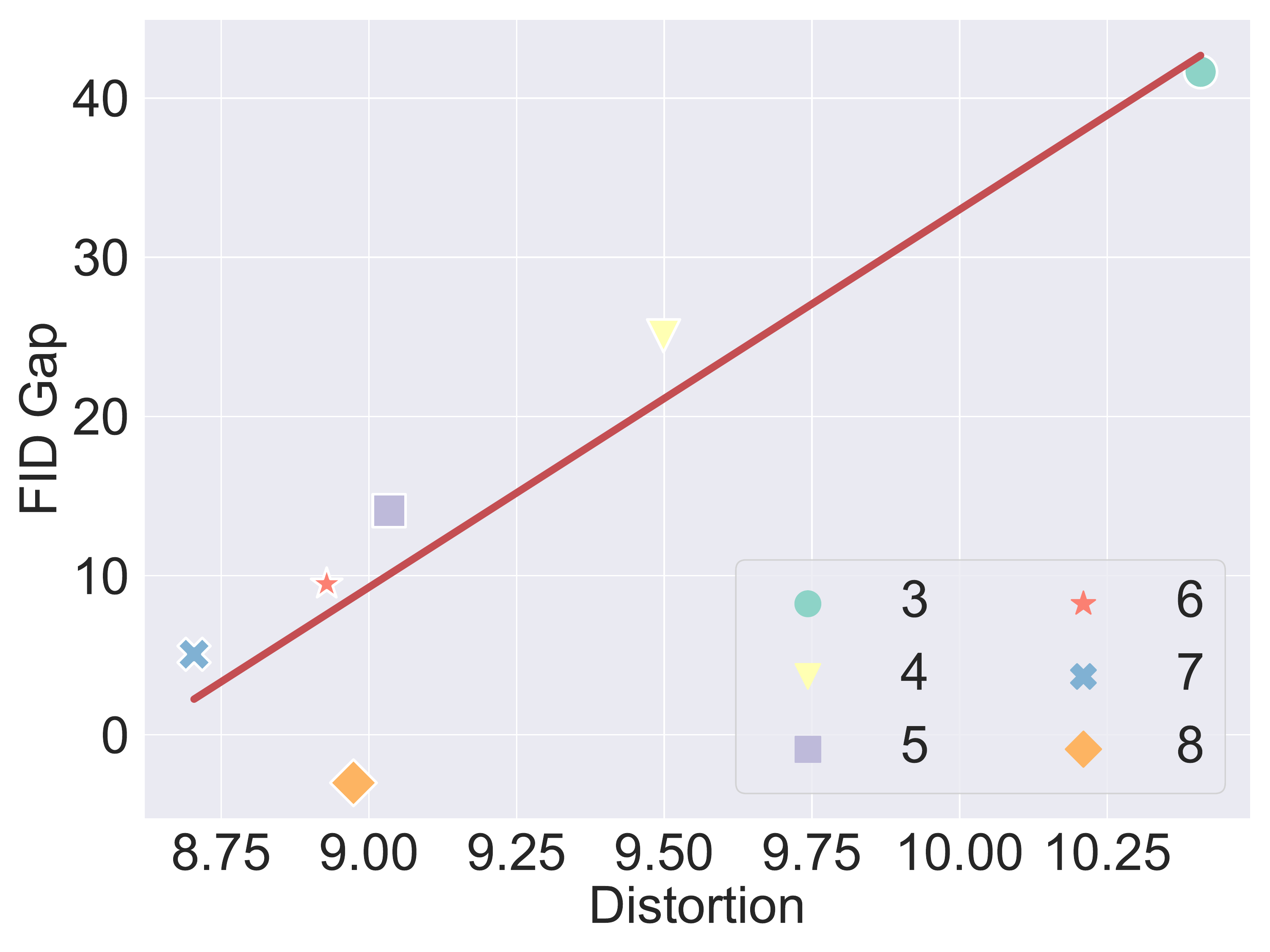}
    \caption{$\theta_{pre}=0.00005$, $\rho=0.926$}
    \end{subfigure}
    \quad
    \begin{subfigure}[b]{0.4\columnwidth}
    \centering
    \includegraphics[width=\columnwidth]{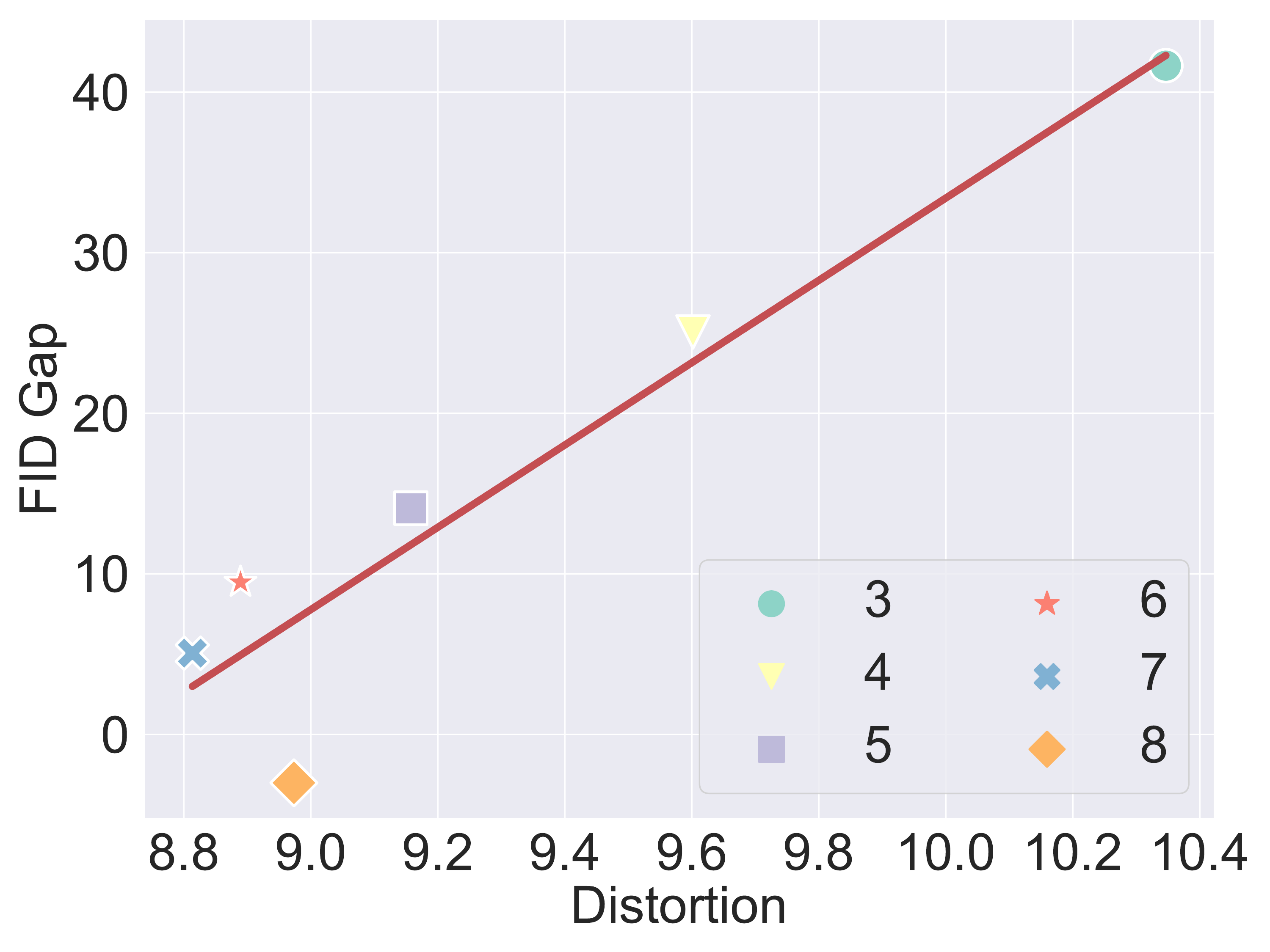}
    \caption{$\theta_{pre}=0.0001$, $\rho=0.945$}
    \end{subfigure}
    \\
    \begin{subfigure}[b]{0.4\columnwidth}
    \centering
    \includegraphics[width=\columnwidth]{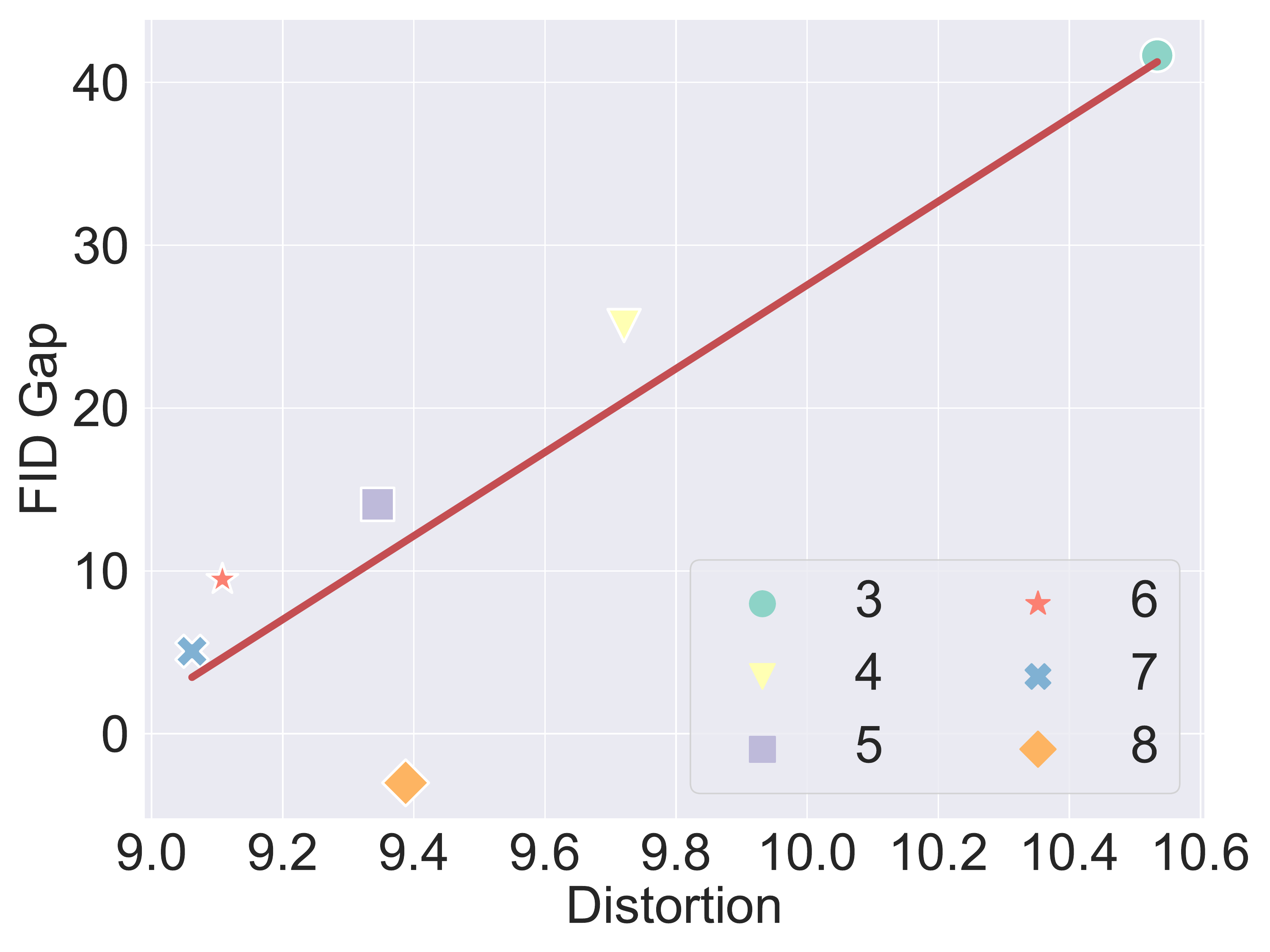}
    \caption{$\theta_{pre}=0.0005$, $\rho=0.883$}
    \end{subfigure}
    \quad
    \begin{subfigure}[b]{0.4\columnwidth}
    \centering
    \includegraphics[width=\columnwidth]{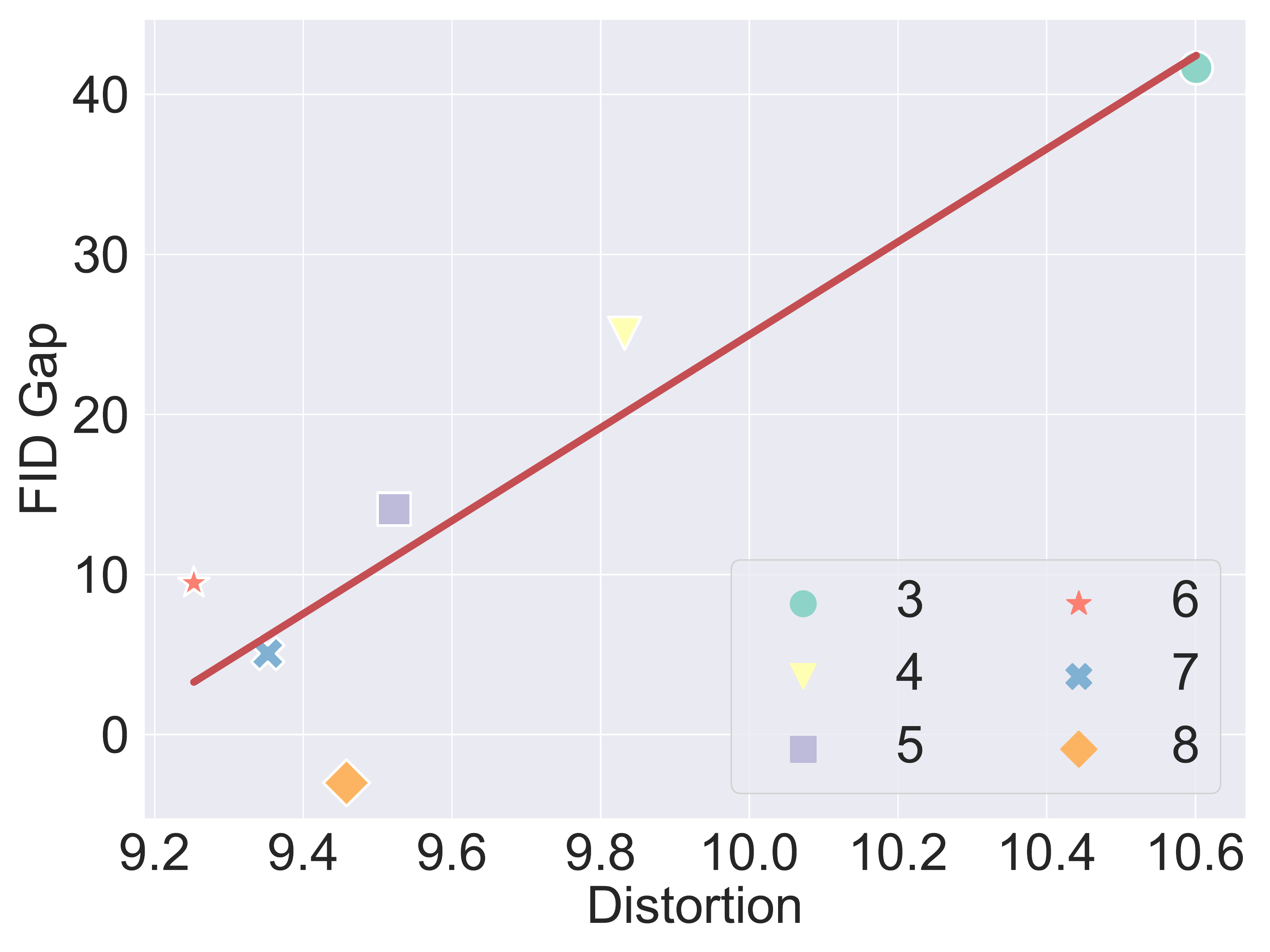}
    \caption{$\theta_{pre}=0.001$, $\rho=0.907$}
    \end{subfigure}
    \caption{ 
    \textbf{Correlation between Distortion metric and FID gap} of StyleGAN2 on LSUN Horse. }
    \label{fig:Dist_fid_app_style-horse}
\end{figure}

\begin{figure}[t]
    \centering
    \begin{subfigure}[b]{0.4\columnwidth}
    \centering
    \includegraphics[width=\columnwidth]{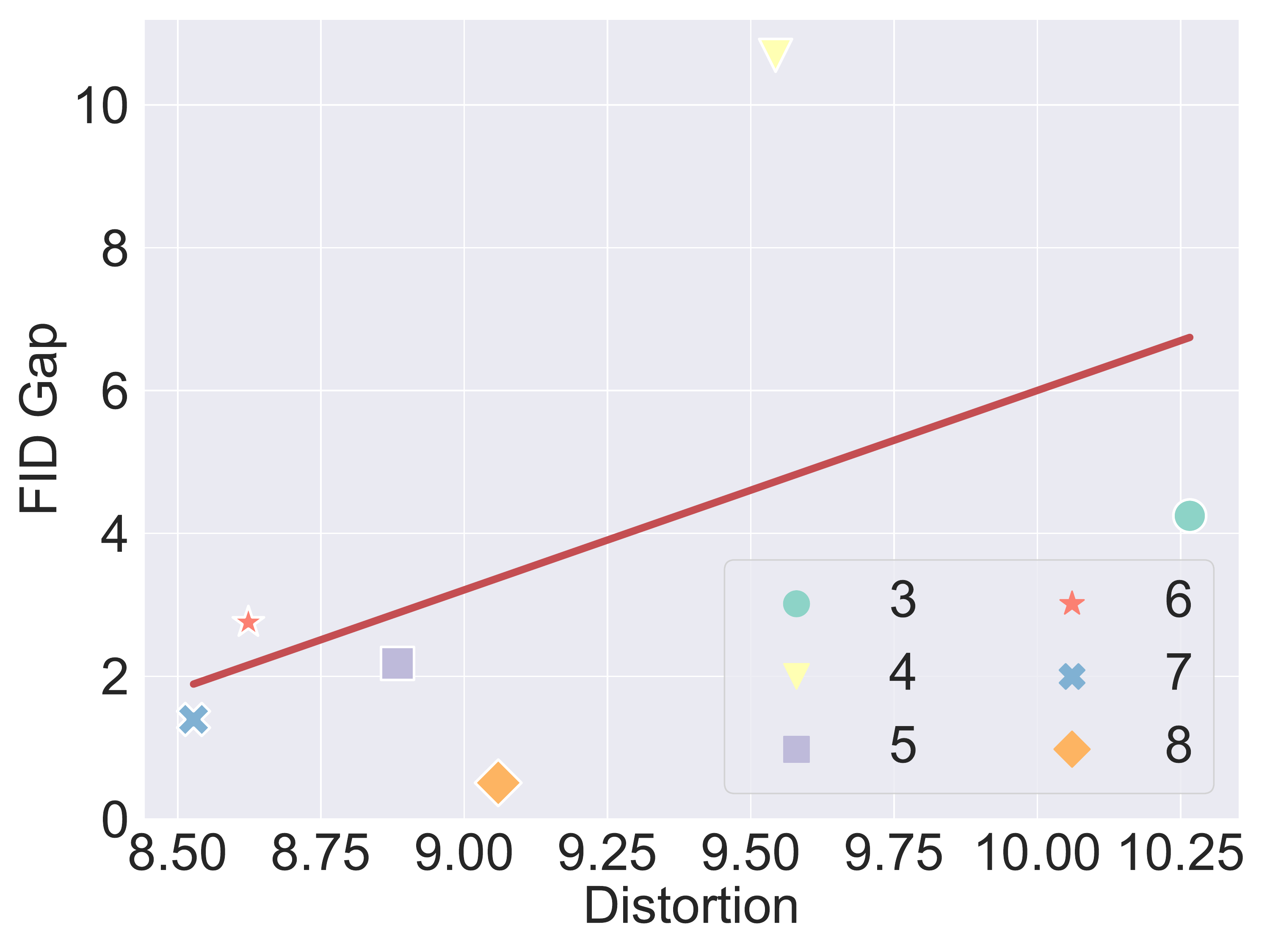}
    \caption{$\theta_{pre}=0.00005$, $\rho=0.496$}
    \end{subfigure}
    \quad
    \begin{subfigure}[b]{0.4\columnwidth}
    \centering
    \includegraphics[width=\columnwidth]{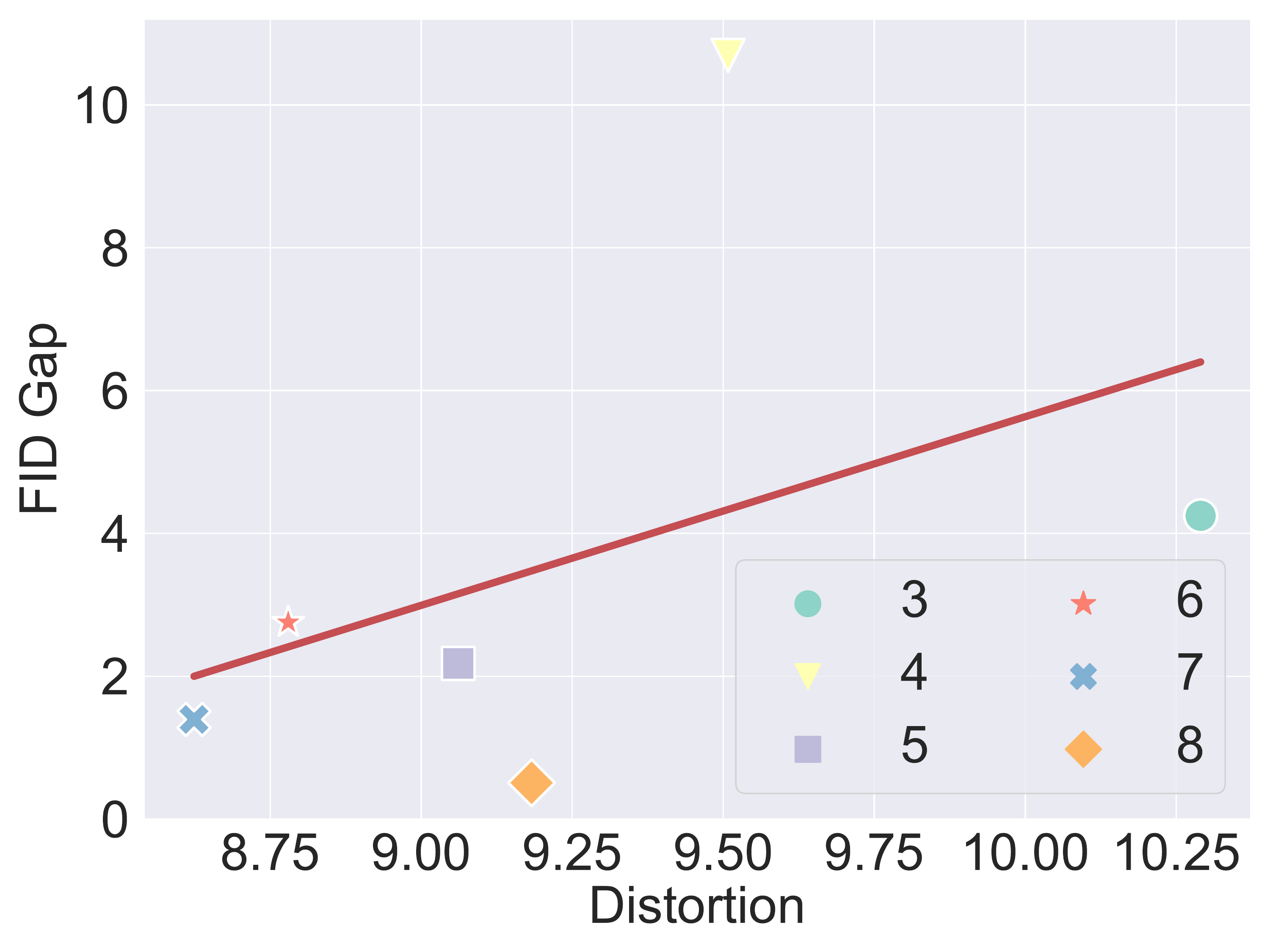}
    \caption{$\theta_{pre}=0.0001$, $\rho=0.430$}
    \end{subfigure}
    \\
    \begin{subfigure}[b]{0.4\columnwidth}
    \centering
    \includegraphics[width=\columnwidth]{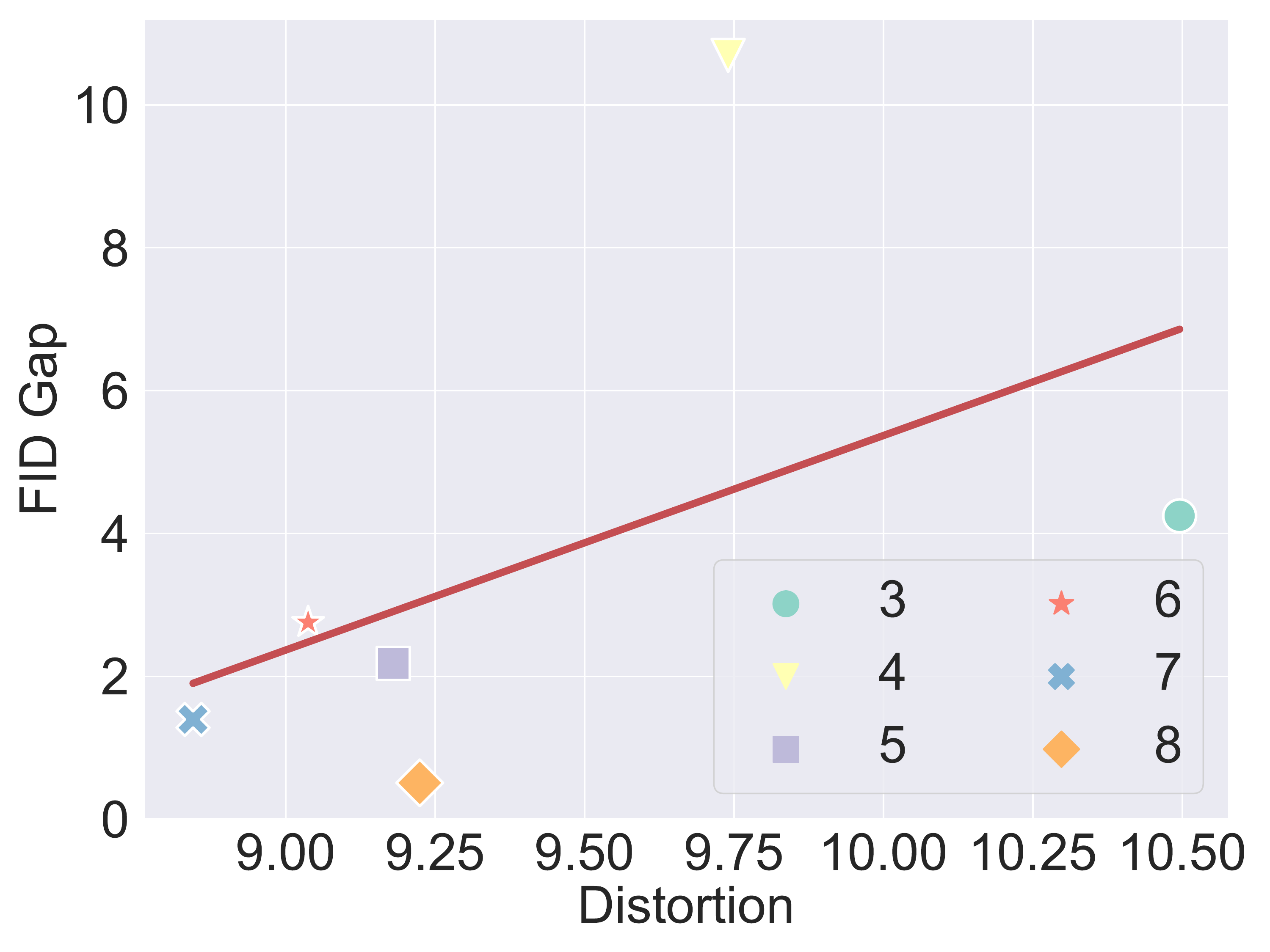}
    \caption{$\theta_{pre}=0.0005$, $\rho=0.494$}
    \end{subfigure}
    \quad
    \begin{subfigure}[b]{0.4\columnwidth}
    \centering
    \includegraphics[width=\columnwidth]{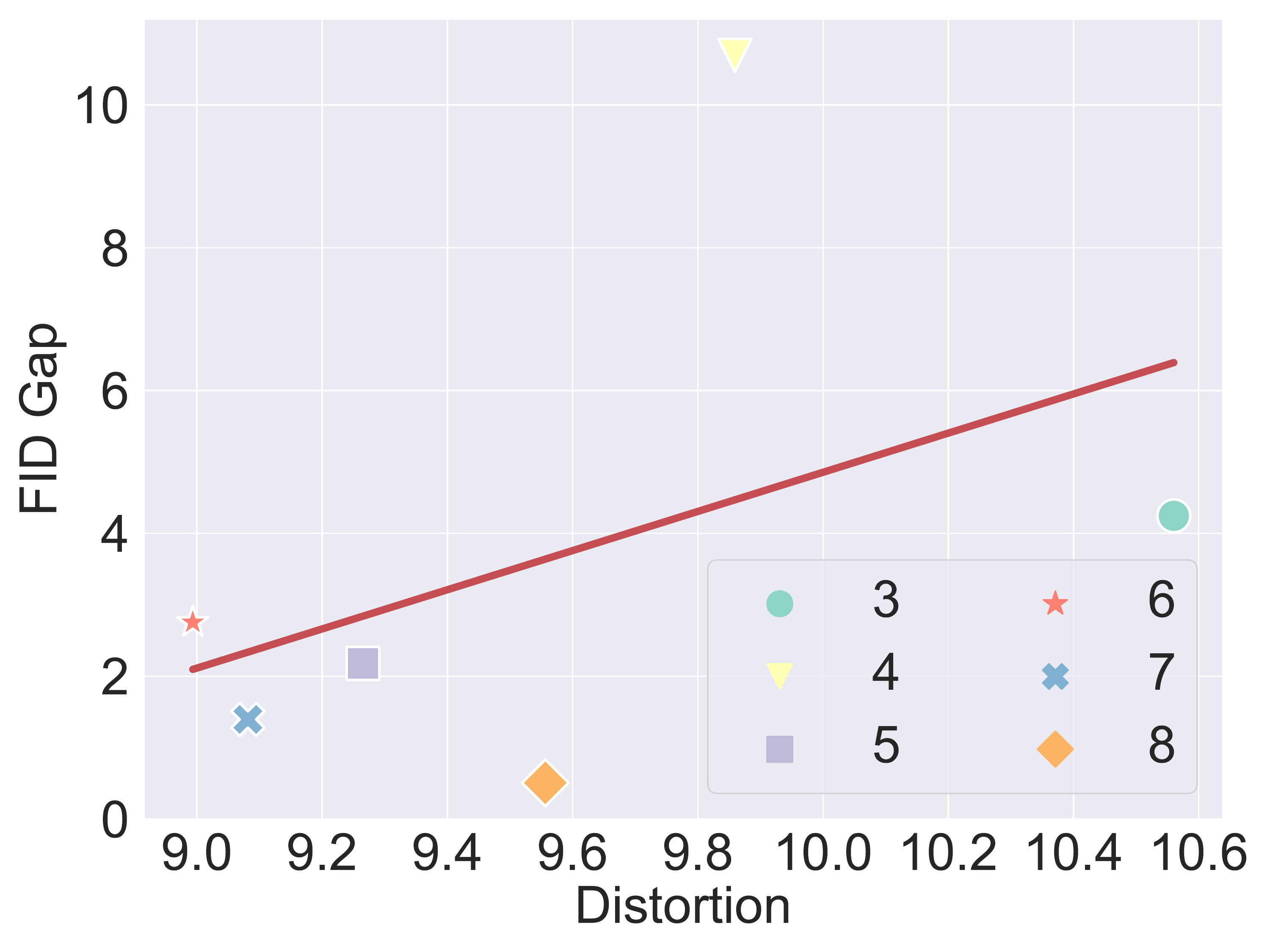}
    \caption{$\theta_{pre}=0.001$, $\rho=0.437$}
    \end{subfigure}
    \caption{ 
    \textbf{Correlation between Distortion metric and FID gap} of StyleGAN2 on LSUN Church.}
\end{figure}

\begin{figure}[t]
    \centering
    \begin{subfigure}[b]{0.4\columnwidth}
    \centering
    \includegraphics[width=\columnwidth]{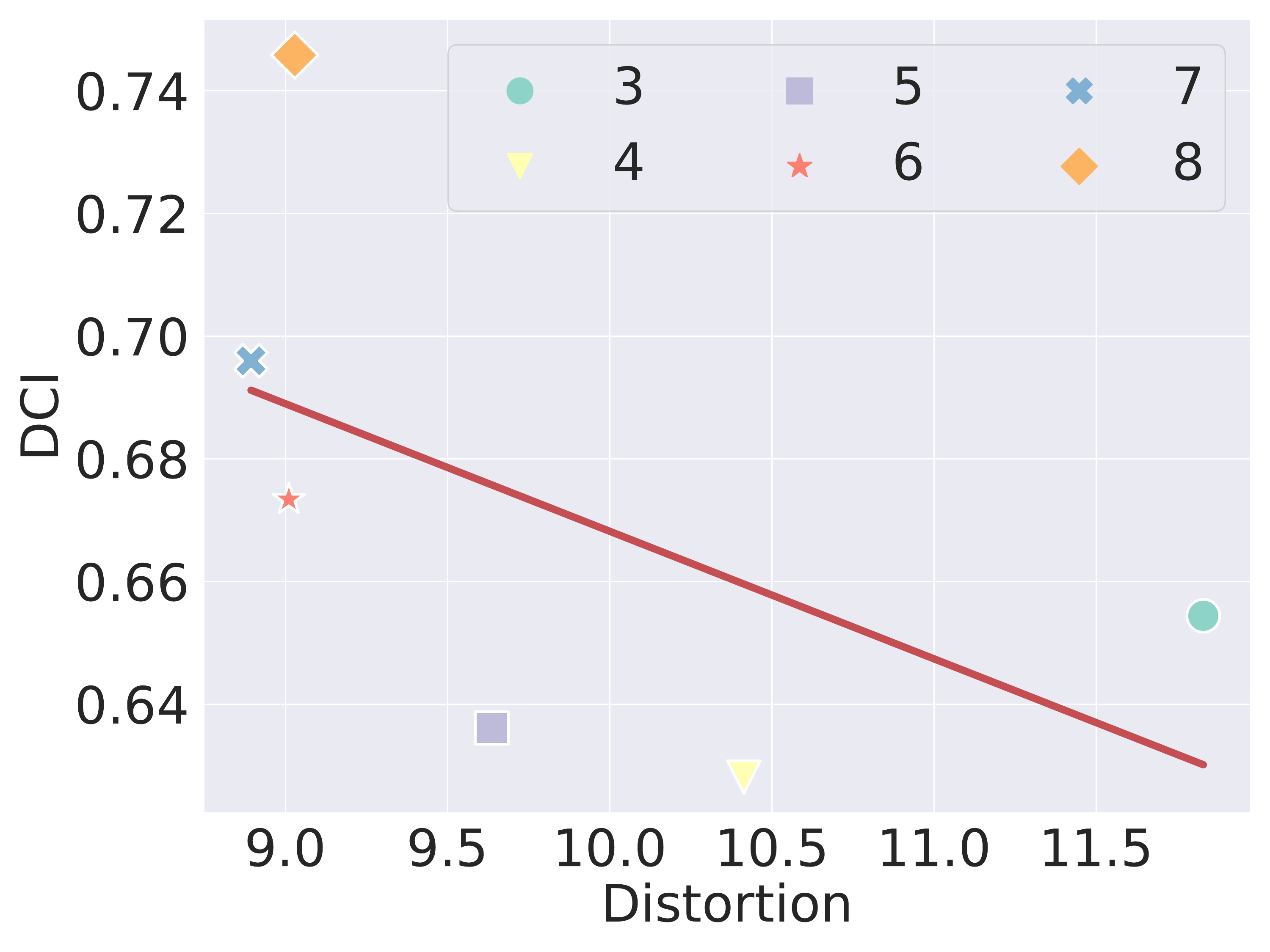}
    \caption{$\theta_{pre}=0.001$, $\rho=-0.545$}
    \end{subfigure}
    \quad
    \begin{subfigure}[b]{0.4\columnwidth}
    \centering
    \includegraphics[width=\columnwidth]{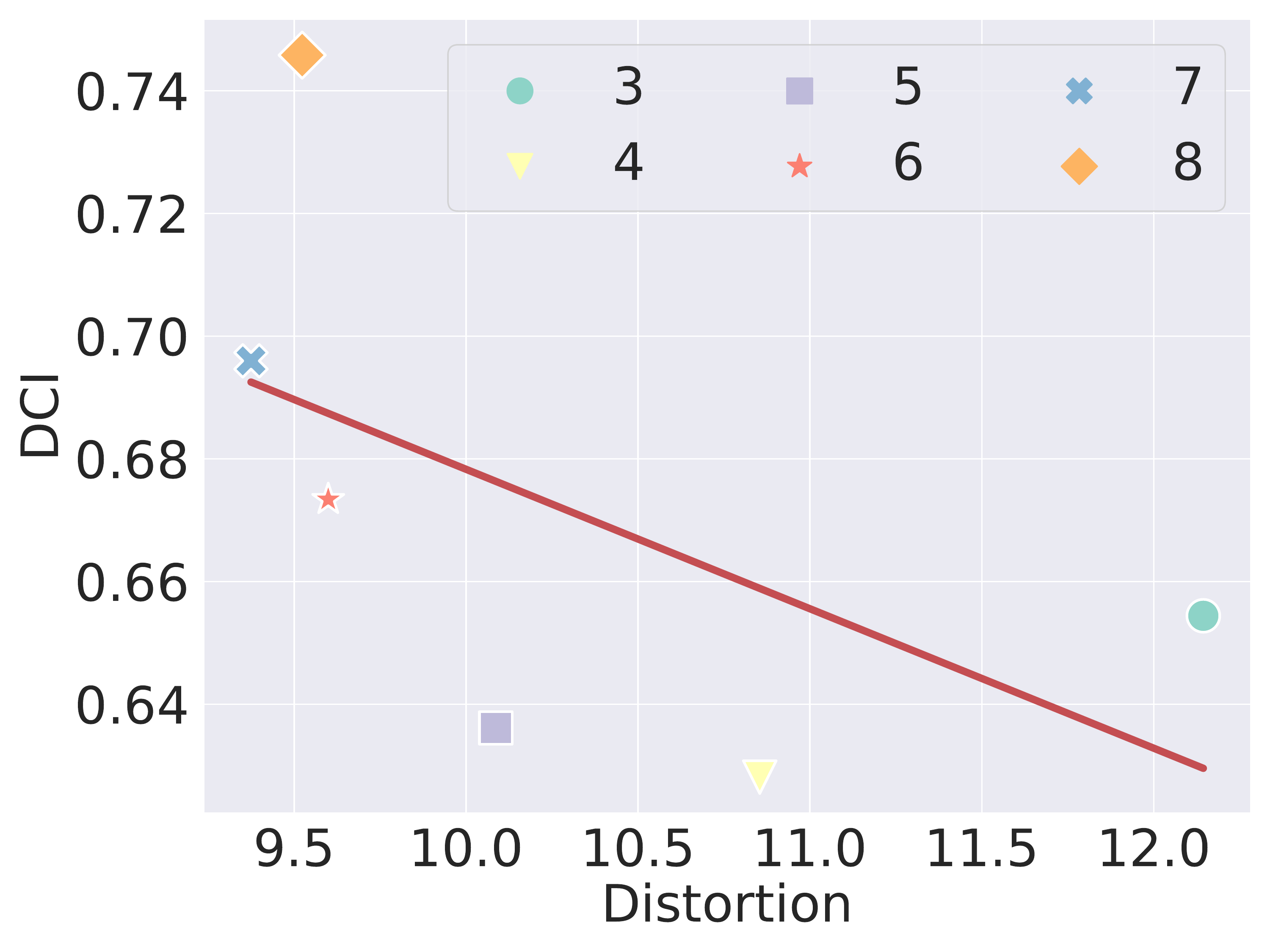}
    \caption{$\theta_{pre}=0.01$, $\rho=-0.555$}
    \end{subfigure} 
    \\
    \begin{subfigure}[b]{0.4\columnwidth}
    \centering
    \includegraphics[width=\columnwidth]{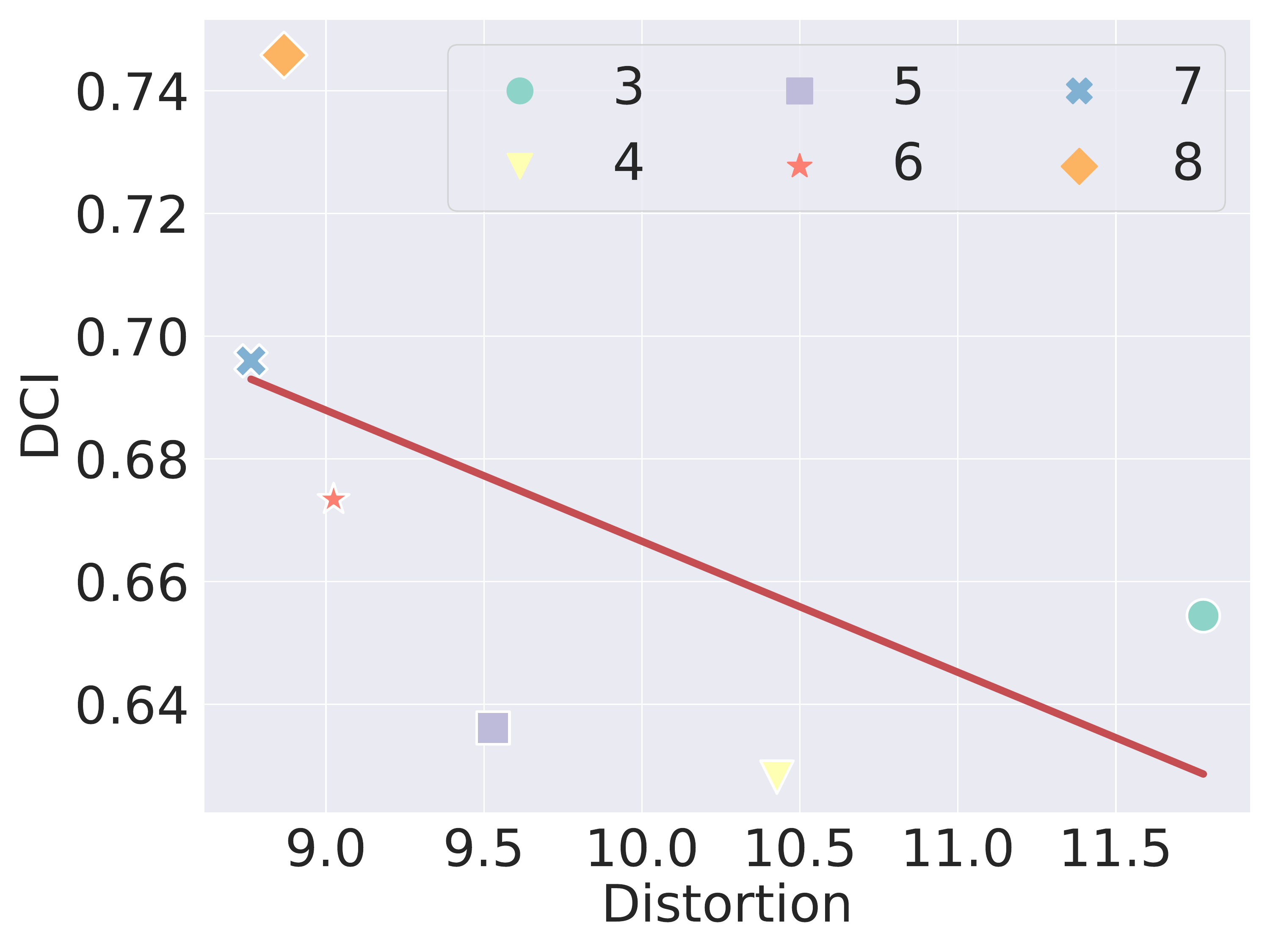}
    \caption{$\theta_{pre}=0.0005$, $\rho=-0.574$}
    \end{subfigure} 
    \quad
    \begin{subfigure}[b]{0.4\columnwidth}
    \centering
    \includegraphics[width=\columnwidth]{figure/DCI/DCI_0.005_corr_-0.5690857100346843.pdf}
    \caption{$\theta_{pre}=0.005$, $\rho=-0.569$}
    \end{subfigure} 
    \caption{
    \textbf{Correlation between Distortion metric and DCI} of StyleGAN2 on FFHQ.
    Each DCI \citep{DCIMetric} score is evaluated for 10k samples of generated images, while the attribute label is generated by 40 attribute classifiers pre-trained on CelebA \citep{CelebA}. As in Fig \ref{fig:Dist_fid_app}, each point and red-line represents the intermediate layers and linear regression, respectively. The Distortion and DCI score show a negative correlation.
    }
\end{figure}

\begin{figure}[t]
    \centering
    \begin{subfigure}[b]{0.4\columnwidth}
    \centering
    \includegraphics[width=\columnwidth]{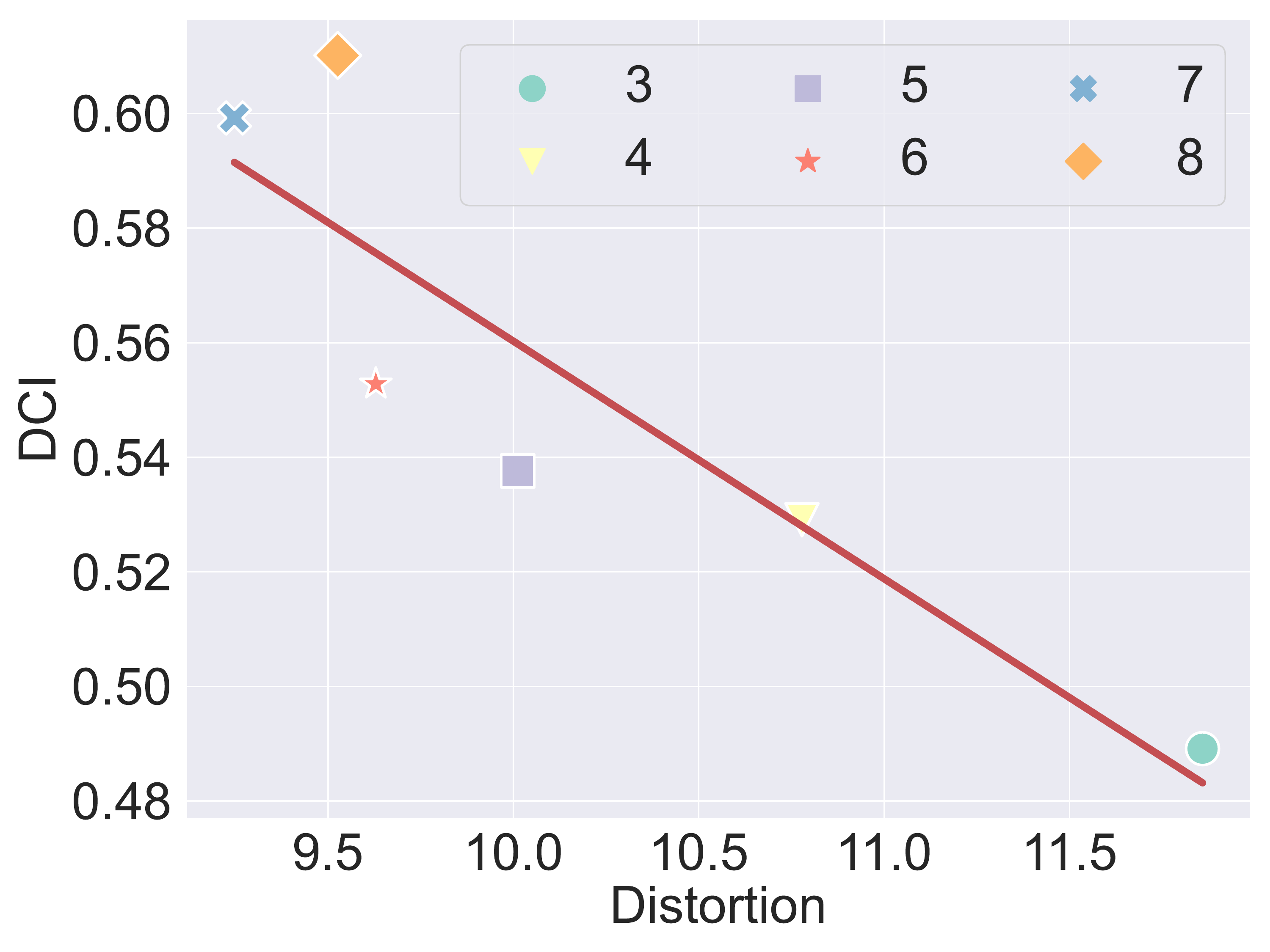}
    \caption{$\theta_{pre}=0.001$, $\rho=-0.896$}
    \end{subfigure}
    \quad
    \begin{subfigure}[b]{0.4\columnwidth}
    \centering
    \includegraphics[width=\columnwidth]{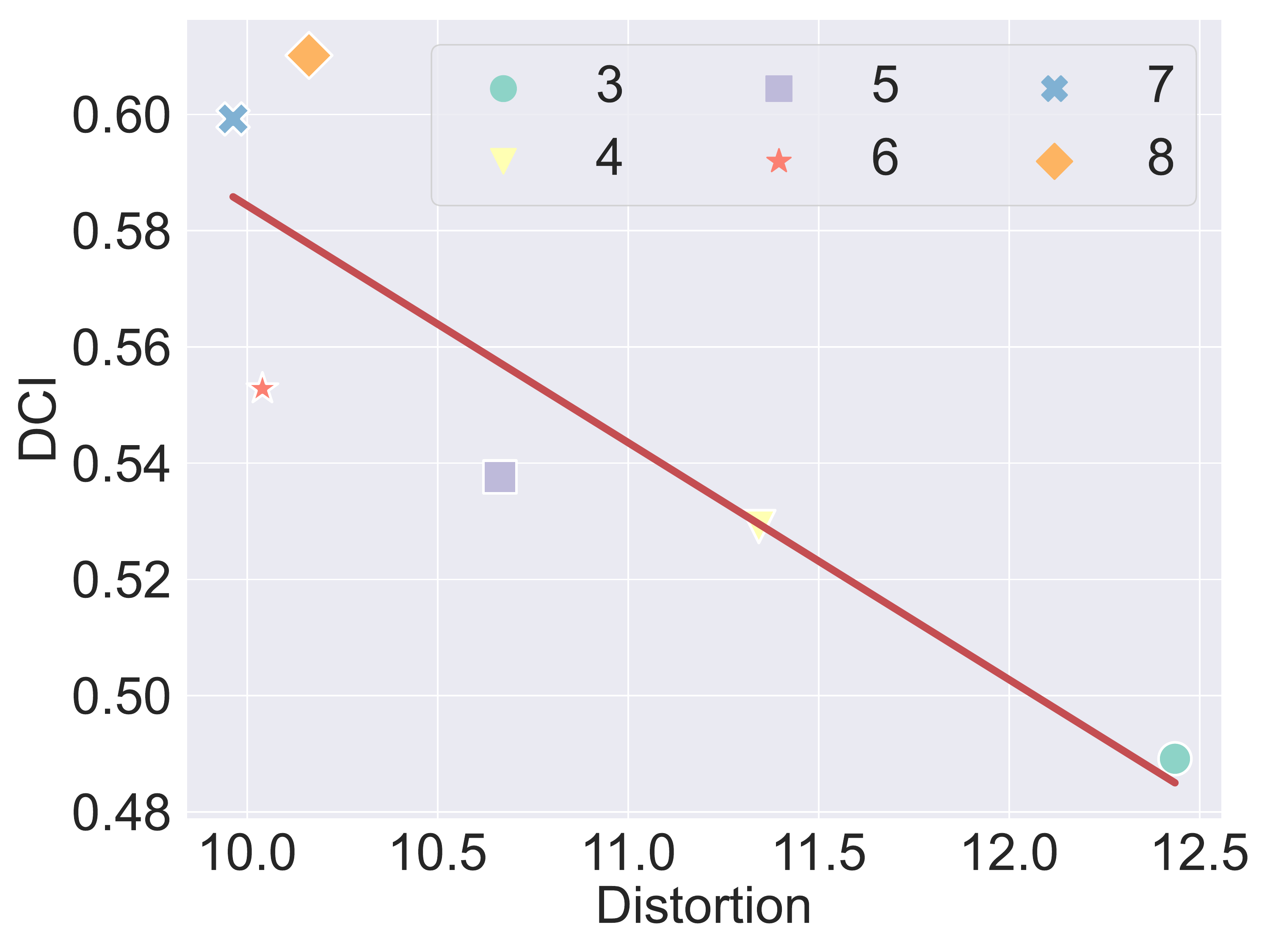}
    \caption{$\theta_{pre}=0.01$, $\rho=-0.868$}
    \end{subfigure} 
    \\
    \begin{subfigure}[b]{0.4\columnwidth}
    \centering
    \includegraphics[width=\columnwidth]{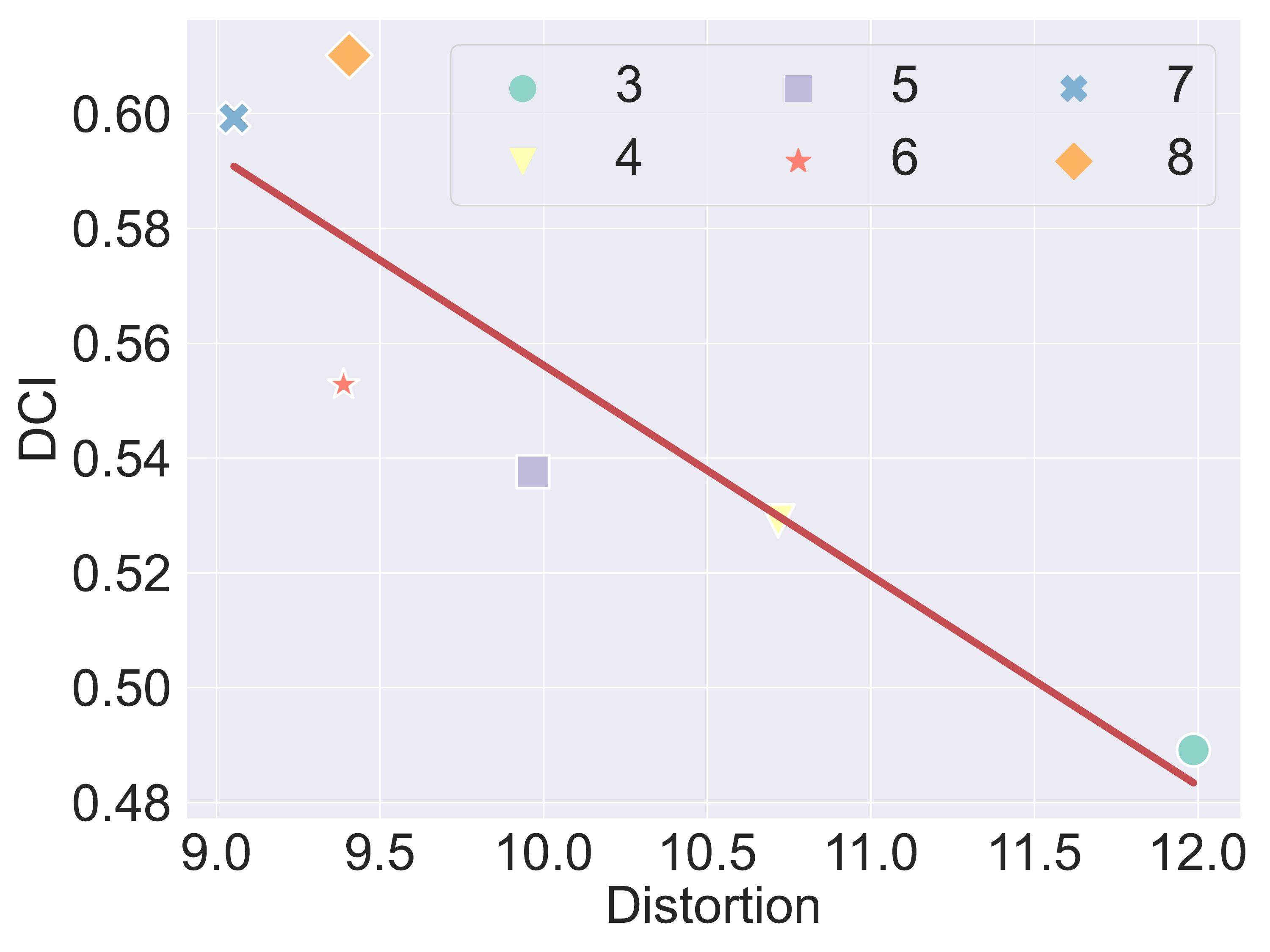}
    \caption{$\theta_{pre}=0.0005$, $\rho=-0.887$}
    \end{subfigure} 
    \quad
    \begin{subfigure}[b]{0.4\columnwidth}
    \centering
    \includegraphics[width=\columnwidth]{figure/DCI/geodesic_DCI_StyleGAN2-e_0.005_corr_-0.9105663103655632.pdf}
    \caption{$\theta_{pre}=0.005$, $\rho=-0.910$}
    \end{subfigure} 
    \caption{
    \textbf{Correlation between Distortion metric and DCI} of StyleGAN2 with config E on FFHQ.
    }
\end{figure}

\begin{figure}[t]
    \centering
    \begin{subfigure}[b]{0.4\columnwidth}
    \centering
    \includegraphics[width=\columnwidth]{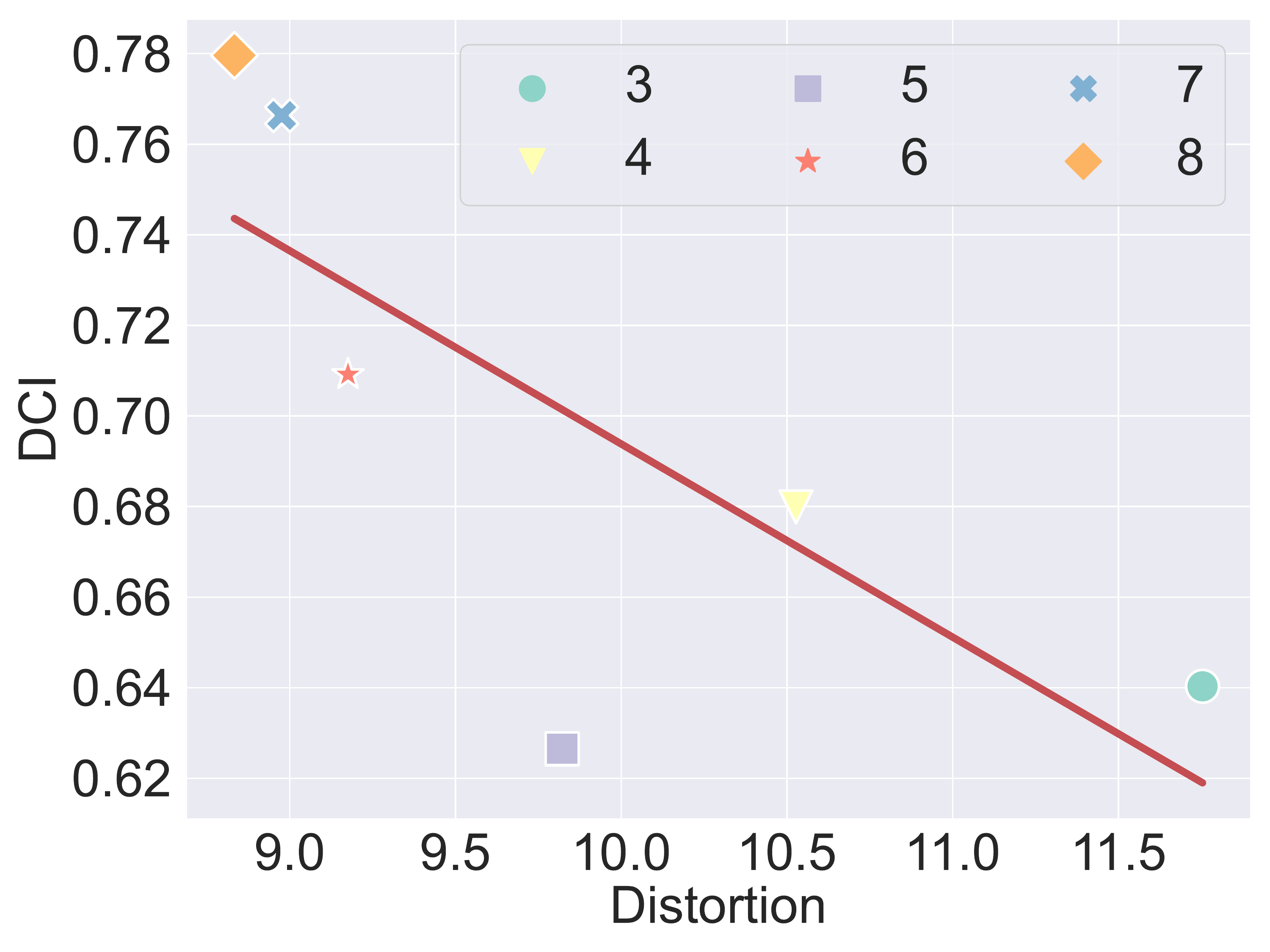}
    \caption{$\theta_{pre}=0.001$, $\rho=-0.756$}
    \end{subfigure}
    \quad
    \begin{subfigure}[b]{0.4\columnwidth}
    \centering
    \includegraphics[width=\columnwidth]{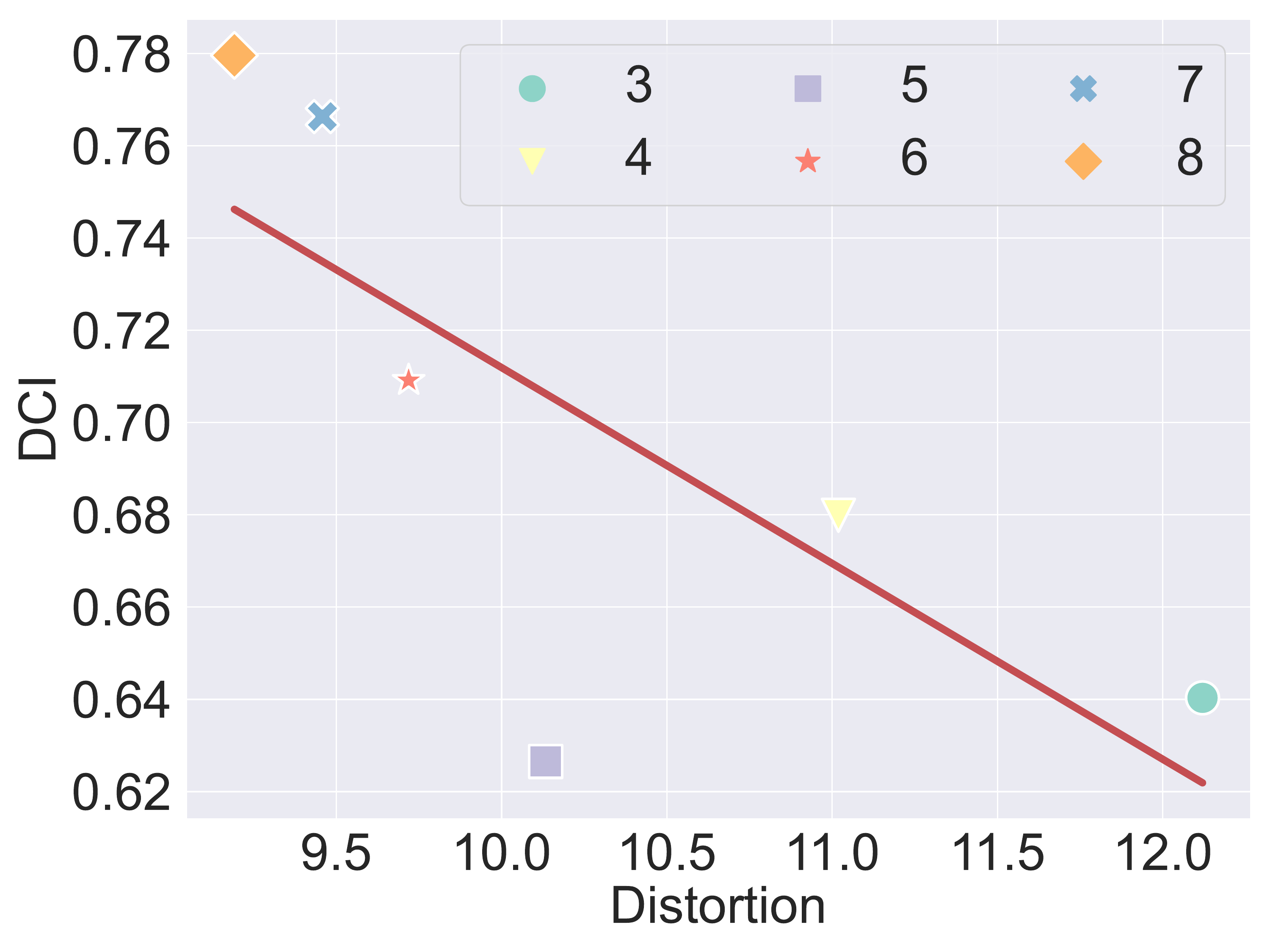}
    \caption{$\theta_{pre}=0.01$, $\rho=-0.740$}
    \end{subfigure} 
    \\
    \begin{subfigure}[b]{0.4\columnwidth}
    \centering
    \includegraphics[width=\columnwidth]{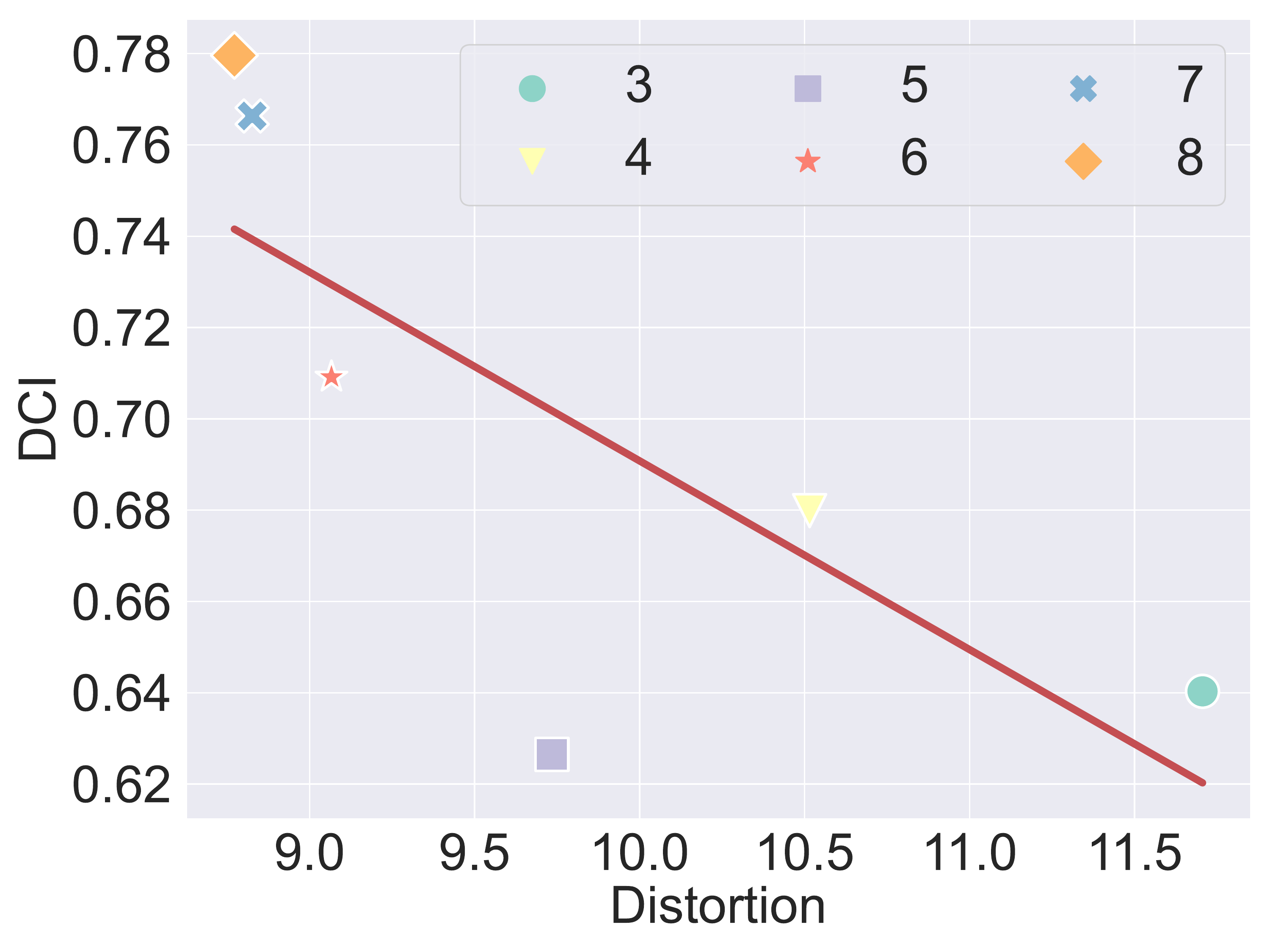}
    \caption{$\theta_{pre}=0.0005$, $\rho=-0.752$}
    \end{subfigure} 
    \quad
    \begin{subfigure}[b]{0.4\columnwidth}
    \centering
    \includegraphics[width=\columnwidth]{figure/DCI/geodesic_DCI_StyleGAN1_0.005_corr_-0.742038540929386.pdf}
    \caption{$\theta_{pre}=0.005$, $\rho=-0.742$}
    \end{subfigure} 
    \caption{
    \textbf{Correlation between Distortion metric and DCI} of StyleGAN1 on FFHQ.
    }
\end{figure}

\clearpage
\section{Additional Experimental results} \label{sec:additional_results}
\begin{figure}[h]
    \centering
    \begin{subfigure}[b]{0.4\columnwidth}
    \centering
    \includegraphics[width=\columnwidth]{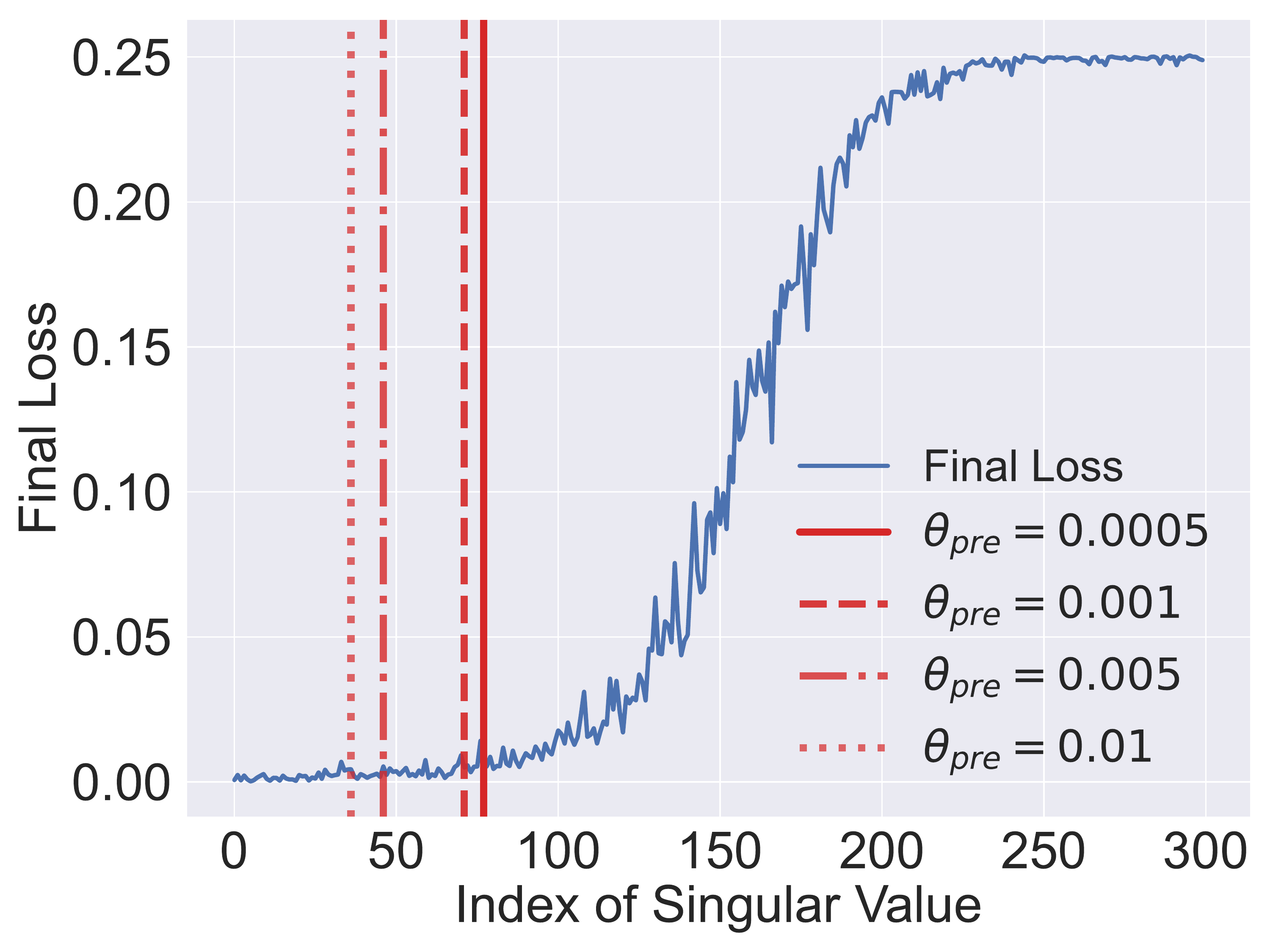}
    \caption{$\|\wvar_{ptb}-\wvar_{init}\|=0.5$}
    \end{subfigure}
    \quad
    \begin{subfigure}[b]{0.4\columnwidth}
    \centering
    \includegraphics[width=\columnwidth]{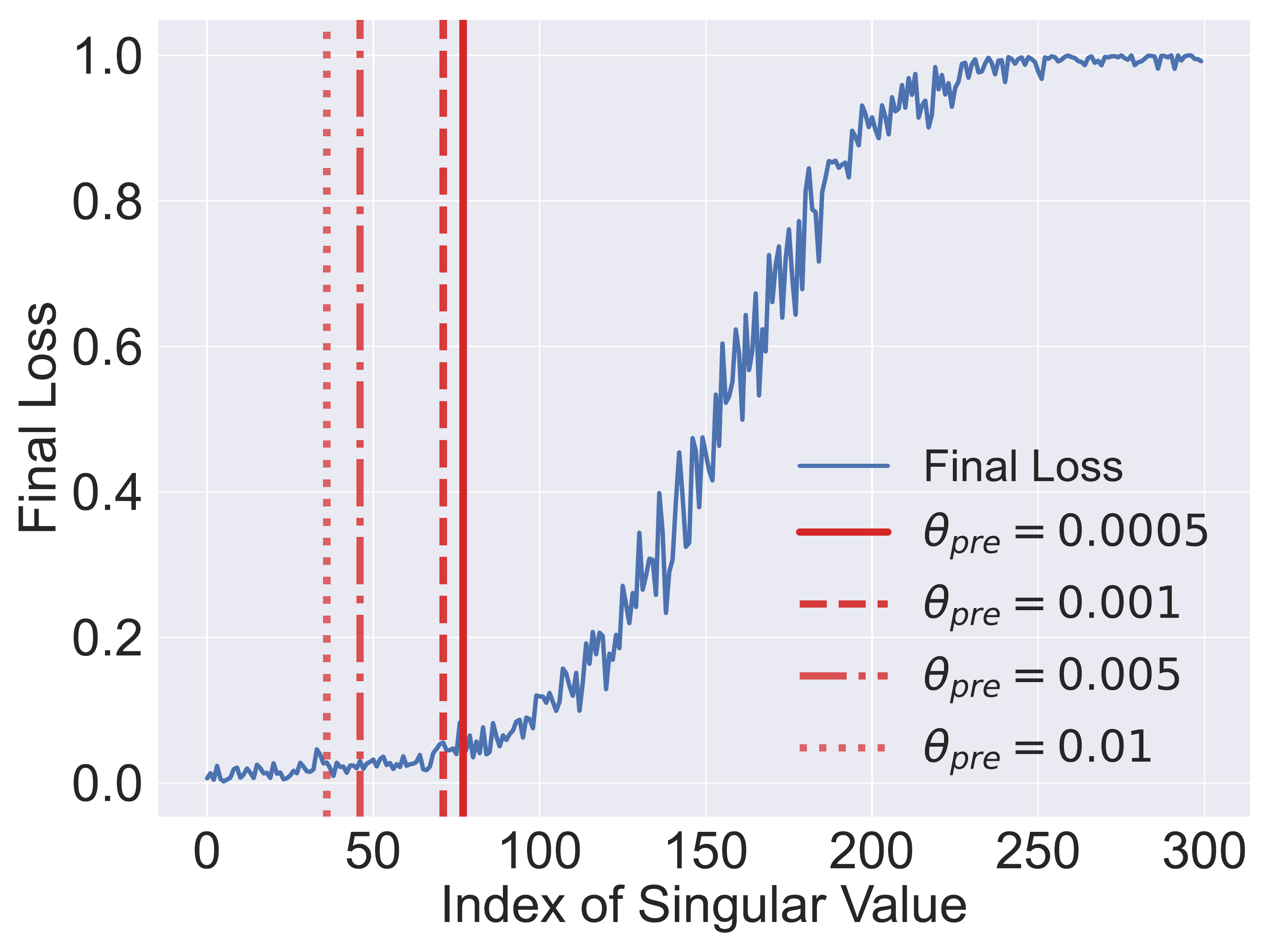}
    \caption{$\|\wvar_{ptb}-\wvar_{init}\|=1$}
    \end{subfigure} 
    \\
    \begin{subfigure}[b]{0.4\columnwidth}
    \centering
    \includegraphics[width=\columnwidth]{figure/OffMfd/Off_mfd_final_ptb_2.pdf}
    \caption{$\|\wvar_{ptb}-\wvar_{init}\|=2$}
    \end{subfigure} 
    \quad
    \begin{subfigure}[b]{0.4\columnwidth}
    \centering
    \includegraphics[width=\columnwidth]{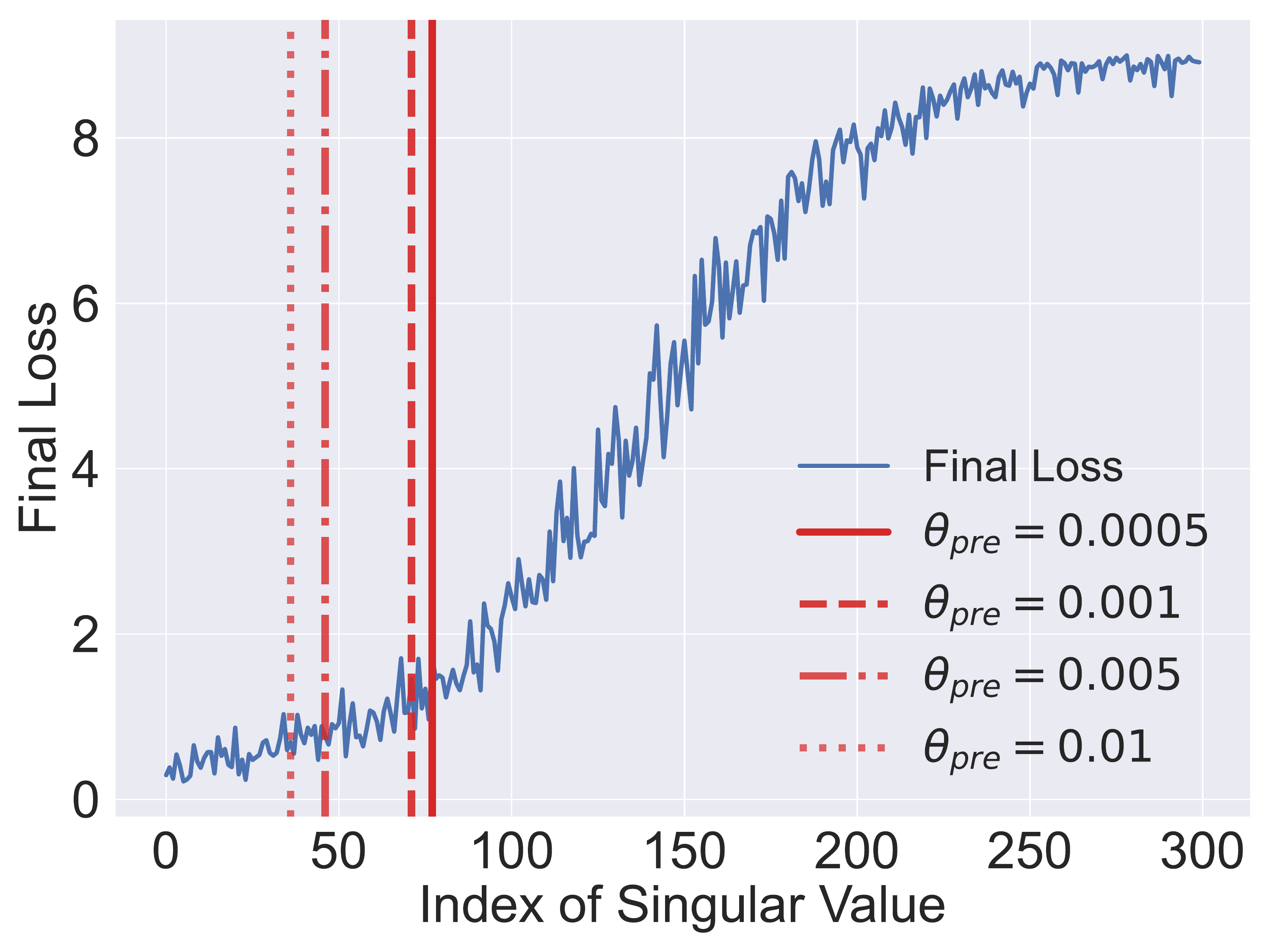}
    \caption{$\|\wvar_{ptb}-\wvar_{init}\|=3$}
    \end{subfigure} 
    \caption{
    \textbf{Off-manifold Results} in $\wset$-space of StyleGAN2 on FFHQ. The red vertical lines denote the estimated local dimension for various $\theta_{pre}$. The small final loss implies that the linear perturbation along that Local Basis component stays inside the learned latent manifold. For every perturbation intensity, there is a transition point from a slow increase to a sharp increase, which is interpreted as an escape from the manifold. The results demonstrate that the proposed dimension estimation algorithm finds a reasonable point without crossing the transition point. In our case, these results are interpreted as choosing the principal part of the manifold without overestimating its dimension.
    }
\end{figure}

\begin{figure}
    \centering
    \begin{subfigure}[b]{0.4\columnwidth}
    \centering
    \includegraphics[width=\columnwidth]{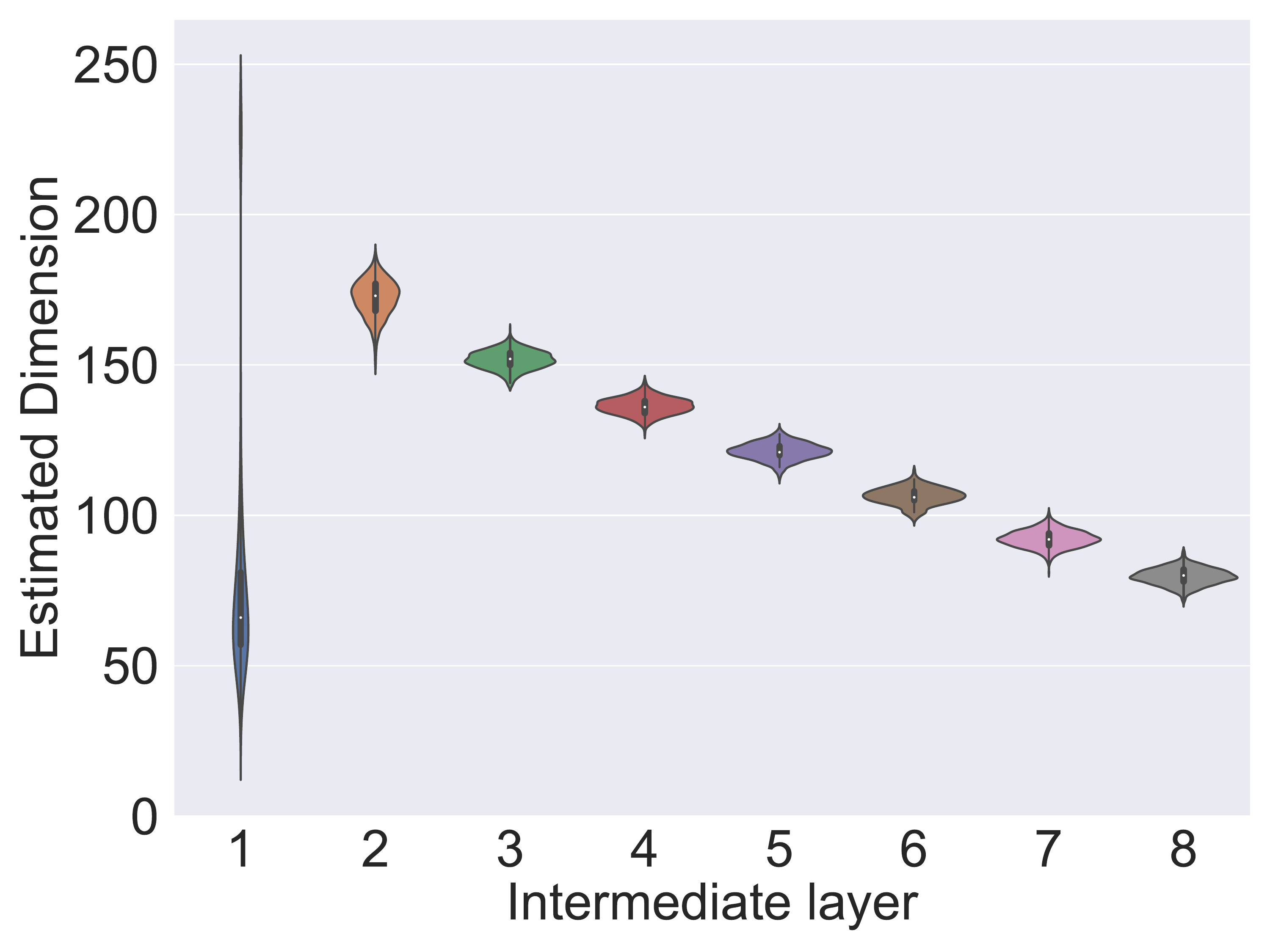}
    \caption{$\theta_{pre} = 0.0005$}
    \end{subfigure} 
    \quad
    \begin{subfigure}[b]{0.4\columnwidth}
    \centering
    \includegraphics[width=\columnwidth]{figure/RankStats/RankStats_SVthres_0.001.pdf}
    \caption{$\theta_{pre} = 0.001$}
    \end{subfigure}
    \\
    \begin{subfigure}[b]{0.4\columnwidth}
    \centering
    \includegraphics[width=\columnwidth]{figure/RankStats/RankStats_SVthres_0.005.pdf}
    \caption{$\theta_{pre} = 0.005$}
    \end{subfigure} 
    \quad
    \begin{subfigure}[b]{0.4\columnwidth}
    \centering
    \includegraphics[width=\columnwidth]{figure/RankStats/RankStats_SVthres_0.01.pdf}
    \caption{$\theta_{pre} = 0.01$}
    \end{subfigure} 
    \caption{
    \textbf{Local Dimension Distribution} of the intermediate layers in the mapping network of StyleGAN2 on FFHQ. Each figure presents the distribution of estimated local dimension at each intermediate layer as we vary $\theta_{pre}$. The distributions are illustrated for 1k samples, respectively.
    The algorithm gives a rather unstable dimension estimate on the most unsparse first layer (Fig \ref{fig:sv_dist}) due to its isotropic gaussian assumption. However, this phenomenon is not observed in the layers with moderate depth, i.e., from 3 to 8. As we introduce the higher preprocessing ratio $\theta_{pre}$, the algorithm gives more strict, i.e., smaller, dimension estimates. Nevertheless, the relative trend between layers is the same. The deeper the latent manifold, the smaller its dimension.
    }
\end{figure}

\begin{figure}
    \centering
    \begin{subfigure}[b]{0.4\columnwidth}
    \centering
    \includegraphics[width=\columnwidth]{figure/traversal_2strip/366745668_Local_Basis_2dim_x0_y1_ptb5.0_sublayer8.pdf}
    \caption{Traversal along Axis 0 and 1}
    \label{fig:rank_traversal_app_1}
    \end{subfigure}
    \quad
    \begin{subfigure}[b]{0.4\columnwidth}
    \centering
    \includegraphics[width=\columnwidth]{figure/traversal_2strip/366745668_Local_Basis_svth0.01_2dim_x34_y35_ptb5.0_sublayer8.pdf}
    \caption{Traversal along Axis 34 and 35 - $\theta_{pre}=0.01$}
    \label{fig:rank_traversal_app_2}
    \end{subfigure} 
    \\
    \begin{subfigure}[b]{0.4\columnwidth}
    \centering
    \includegraphics[width=\columnwidth]{figure/traversal_2strip/366745668_Local_Basis_svth0.005_2dim_x43_y44_ptb5.0_sublayer8.pdf}
    \caption{Traversal along Axis 43 and 44 - $\theta_{pre}=0.005$}
    \label{fig:rank_traversal_app_3}
    \end{subfigure} 
    \quad
    \begin{subfigure}[b]{0.4\columnwidth}
    \centering
    \includegraphics[width=\columnwidth]{figure/traversal_2strip/366745668_Local_Basis_svth0.001_2dim_x67_y68_ptb5.0_sublayer8.pdf}
    \caption{Traversal along Axis 67 and 68 - $\theta_{pre}=0.001$}
    \label{fig:rank_traversal_app_4}
    \end{subfigure} 
    \\
    \begin{subfigure}[b]{0.4\columnwidth}
    \centering
    \includegraphics[width=\columnwidth]{figure/traversal_2strip/366745668_Local_Basis_svth0.0005_2dim_x78_y79_ptb5.0_sublayer8.pdf}
    \caption{Traversal along Axis 78 and 79 - $\theta_{pre}=0.0005$}
    \label{fig:rank_traversal_app_5}
    \end{subfigure} 
    \\
    \begin{subfigure}[b]{0.4\columnwidth}
    \centering
    \includegraphics[width=\columnwidth]{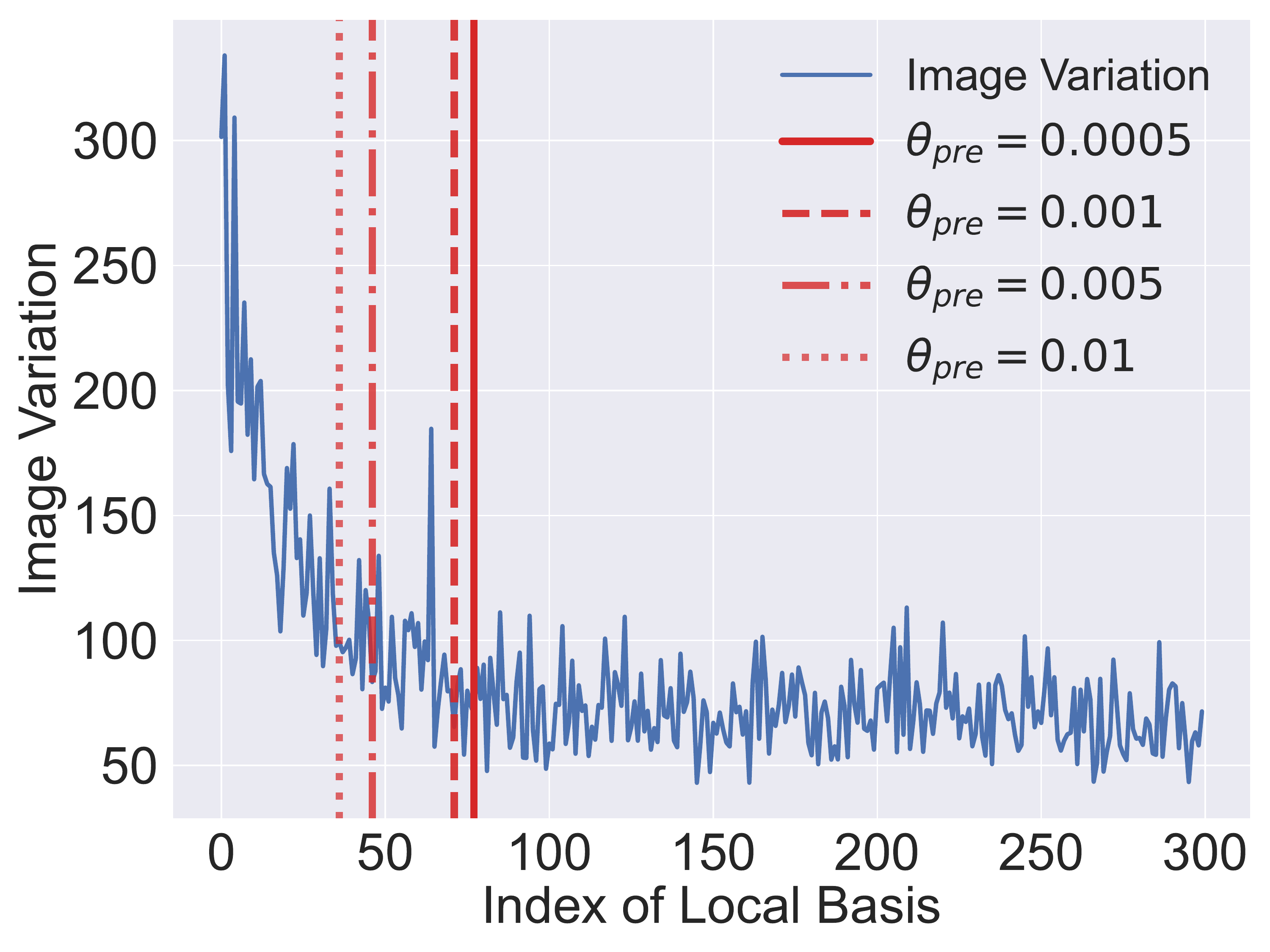}
    \caption{Image variation intensity}
    \label{fig:rank_image_var_app}
    \end{subfigure} 
    \caption{ 
    \textbf{Local Dimension Evaluation in Image Space} of StyleGAN2 on FFHQ. Figure \ref{fig:rank_traversal_app_2}-\ref{fig:rank_traversal_app_5} show image traversals along the $(d-1)$-th and $d$-th axis where $d$ denotes the estimated local dimension for each $\theta_{pre}$.
    Fig \ref{fig:rank_image_var_app} presents the image variation intensity $\| \nabla_{\vvar_{i}^{\wvar}}g(\wvar) \|_{F}$ along each Local Basis $\vvar_{i}$. $\| \nabla_{\vvar_{i}^{\wvar}}g(\wvar) \|_{F}$ is evaluated by the finite difference with stepsize=0.01. Note that while Local Basis is discovered by analyzing only the subnetwork from the input to target latent space, the corresponding image variation monotonically decreases and saturates. Moreover, the estimated local dimension includes the major variations in the image space. This can be observed in the image traversals. The image traversals around the estimated dimension (Fig \ref{fig:rank_traversal_app_2}-\ref{fig:rank_traversal_app_5}) show much smaller image variations compared to Axis 0 and 1 (Fig \ref{fig:rank_traversal_app_1}).
    }
\end{figure}

\begin{figure}[]
    \centering
    \begin{subfigure}[b]{0.4\columnwidth}
    \centering
    \includegraphics[width=\columnwidth]{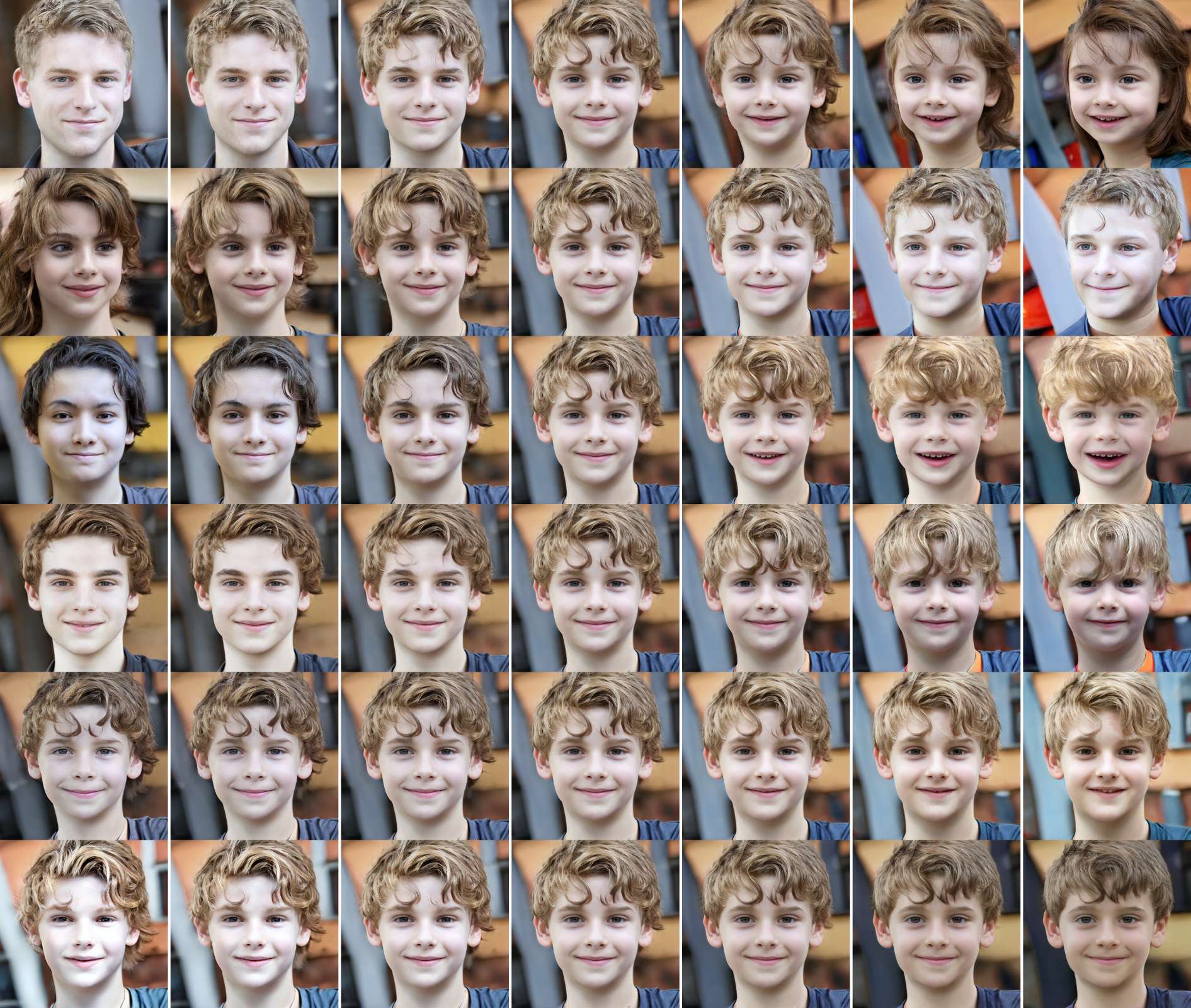}
    \caption{Axes $\{ 0,1,15,30,43,44 \}$ from top to bottom}
    \end{subfigure}
    \quad
    \begin{subfigure}[b]{0.4\columnwidth}
    \centering
    \includegraphics[width=\columnwidth]{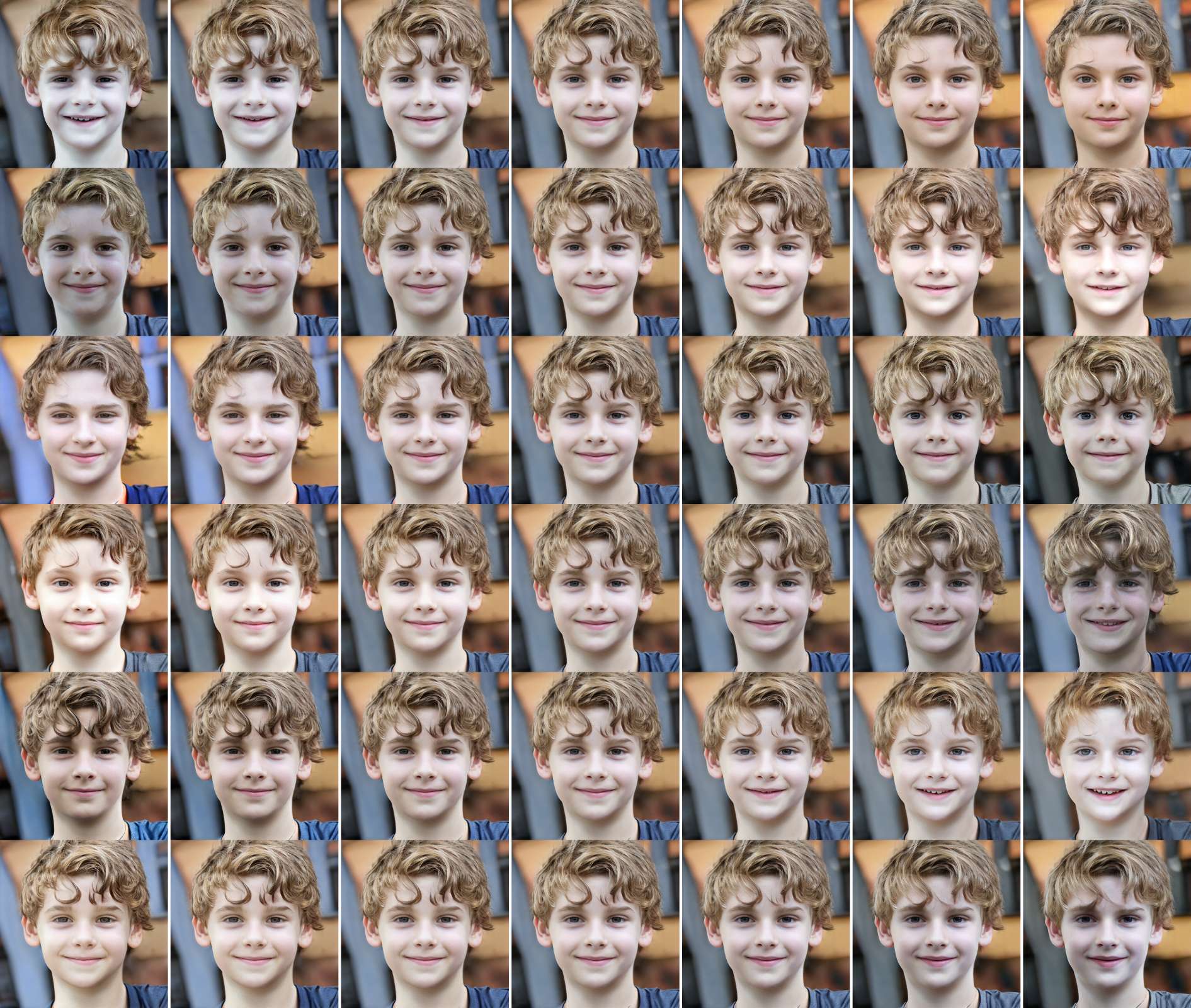}
    \caption{Axes $\{100,150,200,300,400,500\}$}
    \end{subfigure}
    \caption{
    \textbf{Linear Traversals along the various LB axis} on StyleGAN2-FFHQ. At this latent variable, the estimated dimension with $\theta_{pre}=0.01$ is $d=44$. When we traverse along the axis $i \gg d$, the images present minor variations compared to the axis $i \leq d$. Traversals along the 1 - 44 axes display various semantic variations such as gender, skin color, age, etc. 
    }
    \label{fig:1dim_comparison_full}
\end{figure}

\clearpage
\section{Comparison of Semantic Disentanglement} \label{sec:additional_semantic}
Using our Distortion metric, we chose the Max-distorted layer 3 and Min-distorted layer 7 from the mapping network of StyleGAN2-FFHQ \cite{karras2020analyzing}. These two layers are compared with the renowned $\wset$-space (layer 8). Each image traversal is obtained by perturbing a latent variable along the global basis from GANSpace \cite{harkonen2020ganspace}.

\vskip 1in

\begin{figure*}[h]
    \centering
    \begin{subfigure}[b]{0.7\textwidth}
    \centering
    \includegraphics[width=\columnwidth]{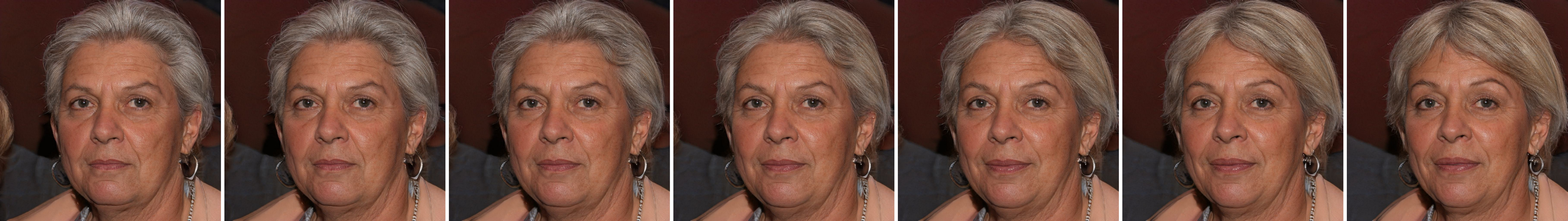}
    \caption{Max-distorted layer 3}
    \end{subfigure} 
    \\
    \begin{subfigure}[b]{0.7\textwidth}
    \centering
    \includegraphics[width=\columnwidth]{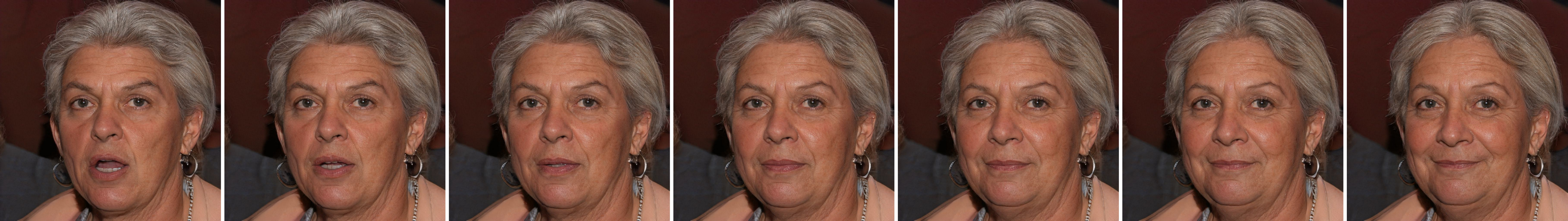}
    \caption{Min-distorted layer 7}
    \end{subfigure} 
    \\
    \begin{subfigure}[b]{0.7\textwidth}
    \centering
    \includegraphics[width=\columnwidth]{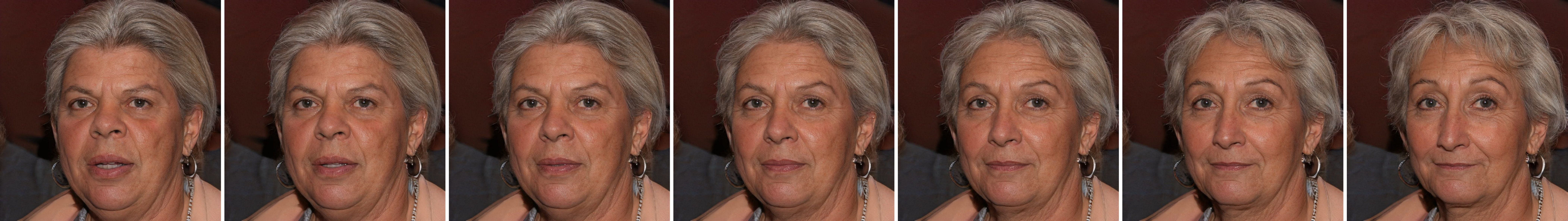}
    \caption{Layer 8 ($\wset$-space)}
    \end{subfigure} 
    \\
    \caption{
    \textbf{Comparison of Semantic Disentanglement between Layers 3, 7, 8} for the semantic "Open Mouth". Layer 7 exhibits a disentangled semantic variation of "Open mouth" without changing the face, such as a longer nose in Layer 8. Layer 3 shows only minor facial expression changes.
    }
\end{figure*}

\begin{figure*}[h]
    \centering
    \begin{subfigure}[b]{0.7\textwidth}
    \centering
    \includegraphics[width=\columnwidth]{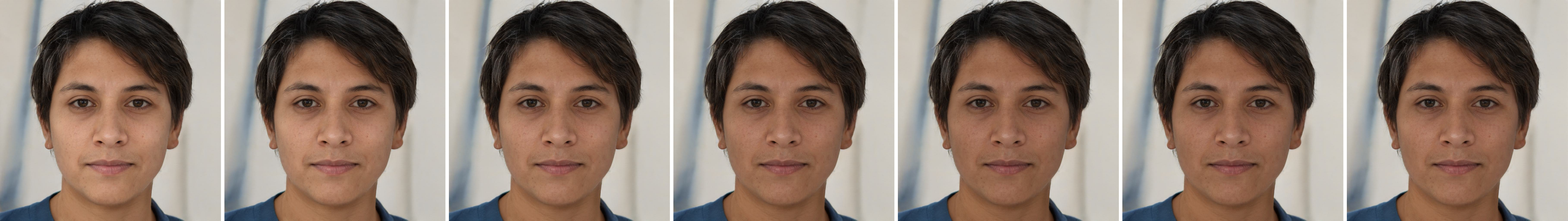}
    \caption{Max-distorted layer 3}
    \end{subfigure} 
    \\
    \begin{subfigure}[b]{0.7\textwidth}
    \centering
    \includegraphics[width=\columnwidth]{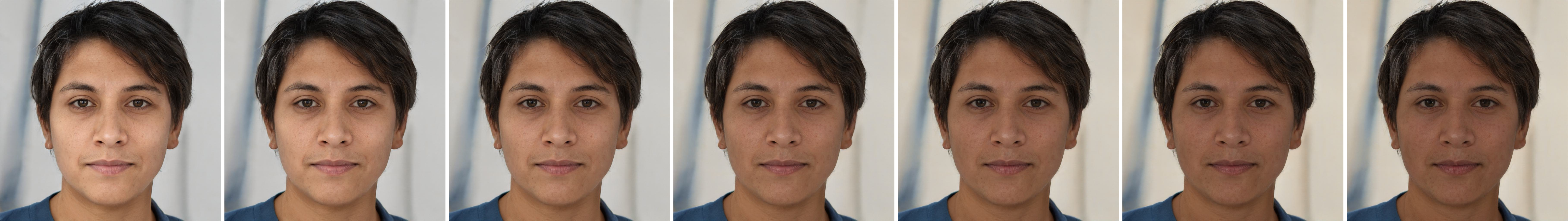}
    \caption{Min-distorted layer 7}
    \end{subfigure} 
    \\
    \begin{subfigure}[b]{0.7\textwidth}
    \centering
    \includegraphics[width=\columnwidth]{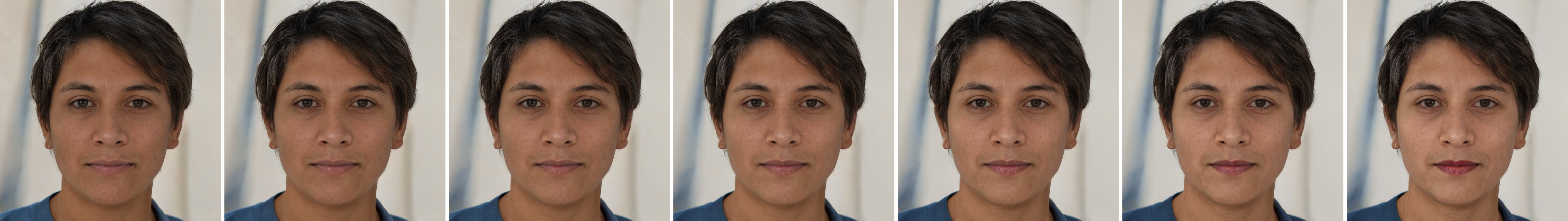}
    \caption{Layer 8 ($\wset$-space)}
    \end{subfigure} 
    \\
    \caption{
    \textbf{Comparison of Semantic Disentanglement between Layers 3, 7, 8} for the semantic "Face Bright". Layer 8 shows an entangled semantic variation of "Lipstick" and "Face Bright" on the right end. On Layer 3, the traversal toward the right end does not show "Face Dark" semantics.
    }
\end{figure*}

\begin{figure*}[h]
    \centering
    \begin{subfigure}[b]{0.7\textwidth}
    \centering
    \includegraphics[width=\columnwidth]{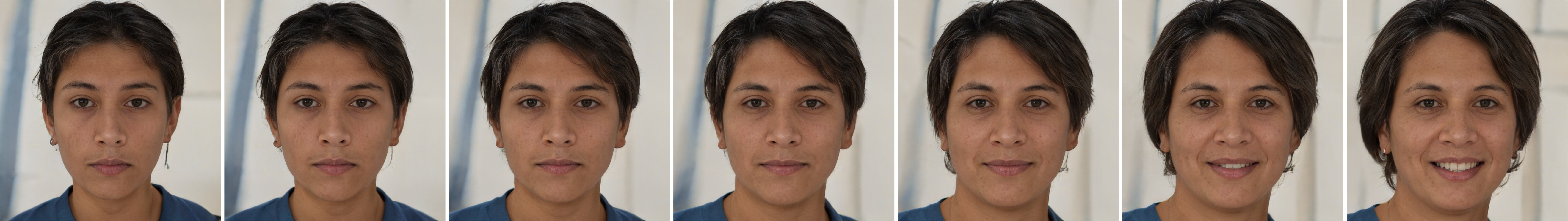}
    \caption{Max-distorted layer 3}
    \end{subfigure} 
    \\
    \begin{subfigure}[b]{0.7\textwidth}
    \centering
    \includegraphics[width=\columnwidth]{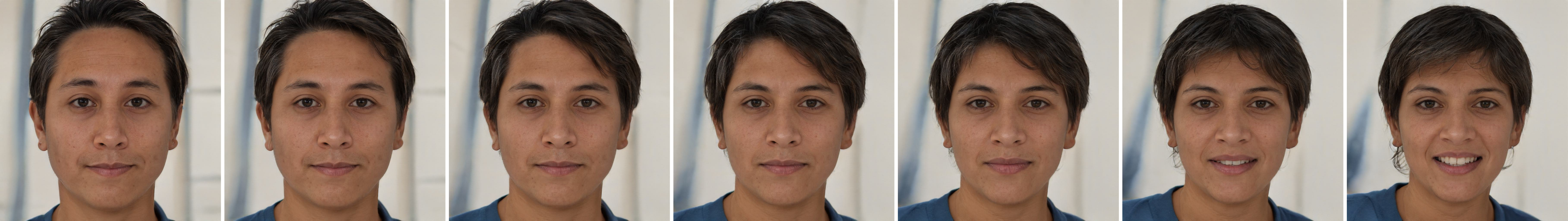}
    \caption{Min-distorted layer 7}
    \end{subfigure} 
    \\
    \begin{subfigure}[b]{0.7\textwidth}
    \centering
    \includegraphics[width=\columnwidth]{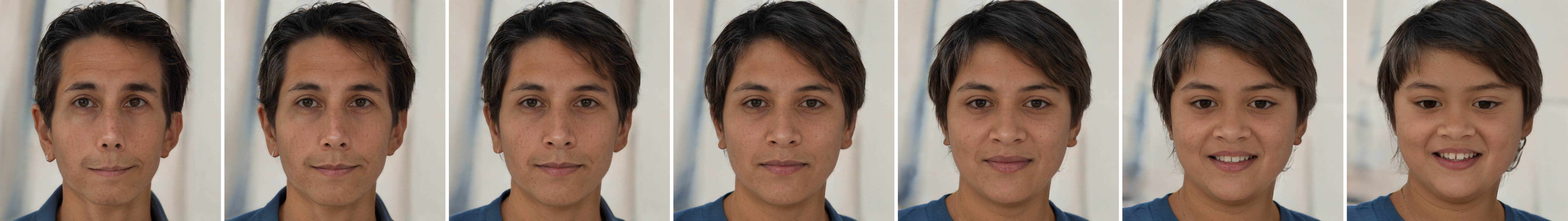}
    \caption{Layer 8 ($\wset$-space)}
    \end{subfigure} 
    \\
    \caption{
    \textbf{Comparison of Semantic Disentanglement between Layers 3, 7, 8} for the semantic "Gender". On Layer 3, the image traversal shows an entangled semantic variation of "Gender" and "Big face" on the right end. Also, Layer 8 exhibits an entangled semantic variation of "Gender" and "Age".
    }
\end{figure*}

\begin{figure*}[h]
    \centering
    \begin{subfigure}[b]{0.7\textwidth}
    \centering
    \includegraphics[width=\columnwidth]{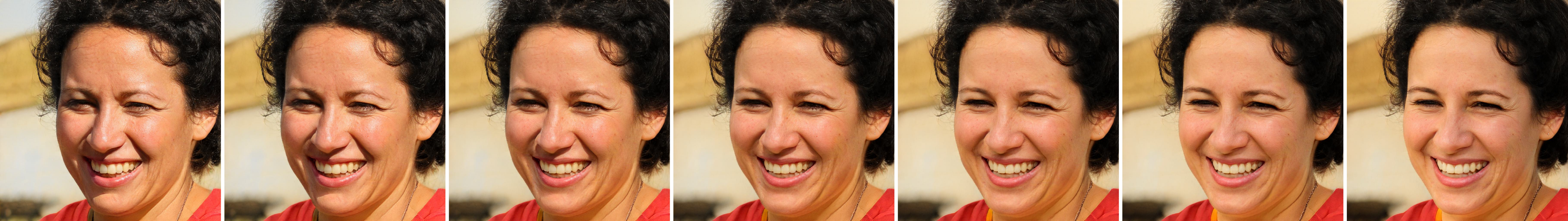}
    \caption{Max-distorted layer 3}
    \end{subfigure} 
    \\
    \begin{subfigure}[b]{0.7\textwidth}
    \centering
    \includegraphics[width=\columnwidth]{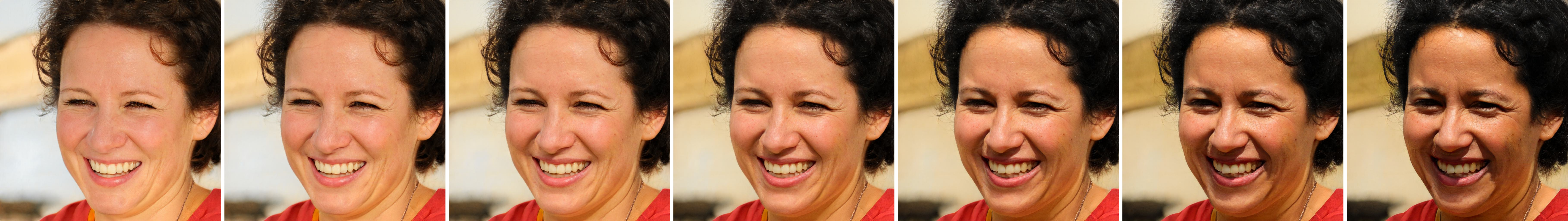}
    \caption{Min-distorted layer 7}
    \end{subfigure} 
    \\
    \begin{subfigure}[b]{0.7\textwidth}
    \centering
    \includegraphics[width=\columnwidth]{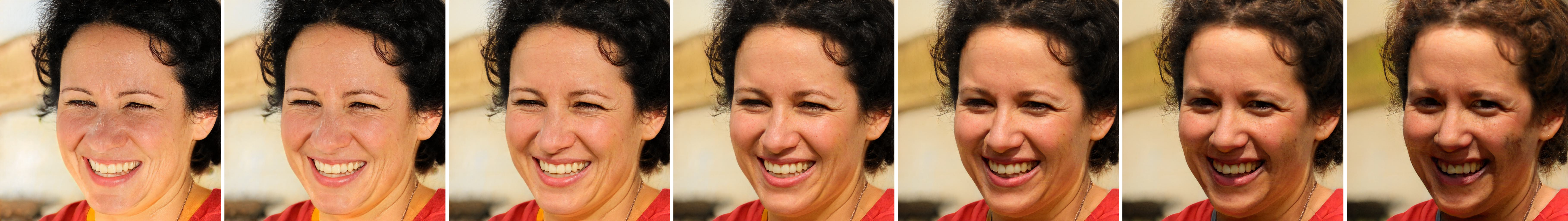}
    \caption{Layer 8 ($\wset$-space)}
    \end{subfigure} 
    \\
    \caption{
    \textbf{Comparison of Semantic Disentanglement between Layers 3, 7, 8} for the semantic "Light Up". Layer 8 exhibits an entangled semantic variation of "Light Down" and "Brown hair" on the right end. On the left end of Layer 3, the traversal shows the "Bright right" transformation. However, the traversal at the right end of Layer 3 does not present the "Bright left" transformation, while showing no shadow on the face. But, the traversal on Layer 7 shows a natural brightness change on the entire face. 
    }
\end{figure*}


\clearpage
\section{Additional Subspace Traversal Comparisons} \label{sec:additional_traversal}
\begin{figure}[h]
    \centering
    \begin{subfigure}[b]{0.4\columnwidth}
    \centering
    \includegraphics[width=\columnwidth]{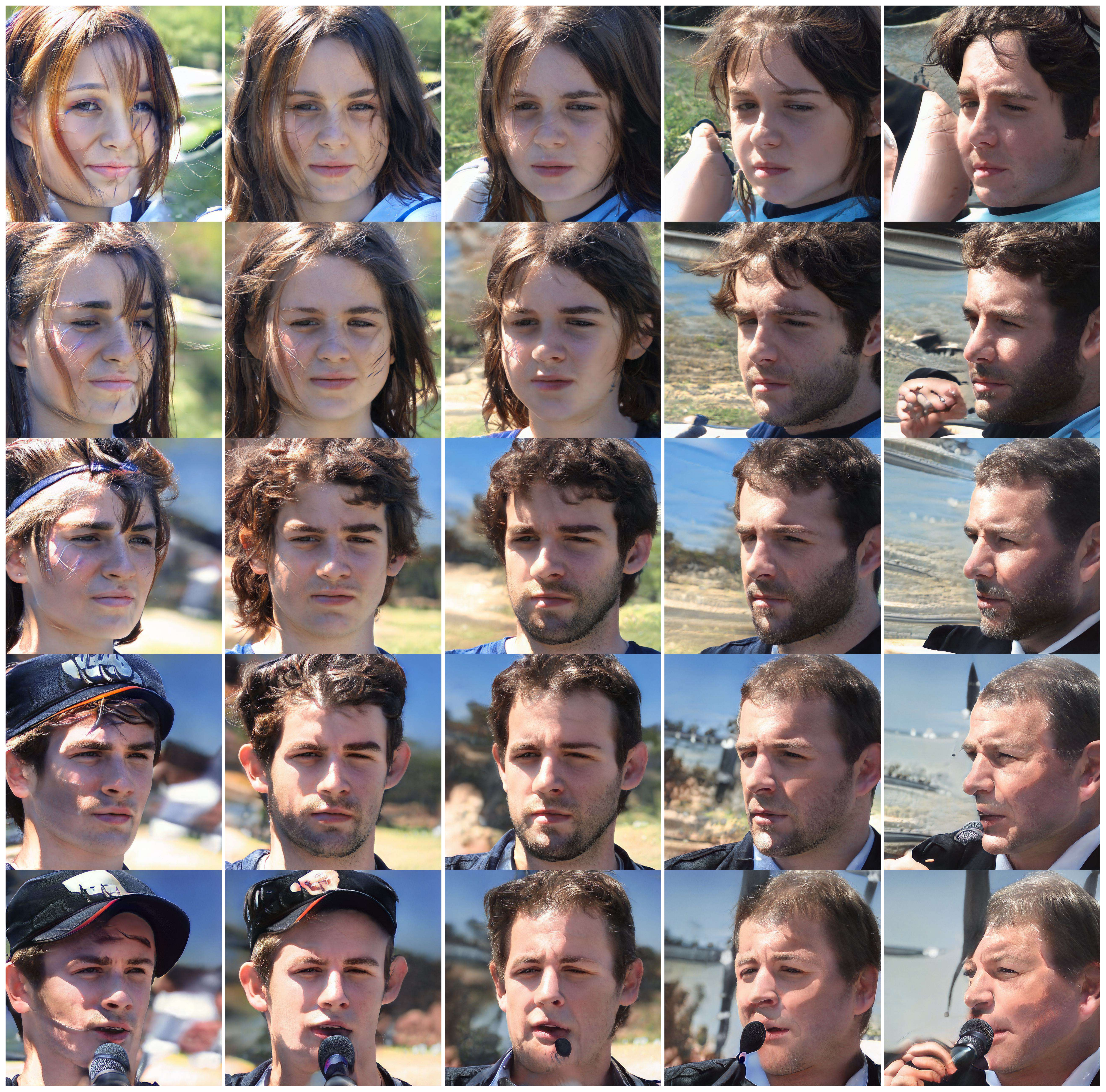}
    \caption{Min-distorted layer - Global Basis}
    \end{subfigure}
    \quad
    \begin{subfigure}[b]{0.4\columnwidth}
    \centering
    \includegraphics[width=\columnwidth]{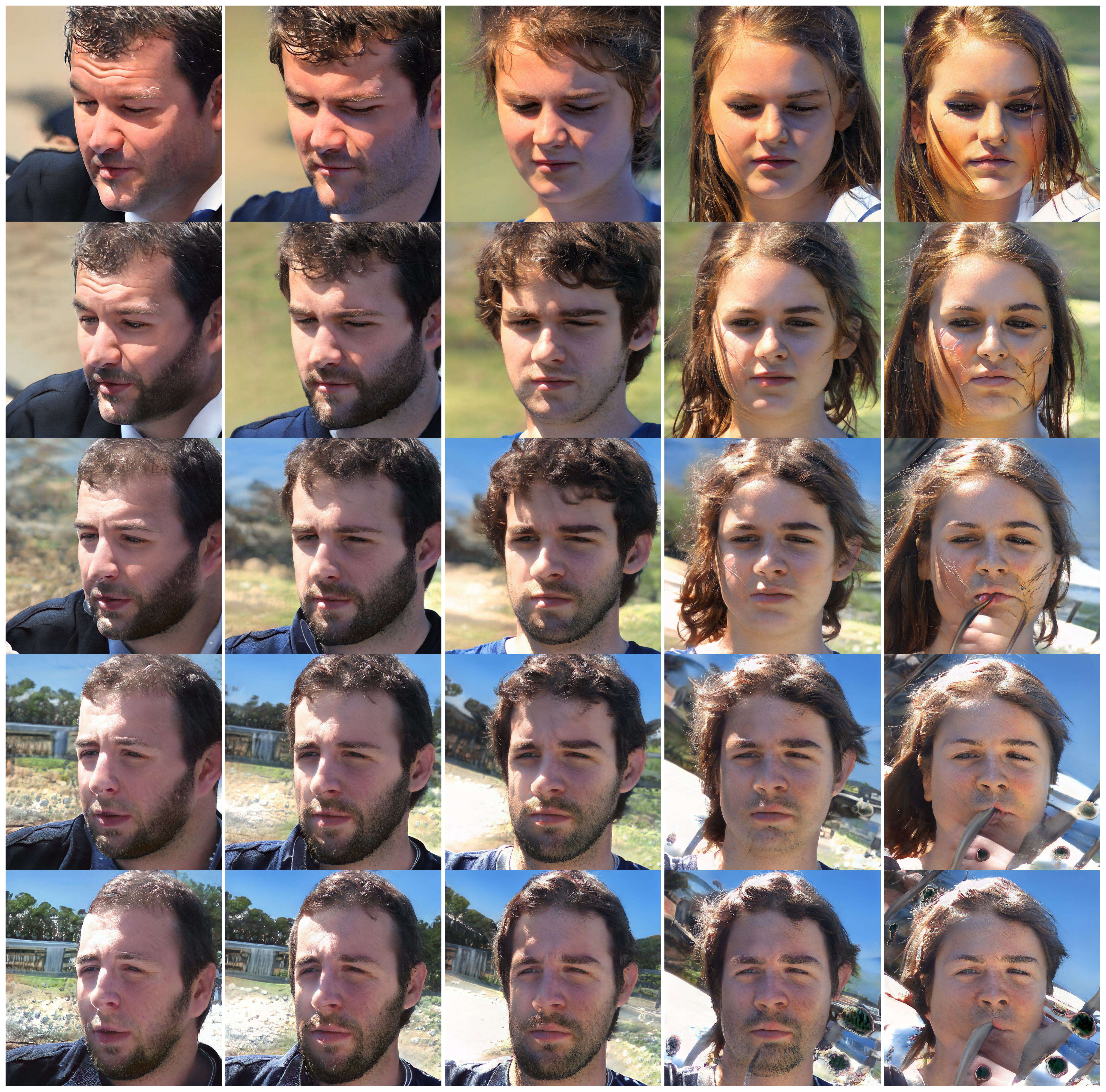}
    \caption{Min-distorted layer - Local Basis}
    \end{subfigure}
    \\
    \centering
    \begin{subfigure}[b]{0.4\columnwidth}
    \centering
    \includegraphics[width=\columnwidth]{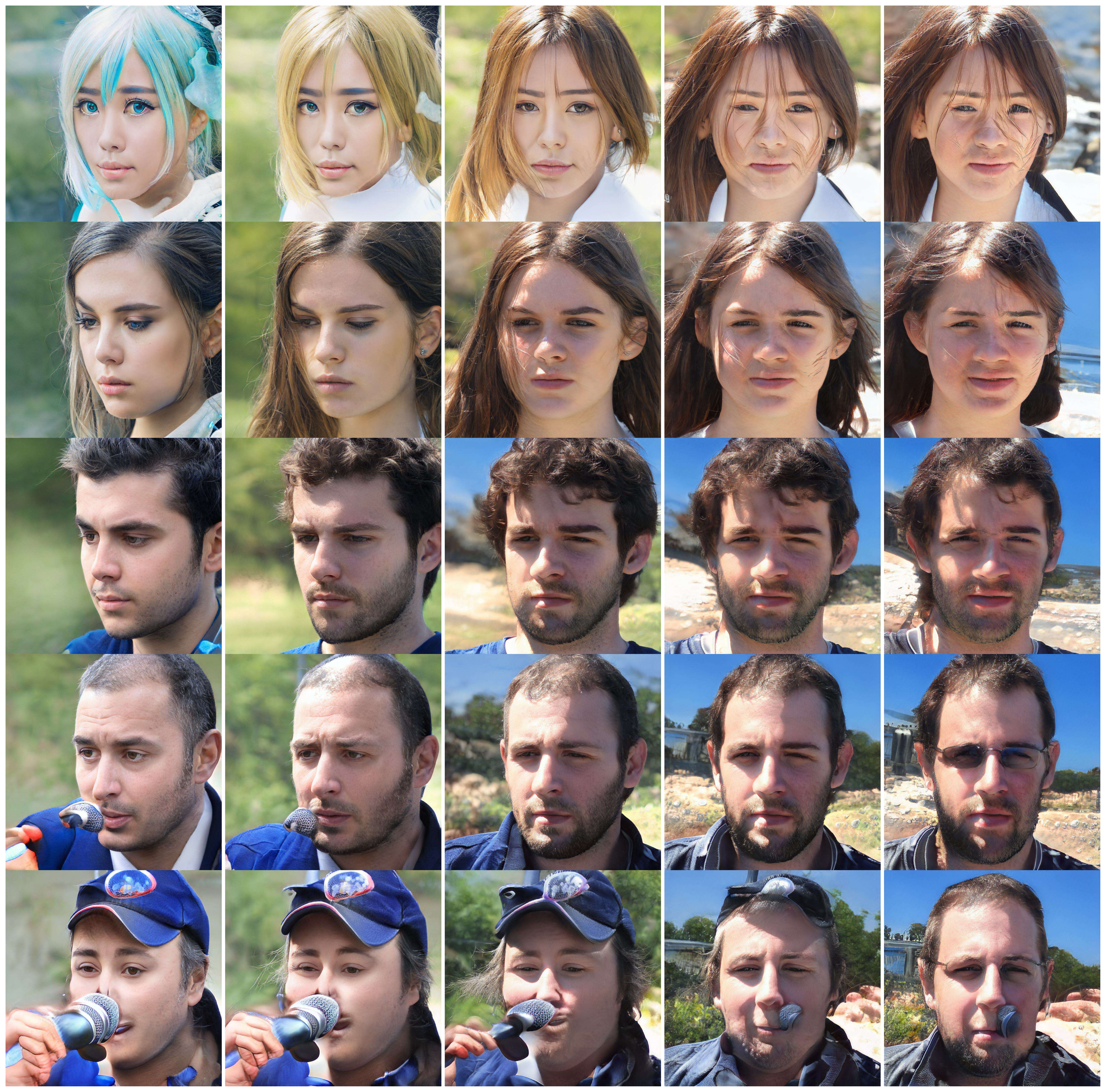}
    \caption{Max-distorted layer - Global Basis}
    \end{subfigure}
    \quad
    \begin{subfigure}[b]{0.4\columnwidth}
    \centering
    \includegraphics[width=\columnwidth]{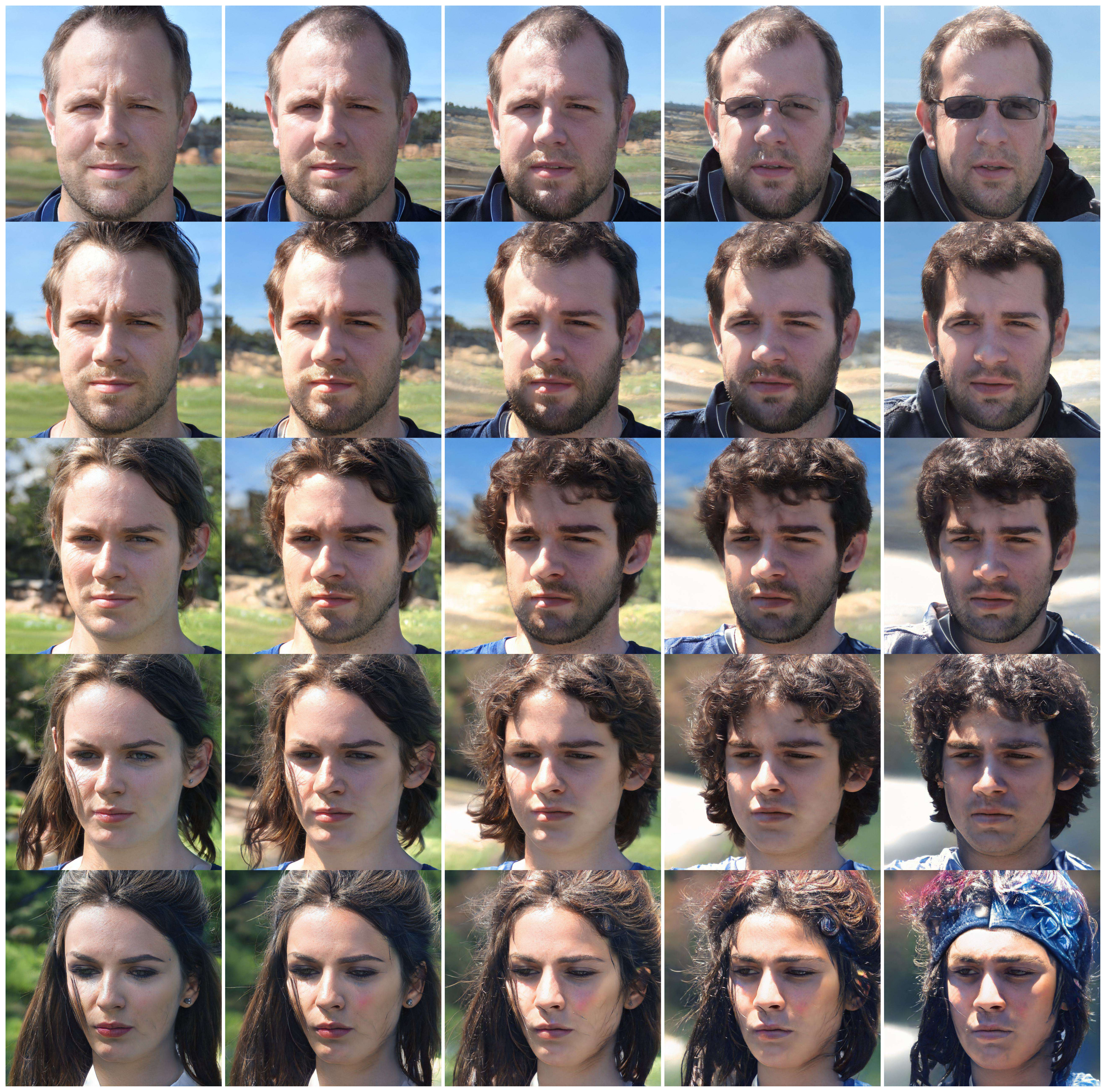}
    \caption{Max-distorted layer - Local Basis}
    \end{subfigure}
    \caption{
    \textbf{Subspace Traversal \citep{LocalBasis} on the min-distorted ($7$th) and max-distorted ($3$rd) intermediate layers} along the global basis \citep{harkonen2020ganspace} and Local Basis \citep{LocalBasis} of StyleGAN2 on FFHQ. The initial image (center) is traversed along the $1$st (horizontal) and $2$nd (vertical) components of the chosen traversal directions with the perturbation intensity 9. The global basis shows a decent image quality on the min-distorted layer, similar to Local Basis. However, on the max-distorted layer, the subspace traversal along global basis exhibits significant failures at corners, such as image collapse (lower-left), visual artifacts (lower-right), and unnatural transformations (top-left).
    }
\end{figure}

\begin{figure}[]
    \centering
    \begin{subfigure}[b]{0.3\columnwidth}
    \centering
    \includegraphics[width=\columnwidth]{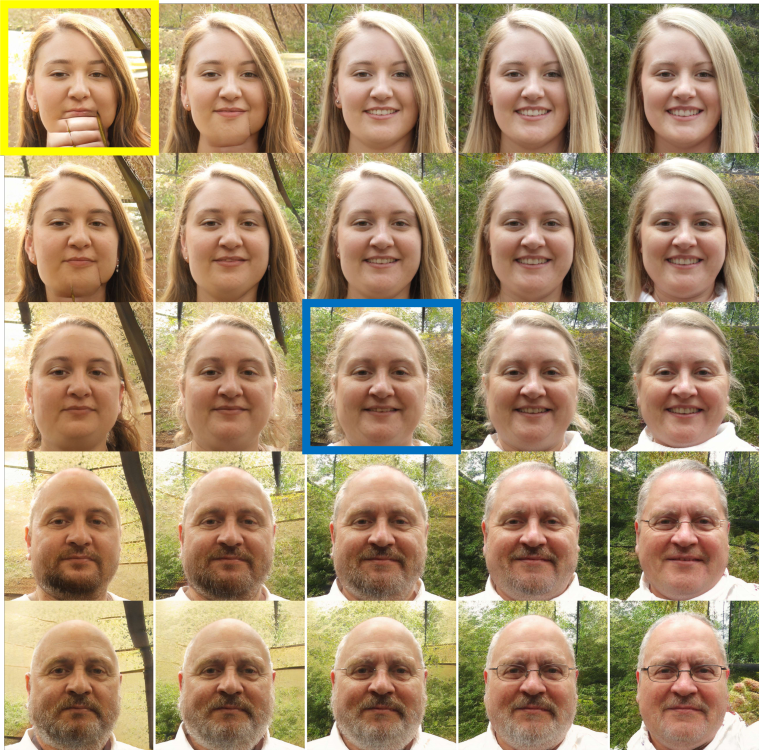}
    \caption{Max-distorted layer 3}
    \end{subfigure}
    \quad
    \begin{subfigure}[b]{0.3\columnwidth}
    \centering
    \includegraphics[width=\columnwidth]{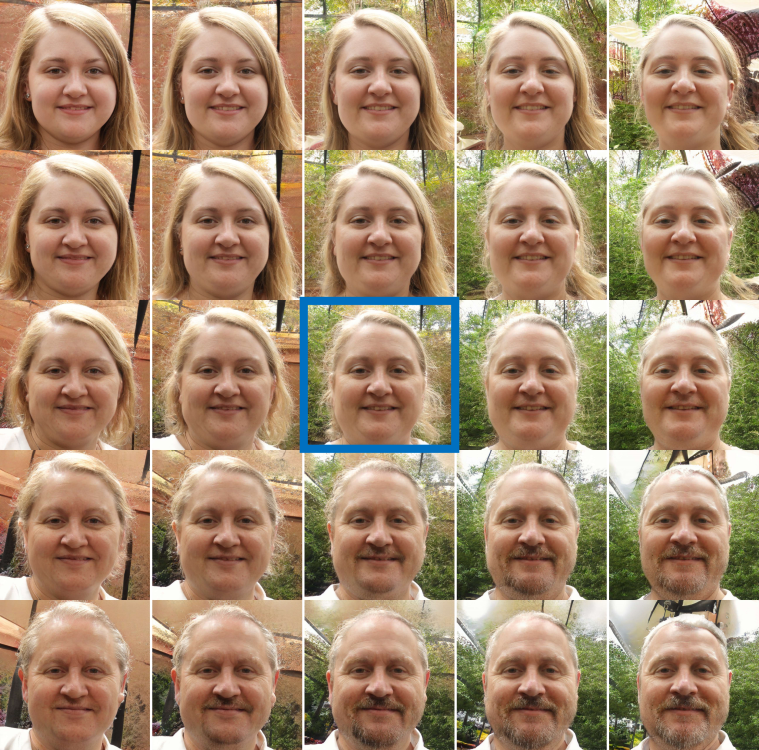}
    \caption{Min-distorted layer 7}
    \end{subfigure}
    \quad
    \begin{subfigure}[b]{0.3\columnwidth}
    \centering
    \includegraphics[width=\columnwidth]{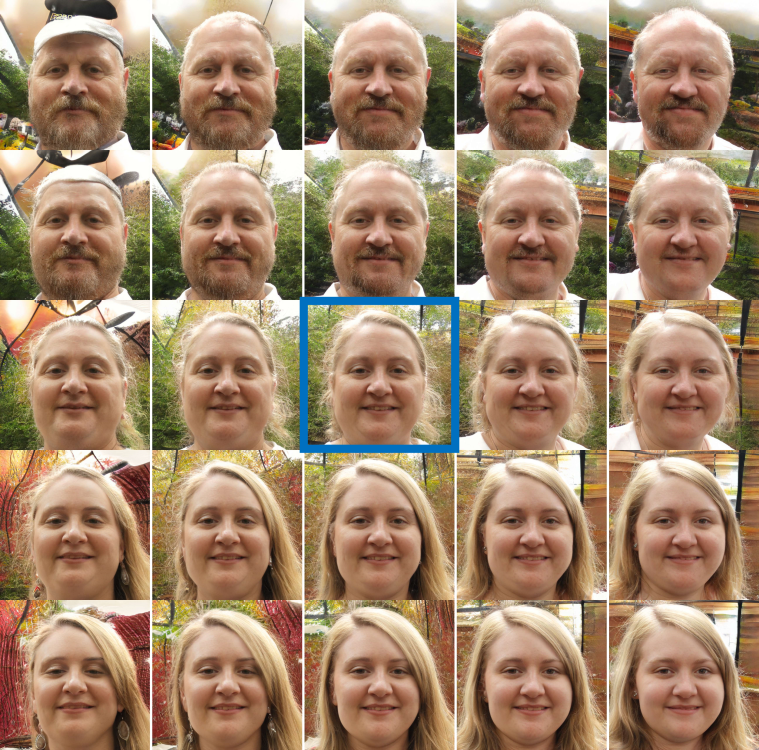}
    \caption{Layer 8 ($\wset$-space)}
    \end{subfigure}
    \caption{
    \textbf{Subspace Traversal} on StyleGAN2-FFHQ. The upper-left corner of layer 3 is severely deteriorated.
    }
\end{figure}


\begin{figure}[]
    \centering
    \begin{subfigure}[b]{0.3\columnwidth}
    \centering
    \includegraphics[width=\columnwidth]{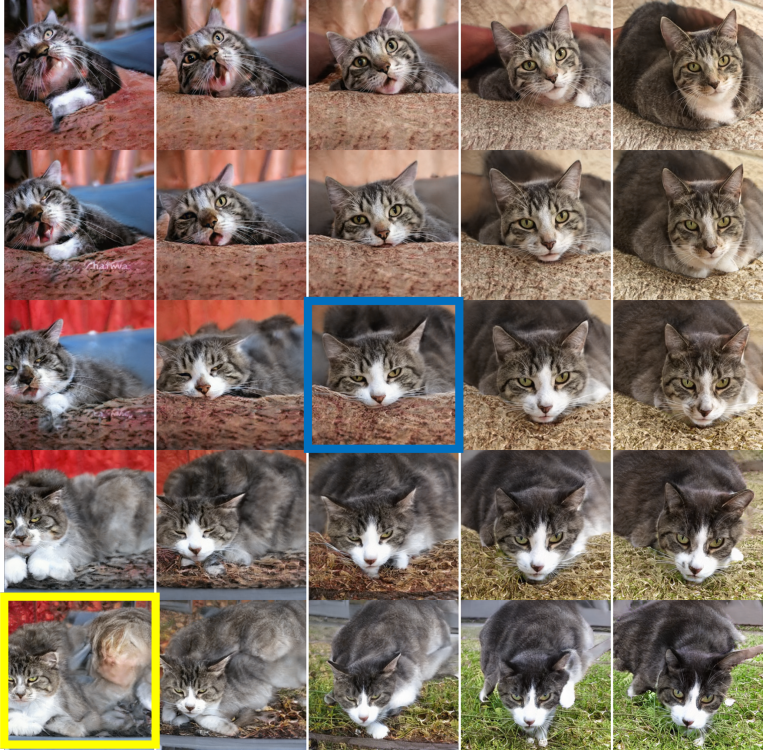}
    \caption{Max-distorted layer 3}
    \end{subfigure}
    \quad
    \begin{subfigure}[b]{0.3\columnwidth}
    \centering
    \includegraphics[width=\columnwidth]{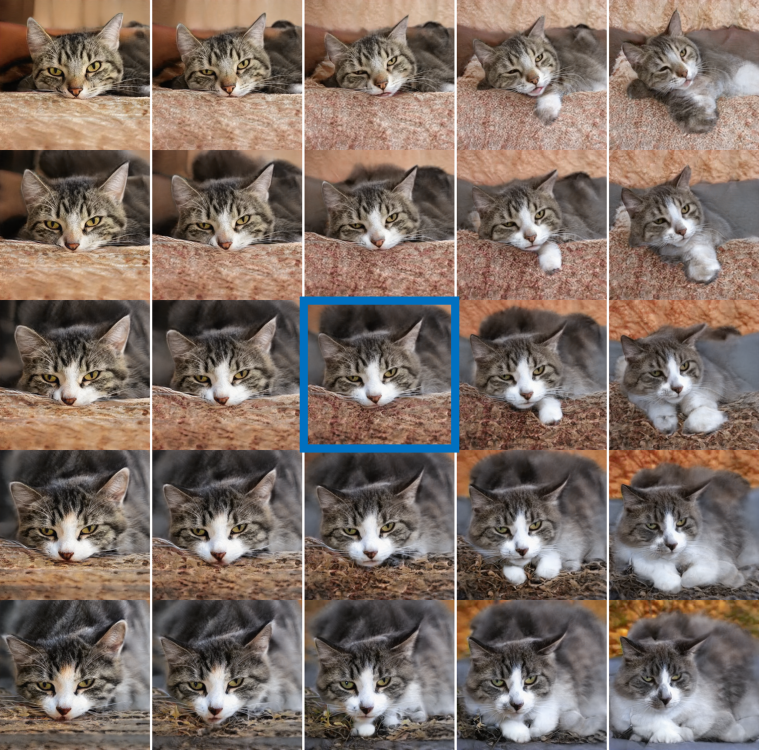}
    \caption{Min-distorted layer 6}
    \end{subfigure}
    \quad
    \begin{subfigure}[b]{0.3\columnwidth}
    \centering
    \includegraphics[width=\columnwidth]{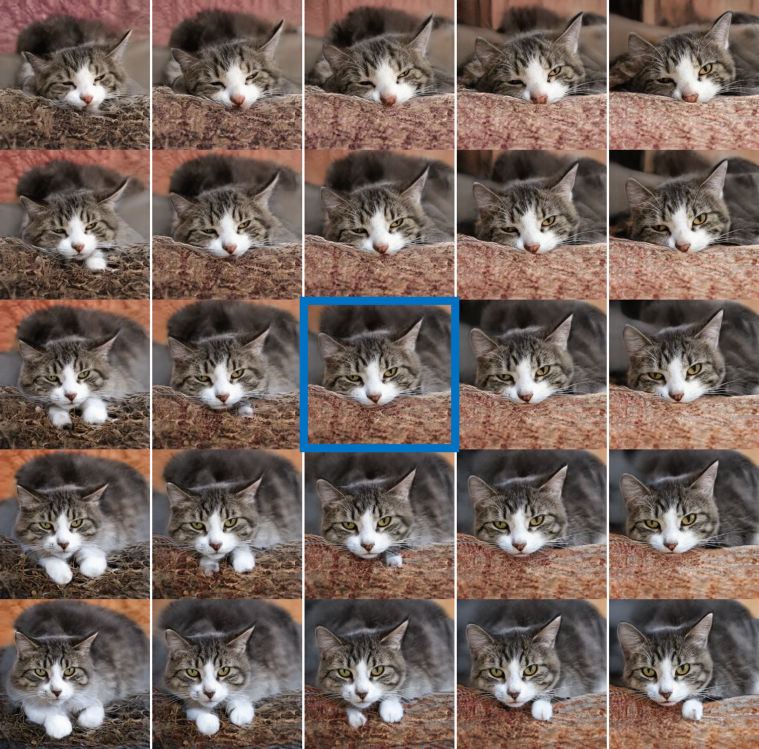}
    \caption{Layer 8 ($\wset$-space) }
    \end{subfigure}
    \caption{
    \textbf{Subspace Traversal} on StyleGAN2-LSUN Cat. The lower-left corner of layer 3 are severely deteriorated.}
\end{figure}

\begin{figure}[]
    \centering
    \begin{subfigure}[b]{0.3\columnwidth}
    \centering
    \includegraphics[width=\columnwidth]{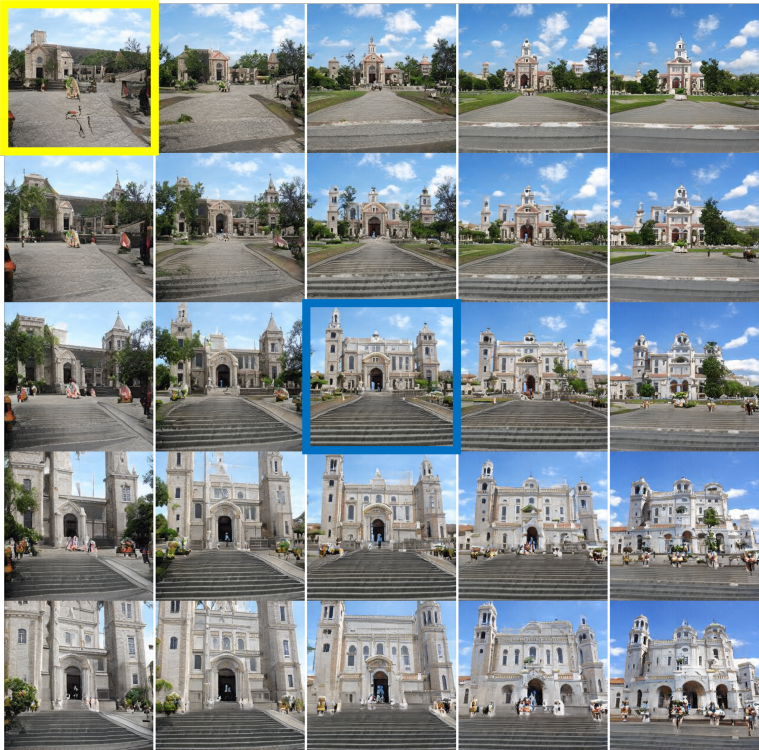}
    \caption{Max-distorted layer 3}
    \end{subfigure}
    \quad
    \begin{subfigure}[b]{0.3\columnwidth}
    \centering
    \includegraphics[width=\columnwidth]{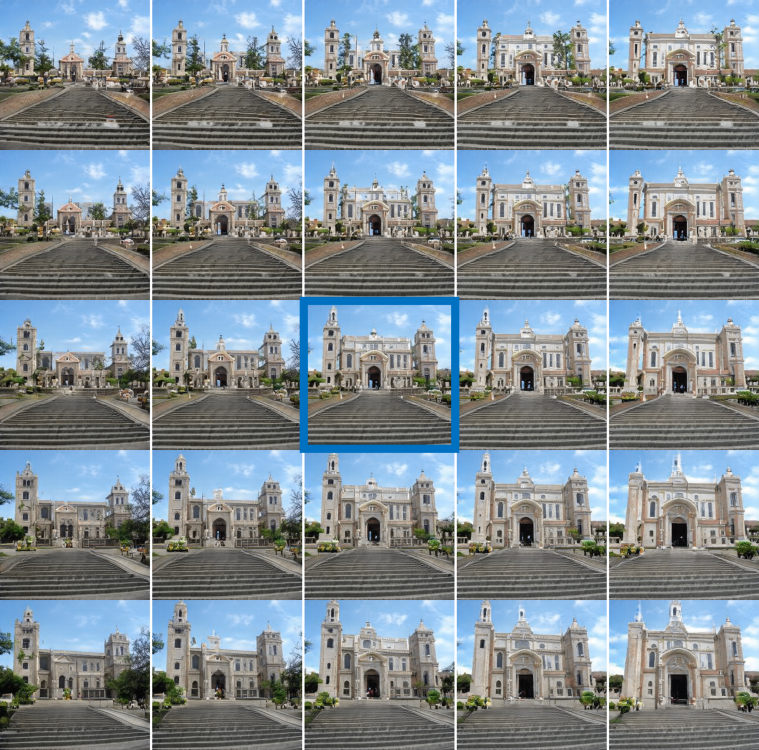}
    \caption{Min-distorted layer 6}
    \end{subfigure}
    \quad
    \begin{subfigure}[b]{0.3\columnwidth}
    \centering
    \includegraphics[width=\columnwidth]{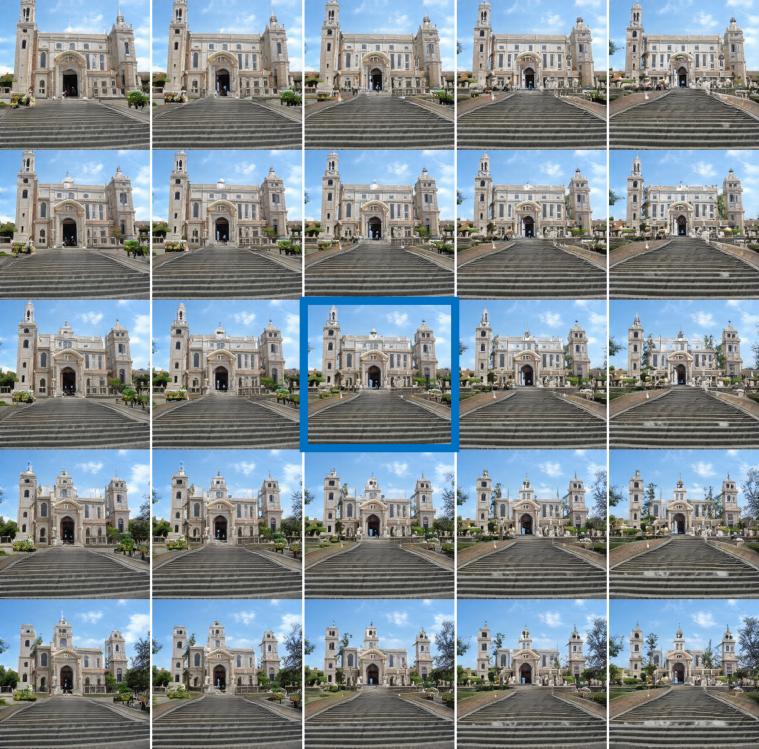}
    \caption{Layer 8 ($\wset$-space) }
    \end{subfigure}
    \caption{
    \textbf{Subspace Traversal} on StyleGAN2-LSUN Church. The upper-left corner of layer 3 are severely deteriorated.}
\end{figure}

\begin{figure}[]
    \centering
    \begin{subfigure}[b]{0.3\columnwidth}
    \centering
    \includegraphics[width=\columnwidth]{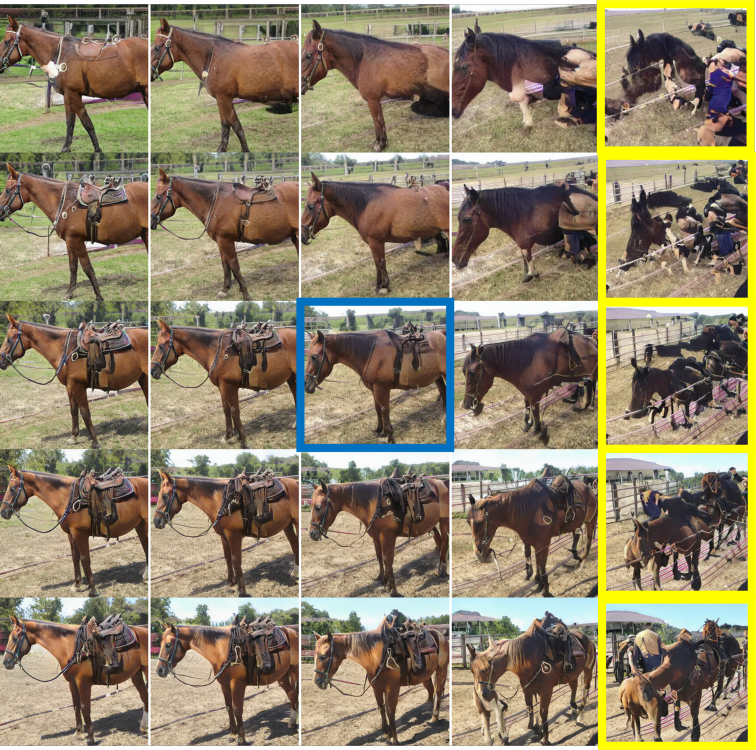}
    \caption{Max-distorted layer 3}
    \end{subfigure}
    \quad
    \begin{subfigure}[b]{0.3\columnwidth}
    \centering
    \includegraphics[width=\columnwidth]{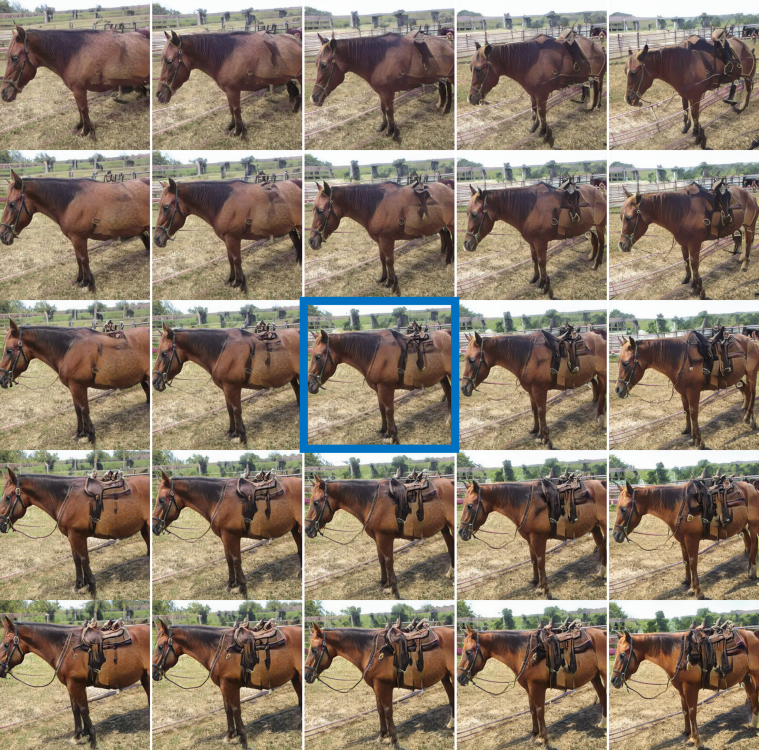}
    \caption{Min-distorted layer 6}
    \end{subfigure}
    \quad
    \begin{subfigure}[b]{0.3\columnwidth}
    \centering
    \includegraphics[width=\columnwidth]{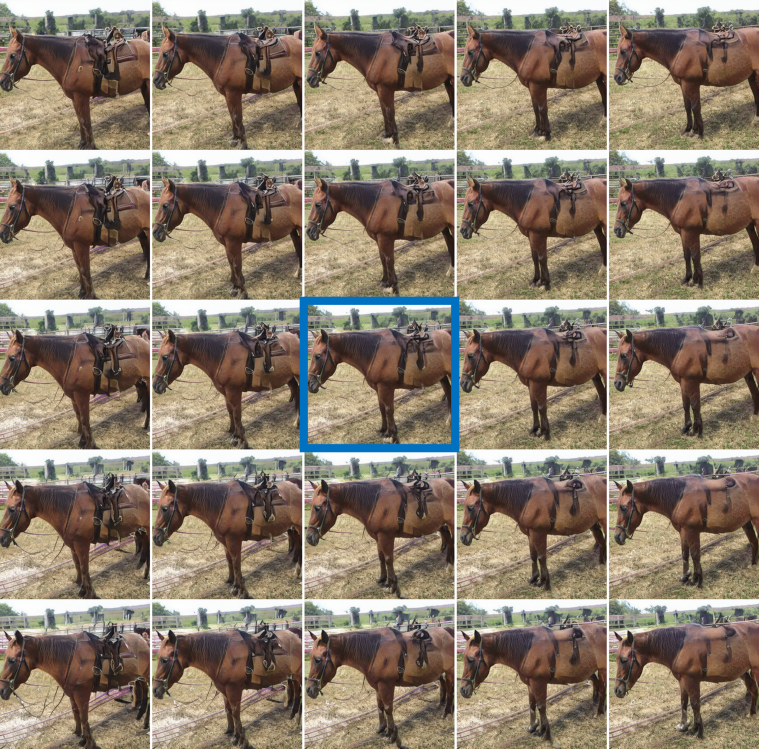}
    \caption{Layer 8 ($\wset$-space) }
    \end{subfigure}
    \caption{
    \textbf{Subspace Traversal} on StyleGAN2-LSUN Horse. The right sides of layer 3 are severely deteriorated.}
\end{figure}

\begin{figure}[]
    \centering
    \begin{subfigure}[b]{0.3\columnwidth}
    \centering
    \includegraphics[width=\columnwidth]{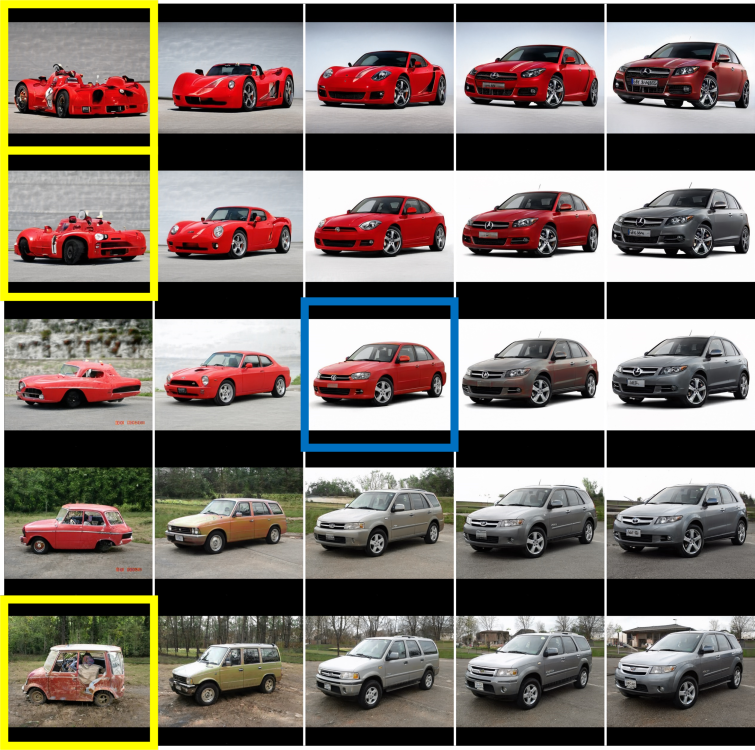}
    \caption{Max-distorted layer 3}
    \end{subfigure}
    \quad
    \begin{subfigure}[b]{0.3\columnwidth}
    \centering
    \includegraphics[width=\columnwidth]{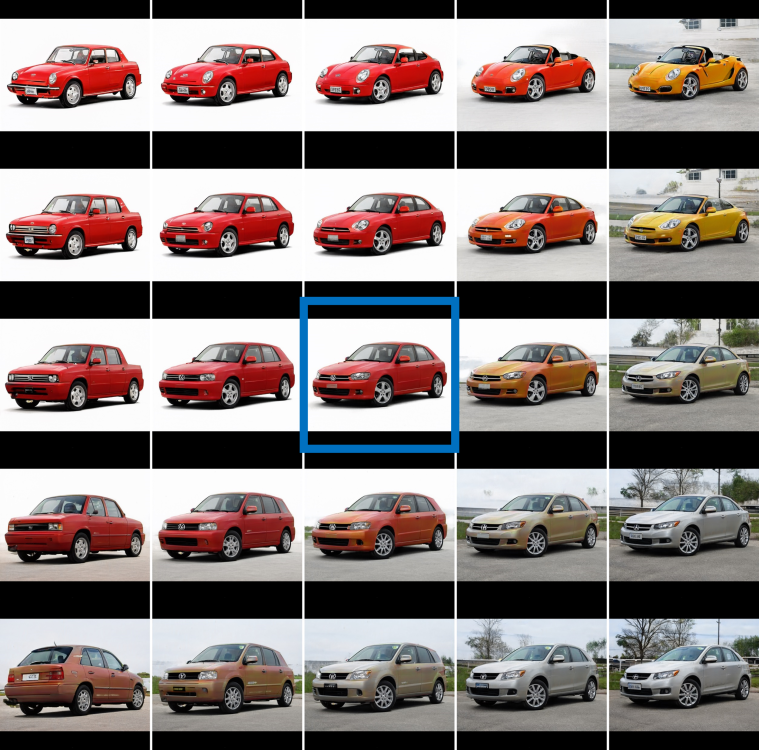}
    \caption{Min-distorted layer 7}
    \end{subfigure}
    \quad
    \begin{subfigure}[b]{0.3\columnwidth}
    \centering
    \includegraphics[width=\columnwidth]{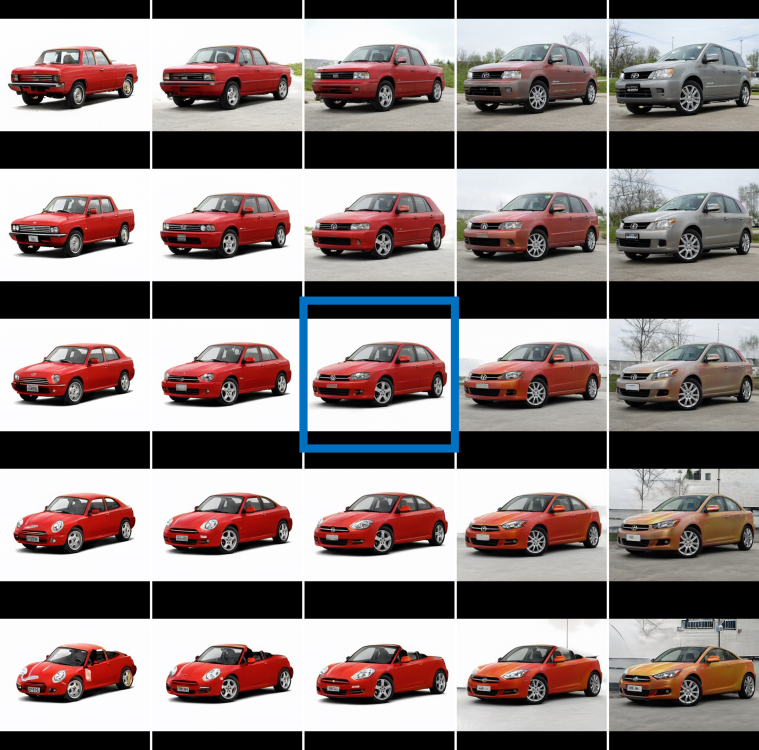}
    \caption{Layer 8 ($\wset$-space)}
    \end{subfigure}
    \caption{
    \textbf{Linear Traversal} on StyleGAN2-LSUN Cars. The left sides of layer 3 are severely deteriorated.}
\end{figure}

\end{document}